\documentclass{article}

\usepackage[preprint]{icml2024}
\usepackage{caption}
\usepackage{subcaption}
\usepackage{microtype}
\usepackage{graphicx}
\usepackage{booktabs} %
\usepackage{empheq}
\usepackage{url}
\usepackage{hyperref}

\usepackage[normalem]{ulem}
\usepackage{amsmath,amsfonts,bm}
\usepackage{mathtools}
\usepackage{amsthm}
\usepackage[nameinlink,capitalise]{cleveref}
\usepackage[acronym,nowarn,section,nogroupskip,nonumberlist]{glossaries}
\usepackage{cancel}
\usepackage{tikz}
\usepackage{nccmath}
\usepackage{minitoc}
\usepackage[breakable]{tcolorbox}
\usepackage{centernot}
\usepackage{pifont}
\usepackage{thm-restate}
\usepackage{multirow}
\usepackage{multicol}
\usepackage{adjustbox}
\usepackage{minitoc}

\usepackage{amsmath,amsfonts,bm}
\usepackage{booktabs}
\usepackage{inconsolata}
\global\long\def\proptoo#1{\stackrel{#1}{\propto}}
\newcommand{\colonpropto}{\mathrel{:\propto}}

\global\long\def\ss{\mathfrak{s}}%
\global\long\def\ssb{\boldsymbol{\mathfrak{s}}}%

\makeatletter
\def\th@remark{%
  \thm@headfont{\bfseries}%
  \normalfont %
  \thm@preskip\topsep \divide\thm@preskip\tw@
  \thm@postskip\thm@preskip
}
\makeatother

\crefname{section}{Sec.}{Sec.}
\crefname{appendix}{App.}{App.}
\crefname{algorithm}{Alg.}{Alg.}
\crefname{figure}{Fig.}{Fig.}
\crefname{definition}{Def.}{Def.}
\crefname{proposition}{Prop.}{Prop.}
\crefname{theorem}{Thm.}{Thm.}

\newcommand{\vheader}{\vspace*{-.125cm}}

\newcommand{\sigmacond}[1]{\sigma}%
\newcommand{\ocond}{\bo_{1:T}}

\newcommand{\phishort}{\text{potential }}
\newcommand{\phishortnospace}{\text{potential}}
\newcommand{\phishorts}{\text{potentials }}
\newcommand{\phishortsnospace}{\text{potentials}}
\newcommand{\bigphishort}{\text{Potential }}
\newcommand{\bigphishorts}{\text{Potentials }}

\newcommand{\optornot}{{*}}
\newcommand{\smcargs}{\bS}%
\newcommand{\zsigma}{\mathcal{Z}_{\sigma}}

\newcommand{\smcprop}{q_{\textsc{smc}}}
\newcommand{\smctgt}{{\si}_{\textsc{smc}}}

\newcommand{\tsmctgt}{\tilde{\si}_{\textsc{smc}}}

\newcommand{\sigmaot}{\sigma_{o_T}}

\hypersetup{
    colorlinks,
    linkcolor={blue!50!black},
    citecolor={blue!50!black},
    urlcolor={blue!80!black},
}
\let\originalleft\left
\let\originalright\right
\renewcommand{\left}{\mathopen{}\mathclose\bgroup\originalleft}
\renewcommand{\right}{\aftergroup\egroup\originalright}
\global\long\def\pv#1#2{\left(#1\,\middle\rvert\,#2\right)}%
\global\long\def\om{\omega}%
\global\long\def\fr#1#2{\frac{#1}{#2}}%
\global\long\def\bh#1{\left[\vphantom{\left(#1\right)^{2}}#1\right]}%
\global\long\def\s{\bs}%
\global\long\def\EEE{\mathbb{E}}%
\global\long\def\DKL{D_{\textsc{kl}}}%
\global\long\def\Z{\mathcal{Z}}%
\global\long\def\pV#1#2{\left(#1\,\middle\|\,#2\right)}%
\global\long\def\si{\sigma}%
\global\long\def\thb{{\boldsymbol{\theta}}}%
\global\long\def\gr{\nabla}%

\newcommand{\ccancel}[1]{{\color{red}\cancel{\color{black}#1}}}
\newcommand{\Cancel}[1]{{\color{blue}\cancel{\color{black}#1}}}
\newcommand{\finaltwist}{\phi(\bs_{1:T})}
\newcommand{\finaltwistonly}{\phi}

\newcommand{\approxfinaltwist}{\psi_{T}}%
\newcommand{\reward}[1]{\finaltwistonly_{#1}(\bs_{1:#1})}

\newcommand{\optimaltwist}[1]{\psi^*_{#1}}
\newcommand{\vartwist}[1]{\psi_{#1}^{\thb}}
\newcommand{\vartwistq}[1]{\psi^{\qparam}_{#1}}

\newcommand{\approxtwist}[1]{{\psi_{#1}}}
\newcommand{\myo}{o}%
\newcommand{\bo}{\boldsymbol{\myo}}
\newcommand{\optimaltwistphi}[1]{\varPhi_{#1}^*}
\newcommand{\vartwistphi}[1]{\varPhi_{#1}}
\newcommand{\vartwistphitheta}[1]{\varPhi_{#1}^{\thb}}
\newcommand{\eqrl}{\overset{\text{(sRL)}}{=} }

\newcommand{\rewardone}[1]{\phi_{#1}(\bs_{1:#1}, o_{#1}=1)}
\newcommand{\bone}{\boldsymbol{1}}

\newcommand{\zfilter}[1]{\cZ_{{#1}}^{\approxtwist{}}}%
\newcommand{\zfilterstar}[1]{\cZ_{{#1}}^{\optimaltwist{}}}%
\newcommand{\qparam}
{{\boldsymbol{\xi}}}

\newcommand{\propheader}[1]{\textup{\textbf{#1}}}

\newcommand{\samplingdist}{\pi_{s}}
\newcommand{\zonestep}[2]{{Z}^{\pitwist{}}_{#1}(\bs_{1:{#2}})}

\newcommand{\pitwist}[1]{\pi_{#1}}
\newcommand{\tpitwist}[1]{\tilde{\pi}_{#1}}

\newcommand{\pitwisttheta}[1]
{\pi_{#1}^{\thb}}
\newcommand{\tpitwisttheta}[1]{\tpitwist{#1}^{\thb}}

\newcommand{\propopttwist}[1]{{\prop}^{\pi^*}_{#1}}

\newcommand{\zpi}[1]{\cZ_{#1}}
\newcommand{\zfinal}{\cZ_\sigma}

\newcommand{\condzero}{}%

\newcommand{\proptwist}[1]{\prop_{#1}^{\pi}}

\newcommand{\ssweight}[1]{w
(\bs_{1:T}^{#1})}

\newcommand{\smcincweight}[2]{w_{#1}(\bs_{1:#1}^{#2})}%

\newcommand{\sind}{\omega}
\newcommand{\sindat}[2]{\sind^{#1}_{#2}}

\newcommand{\klof}[2]{\DKL\pV{#1}{#2}}
\newcommand{\expof}[1]{e^{#1}}

\def\bs {%
\mathbf{s}
}

\tcbset{mybox/.style={colback=red!5, width =\columnwidth, boxrule=0.0pt, top=5pt,bottom=5pt,left=2pt, right=2pt},
  before upper=\setlength{\parindent}
  {0em}\everypar{{\setbox0\lastbox}\everypar{}}
}
\newtcolorbox{mybox}[1][]{breakable, mybox,#1}

\newacronym{SIS}{\textsc{SIS}}{simple importance sampling}
\newacronym{SMC}{\textsc{SMC}}{Sequential Monte Carlo}
\newacronym{LLM}{\textsc{LLM}}{Large language models}
\newacronym{IWAE}{\textsc{IWAE}}{Importance Weighted Autoencoder}
\newacronym{RLHF}{\textsc{RLHF}}{reinforcement learning from human feedback}
\newacronym{PPO}{\textsc{PPO}}{proximal policy optimization}
\newacronym{DPO}{\textsc{DPO}}{Direct Preference Optimization}
\newacronym{CTL}{\textsc{CTL}}{\textit{contrastive twist learning}}
\newacronym{ESS}{\textsc{ESS}}{Effective Sample Size}
\newacronym{DPG}{\textsc{DPG}}{distributional policy gradient}
\newacronym{NCE}{\textsc{NCE}}{noise contrastive estimation}
\newacronym{EBM}{\textsc{EBM}}{energy-based models}
\newacronym{RL}{\textsc{RL}}{reinforcement learning}
\newacronym{PCL}{\textsc{PCL}}{path consistency learning}
\newacronym{MDP}{\textsc{MDP}}{Markov Decision Process}
\newacronym{SNIS}{\textsc{SNIS}}{self-normalized importance sampling}
\newacronym{MCMC}{\textsc{MCMC}}{Markov Chain Monte Carlo}

\NewDocumentCommand{\composet}{e{_^}}{%
  \mathbin{\mathop{\circ}\displaylimits
    \IfValueT{#1}{_{#1}}
    \IfValueT{#2}{^{#2}}
  }%
}

\makeatletter
\newcommand{\AlignFootnote}[1]{%
  \ifmeasuring@
    \chardef\@tempfn=\value{footnote}%
    \footnotemark
    \setcounter{footnote}{\@tempfn}%
  \else
    \iffirstchoice@
      \footnote{#1}%
    \fi
  \fi}
\makeatother

\newcommand{\hl}[1]{\textcolor{blue}{#1}}

\newcommand{\prop}{q}
\newcommand{\target}{\sigma}
\newcommand{\ttarget}{\tilde{\sigma}}
\newcommand{\optdist}{q}%
\newcommand{\base}{p_0}

\newcommand{\vopt}{V^*}

\newcommand{\tpi}{\tilde{\pi}}

\newcommand{\cC}{{\mathcal{C}}}

\newcommand{\cL}{{\mathcal{L}}}

\newcommand{\cV}{{\mathcal{V}}}

\newcommand{\cZ}{{\mathcal{Z}}}

\newcommand{\bS}{{\bm{S}}}

\newcommand{\bbI}{{\mathbb{I}}}

\newcommand{\bbR}{{\mathbb{R}}}

\def\eqref#1{equation~\ref{#1}}

\def\1{\bm{1}}

\DeclareMathAlphabet{\mathsfit}{\encodingdefault}{\sfdefault}{m}{sl}
\SetMathAlphabet{\mathsfit}{bold}{\encodingdefault}{\sfdefault}{bx}{n}

\newcommand{\E}{\mathbb{E}}

\DeclareMathOperator*{\argmax}{arg\,max}
\DeclareMathOperator*{\argmin}{arg\,min}

\global\long\def\si{\sigma}%

\newcommand{\args}{s_{t}|\bs_{1:t-1}}
\newcommand{\argst}{\bs_{1:t-1}}
\newcommand{\argstone}{\bs_{1:t}}

\glsdisablehyper

\usepackage{amsmath}
\usepackage{amssymb}
\usepackage{mathtools}

\usepackage[normalem]{ulem}
\theoremstyle{plain}
\newtheorem{theorem}{Theorem}[section]
\newtheorem{proposition}[theorem]{Proposition}
\newtheorem{lemma}[theorem]{Lemma}

\newtheorem{definition}[theorem]{Definition}
\theoremstyle{definition}

\theoremstyle{remark}
\newtheorem{remark}[theorem]{Remark}
\usepackage[textsize=tiny]{todonotes}
\icmltitlerunning{
Probabilistic Inference in Language Models 
via Twisted Sequential Monte Carlo }

\begin{document}
\doparttoc %
\faketableofcontents

\twocolumn[
\icmltitle{ 
Probabilistic Inference in Language Models \\ via Twisted Sequential Monte Carlo}

\icmlsetsymbol{first}{*}
\icmlsetsymbol{senior}{**}

\begin{icmlauthorlist}
\icmlauthor{Stephen Zhao}{ut,vect,first}
\icmlauthor{Rob Brekelmans}{vect,first}
\icmlauthor{Alireza Makhzani}{ut,vect,senior}
\icmlauthor{Roger Grosse}{ut,vect,senior}
\end{icmlauthorlist}

\icmlaffiliation{ut}{University of Toronto}
\icmlaffiliation{vect}{Vector Institute}

\icmlcorrespondingauthor{
\{stephenzhao, makhzani, rgrosse\} @cs.toronto.edu}{}
\icmlcorrespondingauthor{rob.brekelmans@vectorinstitute.ai}
{}
\icmlkeywords{Machine Learning, ICML}

\vskip 0.3in
]

\printAffiliationsAndNotice{
\icmlEqualContribution} %

\begin{abstract}
Numerous capability and safety techniques of Large Language Models (LLMs), including RLHF, automated red-teaming, prompt engineering, and infilling, can be cast as sampling from an unnormalized target distribution defined by a given reward or potential function over the full sequence. In this work, we leverage the rich toolkit of Sequential Monte Carlo (SMC) for these probabilistic inference problems. In particular, we use learned \textit{twist functions} to estimate the expected future value of the potential at each timestep, which enables us to focus inference-time computation on promising partial sequences. We propose a novel contrastive method for learning the twist functions, and establish connections with the rich literature of soft reinforcement learning. As a complementary application of our twisted SMC framework, we present methods for evaluating the accuracy of language model inference techniques using novel \textit{bidirectional} SMC bounds on the log partition function. These bounds can be used to estimate the KL divergence between the inference and target distributions in both directions. We apply our inference evaluation techniques to show that twisted SMC is effective for sampling undesirable outputs from a pretrained model (a useful component of harmlessness training and automated red-teaming), generating reviews with varied sentiment, and performing infilling tasks.
\end{abstract}

\vheader
\vheader
\vspace*{-.2cm}
\section{Introduction}
{A wide range of language model learning and inference tasks can be viewed as steering a model's generations to satisfy a specified property.}
In particular, traditional \gls{RLHF} pipelines \citep{ziegler2019fine, stiennon2020learning, ouyang2022training, bai2022training, rafailov2023direct} may be viewed as targeting an unnormalized target modulated by a terminal reward function which reflects human feedback
\citep{korbak2022rl}. 
Red-teaming techniques such as prompt-engineering and infilling may seek target outputs with low reward or (high probability of) undesirable responses
\citep{zou2023universal, perez2022red}. 
In reasoning tasks, we may seek to target outputs which are likely to be deemed valid by a `verifier' \citep{cobbe2021training, anil2021learning, dohan2022language, hu2023amortizing}.   Specific properties of the generated responses might also be enforced \citep{khalifa2020distributional, yang2021fudge, lew2023sequential}.
We view the above tasks as instances of \emph{probabilistic inference}: 
sampling from a target unnormalized density and estimating its intractable (log) normalization constant.  Consider a pretrained base model $\base(\bs_{1:T}|\bs_0)$ which generates responses $\bs_{1:T}$ of maximum length $T$ based on a variable-length prompt $\bs_0$.   
We consider 
defining the 
target distribution of interest 
using 
the base model modulated by 
a \phishort function
$\finaltwist$
which evaluates full sequences, 
\vspace{-10pt} 

\small
\begin{align}
\target(\bs_{1:T}|\bs_0) &\coloneqq \frac{1}{\cZ_\sigma(\bs_0)} \base(\bs_{1:T}|\bs_0) \finaltwist,  \label{eq:posterior} \\
\text{where} \,\, \cZ_\sigma(\bs_0) &\coloneqq \sum_{\bs_{1:T}}\ttarget(\bs_{1:T}|\bs_0) = \sum_{\bs_{1:T}} \base(\bs_{1:T}|\bs_0) \finaltwist, \nonumber
\end{align}
\normalsize
where $\ttarget(\bs_{1:T}|\bs_0)$ denotes the unnormalized density.
We refer to $\cZ_\sigma(\bs_0)$ as the normalization constant or partition function, which is intractable due to the summation over $\bs_{1:T}$.  
We drop dependence on $\bs_0$ to avoid clutter, but note that each prompt induces a different partition function. 
In the context of the aforementioned applications, $\finaltwist$ may be derived from a human preference model (for RLHF), an indication of bad behavior (for automated red-teaming), or a verifier's prediction of correctness (for reasoning tasks).
We refer to \cref{table:posteriors} or \citet{korbak2022rl, dohan2022language, phan2023training, hu2023amortizing} for further examples and discussion of probabilistic inference in language models.

\vheader
\paragraph{Twisted Sequential Monte Carlo in Language Models}
In this work, we leverage tools from (twisted) \gls{SMC} \citep{doucet2001sequential, del2006sequential, briers2010smoothing, chopin2020introduction} to perform and evaluate inference in the language modeling setting (\cref{sec:twist_smc}).   
{A particular challenge in sampling from \cref{eq:posterior} is that the target distribution $\target(\bs_{1:T})$ is non-causal.   In order to sample tokens sequentially, one needs to infer the marginal distribution $\sigma(\bs_{1:t})
= \sum_{\bs_{t+1:T}} \sigma(\bs_{1:T})
\propto \sum_{\bs_{t+1:T}} \base(\bs_{t+1:T}|\bs_{1:t})\finaltwist$,
which 
involves an intractable marginalization.
}
To address this 
problem, we propose to learn \textit{twist functions} $\psi_t(\bs_{1:t})$ which modulate the base model such that $\base(\bs_{1:t}) \psi_t(\bs_{1:t})$ matches the target marginals $\target(\bs_{1:t})$, up to normalization. 
The twist functions can 
be used to focus each step of language model 
generation
on promising partial sequences. 
\vheader
\paragraph{Evaluating Inference in Language Modeling}
Sampling from the target distribution  
is closely intertwined with bounding the log partition function.
Similarly to variational inference or traditional RLHF objectives \citep{korbak2022rl}, \gls{SMC} algorithms yield lower bounds on $\log \cZ_\sigma$, where tighter bounds typically coincide with more accurate target sampling.
However, \textit{upper} bounds 
may often be obtained
when an exact target sample is available \citep{grosse2015sandwiching, grosse2016measuring, brekelmans2022improving}.  
The difference between upper and lower bounds on $\log \cZ_\sigma$ in fact yields an upper bound on the symmetrized KL divergence between inference samples and the target distribution \citep{grosse2016measuring}.
For these reasons, 
we argue in \cref{sec:eval} that
log partition function 
estimates are a powerful tool for evaluating language model inference techniques.

\vheader
\paragraph{Contributions} Our probabilistic inference perspective leads to the following contributions:
\begin{itemize}
\vheader
\item \textit{Twisted Sequential Monte Carlo for Language Modeling}:
We view
\textit{twisted} \textsc{SMC} as a general framework for sampling and evaluation of language models.   
While twisted \gls{SMC} is well-known and \citet{lew2023sequential} consider \gls{SMC} with fixed, few-step-ahead target information in the language modeling setting, we propose to \textit{learn} intermediate twist functions 
for target distributions defined by terminal 
\phishort only.
\vheader
\item {\textit{Contrastive Twist Learning}}: 
We develop probabilistic methods for learning intermediate twist functions, presenting 
a novel \textit{contrastive twist learning} (\textsc{CTL}) method inspired by energy-based modeling 
{and density ratio estimation }
in \cref{sec:ctl}.  Further, we adapt existing twisted SMC methods \citep{lawson2018twisted, lawson2022sixo, lioutas2022critic} to the language modeling setting, and highlight connections with inference techniques inspired by (soft) \gls{RL}.

\vheader
\item \textit{Evaluating Inference in Language Models}:  Finally, we demonstrate that twisted \textsc{SMC} provides a rich set of tools for evaluating language model fine-tuning or 
controlled generation techniques.
We propose a novel \textsc{SMC} upper bound on $\log \cZ_\sigma$ 
which is applicable when an exact target sample is available and may be of independent interest.   We leverage these bounds to evaluate the quality of inference by measuring the KL divergence to the target $\sigma(\bs_{1:T})$ in \textit{both} directions, which can be used to diagnose mode-dropping behavior of 
methods such as \gls{PPO} \citep{schulman2017proximal} which optimize a mode-seeking divergence.

\end{itemize}

We proceed to describe background on importance sampling and \textsc{SMC} in \cref{sec:background}, before presenting our framework for 
twisted \textsc{SMC} 
in the language modeling setting in \cref{sec:twist_smc}.  We propose methods to learn the twist functions in \cref{sec:learning} and methods to evaluate inference in \cref{sec:eval}.  Our experimental results in \cref{sec:experiments} showcase the ability of twisted \textsc{SMC} to improve controlled generation and lend insights into inference quality in existing methods.

\vheader
\section{Background}\label{sec:background}
\vheader

\begin{figure*}[t]
\begin{minipage}{.57\textwidth}
    \centering
    \vspace{-.085cm}
    \begin{subfigure}[b]{0.47\textwidth}
        \centering
        \includegraphics[width=\textwidth]{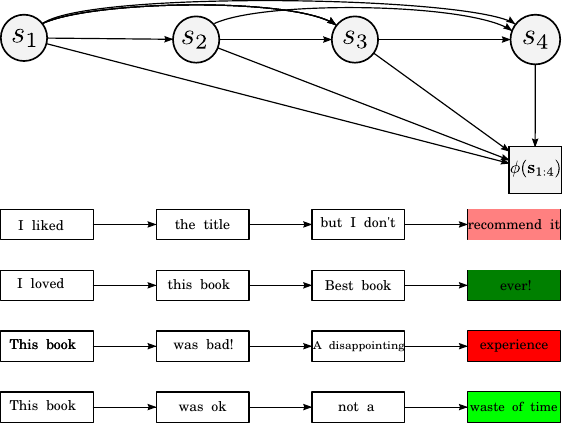}
        \caption{Simple Importance Sampling}
        \label{fig:sub1}
    \end{subfigure}%
    \hspace{.5cm} %
    \begin{subfigure}[b]{0.47\textwidth}
        \centering
        \includegraphics[width=\textwidth]{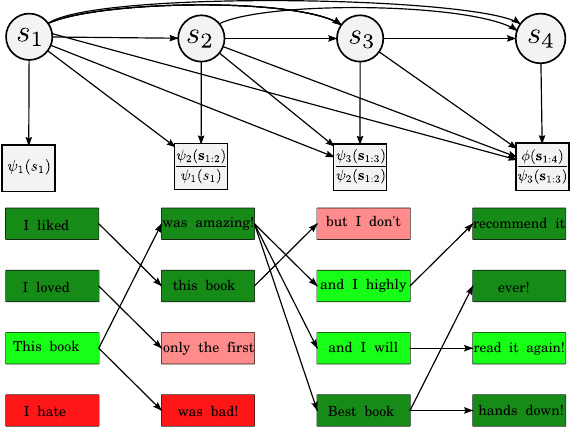}
        \caption{Twisted SMC}
        \label{fig:sub2}
    \end{subfigure}
    \captionof{figure}{Illustrative example of \textsc{SIS} and (Twisted) \textsc{SMC} for sampling book reviews conditioned on positive sentiment $\finaltwist$.  \textsc{SIS} only performs resampling after observing the entire sequence, while \textsc{SMC} can kill or clone partial sequences $\s_{1:t}$ based on incremental importance weights 
    induced by twist functions $\approxtwist{t}(\s_{1:t})$.  
    Green/red indicate high/low
    importance weights at each incremental step of SMC, or at the final step of SIS.   
    For SMC with the base model proposal $\base$ and the optimal twists, the incremental weights $\psi^*_{t}/\psi^*_{t-1}$ (\cref{alg:smc} or \cref{eq:incremental_bg}) are directly correlated with sentiment.
}\label{fig:is_smc}
\end{minipage}
\newcommand{\myspace}{.05cm}
\hfill
 \begin{minipage}{.41\textwidth}
 \vspace*{-.6cm}
\begin{algorithm}[H]
    \caption{(Twisted) SMC Sampling ($\smcprop$)}
  \begin{algorithmic}%
   \vspace*{-.1cm}
    \STATE \textbf{SMC-PROPOSAL}$\big(\base,q,\left\{ \psi_{t}\right\} _{t=1}^{T-1},\phi,K \big)$:  
    \FOR{$t = 1,...,T$} 
        \FOR{$k = 1,...,K$} \vspace{\myspace}
            \STATE Sample $s_{t}^{k}\sim q\pv{s_{t}}{\s_{1:t-1}^{k}}$ \vspace{\myspace}
            \STATE $\s_{1:t}^{k}\leftarrow\mathtt{concat}\left(\s_{1:t-1}^{k},s_{t}^{k}\right)$
            \IF{$t < T$}
                \STATE $w_t^k \leftarrow \fr{\base\pv{s_{t}^{k}}{\bs_{1:t-1}^{k}}}{q\pv{s_{t}^{k}}{\bs_{1:t-1}^{k}}}\fr{\psi_{t}\left(\bs_{1:t}^{k}\right)}{\psi_{t-1}\left(\bs_{1:t-1}^{k}\right)}$ 
            \ELSE
                \STATE $w_t^k \leftarrow \fr{\base\pv{s_{t}^{k}}{\bs_{1:t-1}^{k}}}{q\pv{s_{t}^{k}}{\bs_{1:t-1}^{k}}}\fr{\phi\left(\bs_{1:t}^{k}\right)}{\psi_{t-1}\left(\bs_{1:t-1}^{k}\right)}$             
            \ENDIF
        \ENDFOR
        \IF{$t < T$}
        \FOR{$k = 1,...,K$} \vspace{\myspace}
            \STATE $\om_{t}^{k}\sim\mathtt{cat}\Big({\Big\{ {\normalcolor \fr{w_{t}^{i}}{\sum_{j=1}^{K}w_{t}^{j}}}\Big\} _{i=1}^{K}}\Big)$
            \STATE $\s_{1:t}^{k}\leftarrow\s_{1:t}^{\om_{t}^{k}}$
        \ENDFOR     
        \ENDIF
    \ENDFOR
    \STATE \textbf{return} {$\left\{ {\normalcolor \s_{1:T}^{k},w_{T}^{k}}\right\} _{k=1}^{K}$}
    \\
    \STATE \phantom{\textbf{return}} $\hat{\cZ}_{\sigma}^{\textsc{smc}} = \prod_{t=1}^T \frac{1}{K} \sum_{k=1}^K w_t^k$
\label{alg:smc}
  \end{algorithmic}
  \end{algorithm}%
\end{minipage} 
\vheader
\vheader
\end{figure*}

 Suppose we are given access to an unnormalized density $\tilde{\si}\left(\bs_{1:T}\right)$ {which can be efficiently evaluated}. %
We focus on estimation of the partition function or normalization constant $\Z_{\si}:= \sum_{\bs_{1:T}} \tilde{\si}\left(\bs_{1:T}\right)$, since unbiased estimators with low variance yield approximate 
sampling techniques which closely approximate the target distribution \citep{finke2015extended, maddison2017filtering}.
We review \gls{SIS} and \gls{SMC} techniques in this section.

\vheader
\subsection{Simple Importance Sampling}\label{sec:simple_is}
Simple importance sampling (\gls{SIS}) provides an unbiased estimator of $\Z_{\si}$ 
by calculating importance weights for
any normalized proposal distribution $q(\bs_{1:T})$,
\small 
\begin{align}
\ssweight{i} \coloneqq \fr{\tilde{\si}\left(\bs_{1:T}^{i}\right)}{q\left(\bs_{1:T}^{i}\right)}\,, \label{eq:single_sample_w}
\end{align}
\normalsize
which is unbiased since $\Z_{\si}  =\EEE_{q\left(\bs_{1:T}\right)}\left[
\ssweight{}
\right]$. 
The importance weights also yield an an unbiased $K$-sample estimator of the partition function,
\small 
\begin{align}
\hat{\cZ}_{\si}^{\textsc{sis}}  \coloneqq \fr 1K ~ \sum_{i=1}^K
\ssweight{i}\,,\qquad \bs_{1:T}^{i}\sim {q\left(\bs_{1:T}\right)}\,.
\label{eq:single_sample_w_ali}
\end{align}
\normalsize

By normalizing the weights in \cref{eq:single_sample_w} 
over $K$ samples 
from $q(\bs_{1:T})$, 
we can obtain (biased) estimators of expectations under $\target(\bs_{1:T})$, 
\small
\begin{align}
\mathbb{E}_{\target(\bs_{1:T})}\big[ f(\bs_{1:T}) \big] 
&\approx 
\sum\limits_{k=1}^K \frac{ 
\ssweight{k}
}{\sum_{j=1}^K  
\ssweight{j}
} f(\bs_{1:T}^{k}) \label{eq:sis_expectation} 
\end{align}
\normalsize
or select an approximate target sample $\bs_{1:T}^{\target}$ 
from a categorical distribution with the self-normalized importance weights
\small 
\begin{align}
\bs_{1:T}^{\target} \gets \bs_{1:T}^{\sind}, & \qquad \quad
\sind 
\sim \texttt{cat}\left(\left\{ \fr{w\left(\s_{1:T}^{i}\right)}{\sum_{j=1}^{K}w\left(\s_{1:T}^{j}\right)}\right\} _{i=1}^{K}\right)\,. \label{eq:snis}
\end{align}
\normalsize
The quality of the approximations in \cref{eq:single_sample_w_ali}-(\ref{eq:snis}) depends crucially on how well the proposal 
$q(\bs_{1:T})$ (which may be learned, \cref{sec:proposals}) matches the target $\target(\bs_{1:T})$.  
While we discuss 
evaluation methods
in \cref{sec:eval}, note that if inference is exact (i.e., $q(\bs_{1:T})=\target(\bs_{1:T})$), then the variance of the importance weights is zero, as $\ssweight{} = \cZ_{\si} ~ $ for all $ \bs_{1:T}$. %

\vheader
\subsection{Sequential Monte Carlo}\label{sec:smc_bg}
\vheader

\gls{SMC} improves inference by decomposing it into easier subproblems involving a set of unnormalized intermediate target distributions $\{ \tpitwist{t}(\bs_{1:t}) \}_{t=1}^{T}$.
A key observation is that as long as 
$\pitwist{T}
(\bs_{1:T}) 
= \target
(\bs_{1:T})
$, we obtain an unbiased estimate of the partition function $\cZ_T=\Z_\si$, regardless of the intermediate $\pitwist{t}$
and proposal $q$.   

We begin by defining the \textit{incremental} importance weights
\begin{align}
\smcincweight{t}{}
 \coloneqq \frac{\tpitwist{t}(\bs_{1:t})}{ \tpitwist{t-1}(\bs_{1:t-1}) \prop(s_{t}|\bs_{1:t-1})}\,. \label{eq:incremental_bg}
\end{align}
\normalsize
where $\tpitwist{t}$ is the unnormalized density of $\pitwist{t} = \tpitwist{t}/\zpi{t}$.

\gls{SMC} maintains a set of $K$ partial sequences, by first sampling from the proposal $q(s_t^{k}|\bs_{1:t-1}^{k})$ in each index $k
$.
Optional 
resampling steps may be performed to 
clone sequences 
with high
incremental importance weights using
\begin{align}
\hspace*{-.1cm} \bs_{1:t}^{k} \gets \bs_{1:t}^{\sindat{k}{t}},  \,\,\,\quad \sindat{k}{t} \sim 
\mathtt{cat}\left({\bigg\{ {\normalcolor \fr{\smcincweight{t}{i}}{\sum_{j=1}^{K}\smcincweight{t}{j}}}\bigg\} _{i=1}^{K}}\right),
\label{eq:incremental_resample}
\end{align}
\normalsize
similarly to \cref{eq:snis}. 
Since resampling is performed \textit{with} replacement, sequences with high weights may be cloned multiple times.   
The 
resulting $\bs_{1:t}^{\sindat{k}{t}}$ are used as prefixes for the next step of proposal sampling 
in index $k$
(see \cref{alg:smc}).

We can show that \gls{SMC} yields an unbiased estimator $\hat{\cZ}_\sigma^{\textsc{smc}}$  of the normalization constant $\cZ_\sigma$, by considering the extended state space $\bS \coloneqq \{ s_{t}^k, \omega_t^k \}_{k,t=1}^{K,T}$ of token and index random variables from the sampling procedure $\bS \sim \smcprop(\bS)$ in \cref{alg:smc}.
Assuming resampling at every step,\footnote{
The decision to resample may be based on an adaptive condition such as \gls{ESS} \citep{chopin2020introduction}.   For $R \leq T$, let $\left\{ t_{r}\right\} _{r=1}^{R}$ index times where resampling occurs and fix $t_{0}=0$ and $t_{R}=T$. The estimator becomes $
\hat{\cZ}_\sigma^{\textsc{smc}} = \prod_{r=1}^{R}\fr 1K\sum_{i=1}^{K}\left(\prod_{t=t_{r-1}+1}^{t_{r}}w_{t}\left(\s_{1:t}^{i}\right)\right) $,
and the final-step weights for expectations in \cref{eq:sis_expectation} or sampling in \cref{eq:snis} are 
given by 
$\prod_{t=t_{R-1}+1}^{T}w_{t}\left(\s_{1:t}^{i}\right)$.\label{footnote:smc_bound}} 
\begin{align}
\hspace*{-.22cm}
 \cZ_{\sigma} = 
 \mathbb{E}_{}\left[ \hat{\cZ}^{\textsc{smc}}_\sigma \right] 
 &= \EEE_{\smcprop(\bS)}\left[\prod_{t=1}^{T}\fr 1K\sum_{k=1}^{K}w_t\left(\bs_{1:t}^{k}\right)\right] .
\label{eq:smc_ess} 
\end{align}
\normalsize
To see that 
$\hat{\cZ}^{\textsc{smc}}_\sigma$
is unbiased, 
we can view
\cref{eq:smc_ess} as performing simple importance sampling $\cZ_\sigma = \mathbb{E}_{\smcprop( \bS) }\left[ \frac{\tsmctgt(\bS)
}{\smcprop(\bS)
} \right]$ in the extended state space, for appropriate definitions of $\smctgt(\bS)$ and $\smcprop(\bS)$ detailed in \cref{app:smc} or \citep{andrieu2010particle, maddison2017filtering}.   Intuitively, we may view the average incremental importance weights at each step as estimating the partition function ratio $
\zpi{t}/\zpi{t-1}
\approx \frac{1}{K} \sum_{k=1}^K \smcincweight{t}{k}$.  \cref{eq:smc_ess}
 composes intermediate partition function ratio estimators to obtain an estimate of the final $\cZ_T = \cZ_\sigma = \prod_{t=1}^T 
 \zpi{t}/\zpi{t-1}
 $, with $\zpi{0}=1$.

With no resampling, \gls{SMC} reduces to \gls{SIS} with target $\sigma(\bs_{1:T}) = \pi_T(\bs_{1:T})$ and proposal $q(\bs_{1:T})$.  {Using the final-step \gls{SMC} weights, we may estimate expectations or draw approximate samples $\bs_{1:T}^\sigma$ as in \cref{eq:sis_expectation}-(\ref{eq:snis}).}

\cref{fig:is_smc} illustrates the key advantage of \gls{SMC} resampling over \gls{SIS}.   
While a suboptimal $q(\bs_{1:T})$ may produce sequences with low probability under the target $\sigma(\bs_{1:T})$, 
\gls{SMC} resampling with well-chosen intermediate targets $\pitwist{t}$ 
clones 
the most promising partial sequences $\bs_{1:t}$ at step $t$.  Since later sampling proceeds from these prefixes,  we 
expect to 
obtain final sequences which better cover the high-probability regions of the target distribution.   {We discuss techniques to evaluate the quality of \gls{SMC} or \gls{SIS} sampling in \cref{sec:eval}.}

\vheader
\section{Twisted Sequential Monte Carlo for Language Modeling}\label{sec:twist_smc}
\vheader

A key design choice in the \gls{SMC} procedure 
above is the 
intermediate targets 
$\{ \pitwist{t} \}_{t=1}^{T-1}$, where we assume $\pitwist{T}(\s_{1:T})=\si(\s_{1:T})$ is always the target distribution.
In state-space models with observation likelihoods or environments with intermediate rewards,
\textit{filtering} \gls{SMC} 
considers
target information collected from times 
$\tau \leq t$
to define $\pi_t$.
\citep{chopin2020introduction}.  %
Previous work on \gls{SMC} for language models \citep{lew2023sequential} has considered per-token or few-step-ahead statistics to define tractable 
intermediate
$\pitwist{t}$. 
However, we are often interested in target distributions which are determined by a {terminal} 
\phishort
$\finaltwist$ only, as in \cref{eq:posterior}.

In such settings, \textit{twisted} \gls{SMC} methods \citep{briers2010smoothing, whiteley2014twisted, lawson2022sixo} consider the \textit{full} 
target information 
(until time $T$) to define $\{\pi_t\}_{t=1}^{T-1}$.
In other words, our desired intermediate targets 
are
the true marginals $\sigma(\bs_{1:t})$ of the target distribution.
Intuitively, note that in order to exactly sample $\bs_{1:T} \sim \sigma(\bs_{1:T})$, we need to ensure partial sequences are distributed according to the intermediate marginals $\bs_{1:t} \sim \sigma(\bs_{1:t})$.
In 
\cref{sec:twist_fns},
we will 
represent the intermediate targets $\{\pi_t\}_{t=1}^{T-1}$
using
\textit{twist} functions $\approxtwist{t}\colon 
\bs_{1:t} 
\rightarrow \bbR$ which modulate the base model to (approximately) match the target marginals, thereby summarizing future 
information
relevant to sampling at time~$t$.

\vheader
\subsection{Twist Functions}\label{sec:twist_fns}
\vheader

We represent the intermediate target distributions $\{\pi_t\}_{t=1}^{T-1}$ for \gls{SMC} sampling using the following general form.

\begin{definition}[\propheader{ Twisted (Intermediate) Targets }]\label{def:twist_targets}
Using approximate twist functions $\{ \approxtwist{t}\}_{t=1}^{T-1}$ and the final target 
$\finaltwistonly$, we define the twisted intermediate target distributions
\begin{align}
\hspace*{-.25cm}
\begin{split}
\pitwist{t}(\bs_{1:t}\condzero) = \begin{cases} 
\frac{1}{\zfilter{t}}~  p_0(\bs_{1:t}\condzero) ~
\approxtwist{t}(\bs_{1:t}) \qquad  \hfill t \neq T\\
\frac{1}{\cZ_\sigma} ~ p_0(\bs_{1:T}\condzero)~ 
\finaltwist
\hfill t= T
\end{cases}
\end{split} \label{eq:filtering} \raisetag{20pt}
\end{align}
\normalsize
\end{definition}
For an arbitrary proposal $q$ and the unnormalized targets in \cref{eq:filtering}, 
the incremental importance weights 
are given by
\begin{align}
 \smcincweight{t}{}
 &= \frac{\base(s_{t}|\bs_{1:t-1})}{\prop(s_{t}|\bs_{1:t-1})}\frac{\approxtwist{t}(\bs_{1:t})}{ \approxtwist{t-1}(\bs_{1:t-1})} . \label{eq:incremental_twisted}
\end{align}
\normalsize

While uninformed twist functions $\approxtwist{t}$ may result in $\pitwist{t}(\bs_{1:t}\condzero)$ which are no closer to the target marginal $\target(\bs_{1:t}\condzero)$ than the base model $\base(\bs_{1:t}\condzero)$  (for example, in early stages of learning), the crucial fact is that our final target distribution in \cref{eq:filtering} reflects the 
target 
\phishort
$\finaltwist$.   As in \cref{sec:smc_bg}, this ensures that, regardless of the intermediate twists, our resulting importance sampling estimators 
will be unbiased.

Finally, the optimal twists $\optimaltwist{t}(\bs_{1:t})$ recover
the intermediate marginals $\pi_t^*(\bs_{1:t}) = \sigma(\bs_{1:t})$ of the target distribution. 
We state the sense in which $\pi_t^*$ and $\psi_t^*$ are optimal in \cref{app:optimality}, and prove 
the 
following proposition in \cref{app:framework} \cref{prop:opt_twists_int_reward}.

\begin{proposition}[\propheader{Optimal Twists}]\label{prop:optimal_twists}
For a given target distribution  $\target(\bs_{1:T}\condzero)$ in \cref{eq:posterior}, 
the optimal twist functions $\optimaltwist{t}(\bs_{1:t})$
(in regions where $\base(\bs_{1:t})>0$)
correspond to
\begin{align}
\pitwist{t}^*(\bs_{1:t}) = 
\target(\bs_{1:t}) 
&= \frac{1}{\zfilterstar{t}}~  p_0(\bs_{1:t}\condzero) ~
\optimaltwist{t}(\bs_{1:t}).
 \label{eq:optimal_filtering} 
\end{align}
Up to a constant independent of $\bs_{1:t}$, the optimal twists are
\begin{align}
 \optimaltwist{t}(\bs_{1:t}) &
 \propto 
 \sum_{\bs_{t+1:T}} \base(\bs_{t+1:T}|\bs_{1:t}) \finaltwist .
 \label{eq:optimal_twists} 
\end{align}
and 
satisfy the recursion
\begin{align}
\psi_{t}^{*}\left(\bs_{1:t}\right) \propto \sum \limits_{s_{t+1}}\base\pv{s_{t+1}}{\s_{1:t}}\psi_{t}^{*}\left(\bs_{1:t+1}\right).
\label{eq:psi_recursion}
\end{align}
\end{proposition}
Since the optimal twist functions are unavailable due to the need to marginalize over future timesteps, we consider learning approximate twist functions using methods in \cref{sec:learning}.

\vheader
\subsection{Proposal Distribution}\label{sec:proposals}
\vheader
For 
a 
given 
set of 
targets $\{\pi_t \}_{t=1}^T$,
the importance weights in \cref{eq:incremental_twisted} depend crucially on the choice of proposal.

\vheader
\paragraph{Base Model as Proposal}
The most straightforward choice of proposal is the base pre-trained model, $q = \base$.     While we demonstrate in \cref{sec:experiments} that \gls{SMC} resampling with learned twists and the base model proposal can %
closely approximate the target distribution, this may require large $K$.
We can achieve greater efficiency using better choices of proposal.

\paragraph{Twist-Induced Proposal}
For given targets $\{\pi_t \}_{t=1}^T$, the optimal proposal minimizes the variance of the importance weights (\cref{app:optimality}). %
In the language model setting with a terminal \phishort only, 
we will 
in fact 
be able to sample from the optimal proposal for the one-step importance weights.

\vspace*{.05cm}
\begin{restatable}{proposition}{twistinduced}
\propheader{(Twist-Induced Proposal).}\label{prop:transition}    
    For a given set of intermediate 
    twisted targets $\pitwist{t}(\bs_{1:t})$ in \cref{eq:filtering}, the proposal which minimizes the variance of the one-step incremental importance weights 
    $w_t$
    is given by
\begin{align}
\proptwist{t}{(s_{t}|\bs_{1:t-1})} &\propto \frac{\pitwist{t}(\bs_{1:t})}{\pitwist{t-1}(\bs_{1:t-1})} \label{eq:transition} 
\\&
= \frac{1}{\zonestep{t}{t-1}} \base(s_{t}|\bs_{1:t-1}) \approxtwist{t}(\bs_{1:t}). \nonumber
\end{align}
\normalsize
\end{restatable}
{See \cref{app:proof-twist-ind-prop} for proof.}
For $t < T$, we can construct a parameterization of $\approxtwist{t}(\bs_{1:t})$ such that the proposal is tractable to sample in transformer architectures, where 
the normalization $\zonestep{t}{t-1} 
= \sum_{s_{t}}\base(s_{t}|\bs_{1:t-1}) \approxtwist{t}(\bs_{1:t})
$ sums over the discrete vocabulary of next tokens $s_t \in \cV$.   
However, for the final timestep, 
 note that $\finaltwist$ may require calls to a different neural network such as a reward model or classifier.  We thus consider an approximate $\approxfinaltwist(\bs_{1:T}) \approx \finaltwist$ for the proposal
 ${q}_T(s_T|\bs_{1:T-1}) \propto \base(s_{T}|\bs_{1:T-1}) \approxfinaltwist(\bs_{1:T})$
 in the final step.  
 With slight abuse of notation, we let $\proptwist{}
 (\bs_{1:T})
 $ denote this tractable proposal over full sequences,
 \begin{equation}
 \hspace*{-.22cm}
 \proptwist{}(\bs_{1:T}) \coloneqq \Big( \prod_{t=1}^{T-1} \proptwist{t}{(s_{t}|\bs_{1:t-1})} \Big) 
 ~ {q}_T(s_T|\bs_{1:T-1})\,.
 \label{eq:twist_prop}
 \end{equation}
Using this proposal, the incremental weights become
\begin{equation}
\hspace*{-.22cm}
\resizebox{.915\columnwidth}{!}{\ensuremath{
\smcincweight{t}{} = \begin{cases} 
\dfrac{\sum_{s_{t}}\base(s_{t}|\bs_{1:t-1}) \approxtwist{t}(\bs_{1:t}) }{ \approxtwist{t-1}(\bs_{1:t-1}) } 
\,\,\, \hfill t < T \\[3.5ex]
\dfrac{
\sum_{s_{T}}\base(s_{T}|\bs_{1:T-1}) \approxtwist{T}(\bs_{1:T})
}{\approxtwist{T-1}(\bs_{1:T-1})} \dfrac{\finaltwist}{\approxfinaltwist(\bs_{1:T})} \quad \hfill t = T 
\end{cases}\,,
}}\label{eq:locally_normalized_is_weights} \raisetag{.5cm}
\end{equation}
\normalsize
which are independent of $s_{t}$ for $t < T$.

\vheader
\paragraph{Variational Proposal}
As noted in \cref{sec:simple_is}, \gls{SMC} with no %
resampling steps reduces to \gls{SIS} with the full target distribution $\target(\bs_{1:T})$.
Policy gradient methods \citep{schulman2017proximal, parshakova2019distributional, korbak2022controlling, go2023aligning} which 
directly learn a tractable approximation $q(\bs_{1:T})$ to the target distribution may thus be viewed 
as a particularly simple instance of \gls{SMC}, or inference more generally (see \citet{korbak2022rl}).   
We may also evaluate these inference methods using %
our 
proposed tools in \cref{sec:eval}.    
See \cref{tab:losses} and \cref{app:prop} for detailed losses and discussion.

Finally, note that a separate proposal $q$ might also be learned alongside the twisting targets $\{\pi_t \}_{t=1}^{T-1}$.   
This may be useful to approximate the variance-minimizing proposal for multi-step or adaptive resampling (\cref{prop:multi-step-optimal-proposal}) beyond the tractable optimal one-step proposal in \cref{prop:transition}.  We discuss training losses based on multi-step importance weights in \cref{sec:consistency_twist}.

\vheader
\subsection{{Conditional Target Distributions}}\label{sec:conditional_twist}
\vheader
More generally, we may consider \textit{conditional} target distributions, obtained by conditioning on an observation random variable $o_T$.   This mirrors 
the standard setting of \gls{SMC} in state-space models \citep{doucet2001sequential, briers2010smoothing, gu2015neural, maddison2017filtering, lawson2022sixo}, with further discussion in \cref{app:conditional_twist_expo}.  

Defining $\phi(\bs_{1:T},o_T) =\sigmacond{T}(\myo_T|\bs_{1:T}) $ as a probabilistic model of $o_T$, our target distribution is the posterior $\sigma(\bs_{1:T}|o_T)$,
\begin{align}
\begin{split}
\sigma(\bs_{1:T}|o_{T}) 
&= \frac{1}{\cZ_\sigma(\myo_{T})} \base(\bs_{1:T})\sigmacond{T}(\myo_T |\bs_{1:T})~,
\end{split}
\label{eq:conditional_twist_def}
\end{align}
where the partition function $\cZ_\sigma(\myo_{T}) = \sigma(\myo_T) = \sum_{\bs_{1:T}} \base(\bs_{1:T})\sigmacond{T}(\myo_T |\bs_{1:T})$
is the marginal of the given $o_T$.

In this setting,  \cref{prop:optimal_twists} suggests that the optimal twists, which match the marginals $\sigmacond{t}(\bs_{1:t}|o_T)$, 
correspond to the conditional likelihood of 
$o_T$ given $\bs_{1:t}$,%
\begin{align}
\begin{split}
\optimaltwist{t}(\bs_{1:t}, o_T) &\propto \sum\limits_{\bs_{t+1:T}} \base(\bs_{t+1:T}|\bs_{1:t}) \phi(\bs_{1:T}, o_T) \\
&= \sigma(o_T|\bs_{1:t}) ~,
\end{split}
\label{eq:cond_lkd_main}
\end{align}
since $\sigma(o_T|\bs_{1:t}) = \sum_{\bs_{t+1:T}} \sigma(o_T, \bs_{t+1:T}|\bs_{1:t})$.
We can proceed to construct intermediate target distributions and proposals as in the previous sections, where $\psi_t(\bs_{1:t}, o_T)$ 
and even $q_t(s_t|\bs_{1:t-1}, o_T)$ 
may be conditioned on a particular value of $o_T$. 

To recover the unconditional setting, we can fix a binary observational variable $\sigma(o_T=1|\bs_{1:T}) \coloneqq \finaltwist$ \citep{levine2018reinforcement} and omit explicit conditioning, showing that conditional twist learning 
generalizes our previous exposition.\footnote{{To obtain a probabilistic interpretation for $ \sigma(o_T=1|\bs_{1:T}) = \finaltwist$, note we need to ensure $\finaltwist \in [0,1]$.
As a result, 
sampling from the target $\sigma(\bs_{1:T}| o_T=1)$ or joint $\sigma(\bs_{1:T}, o_T=1)$ is no easier 
in this interpretation 
than in \cref{eq:posterior}, which is intractable in general.   For example, finding $\phi_{\text{max}} = \max_{\bs_{1:T}} \finaltwist$ 
and dividing $\finaltwist \leftarrow \finaltwist/\phi_{\text{max}}$ 
to rescale $\sigma(o_T=1|\bs_{1:T})$
is equivalent to being able to
perform rejection sampling 
with the base model proposal $\base(\bs_{1:T})$ (see \cref{sec:positive_ebm})}.}

\vheader
\paragraph{Exact Target Sampling on Simulated Data} 
Assuming $\sigma(o_T|\bs_{1:T})$ is tractable to sample, we may obtain an exact sample from the target posterior for simulated $o_T$ using ancestral sampling.  In particular, by sampling $\bs_{1:T}, 
o_T 
\sim \base(\bs_{1:T})\sigma(o_T|\bs_{1:T})$,
we obtain a sample from the joint distribution, which also factorizes as $\sigma(o_T, \bs_{1:T}) = \sigma(o_T) \sigma( \bs_{1:T}|o_T)$.    Using the latter factorization, we may interpret $\bs_{1:T}$ as an exact sample from the target posterior for the given $o_T$.

We refer to this
as the Bidirectional Monte Carlo (BDMC) trick \citep{grosse2015sandwiching, grosse2016measuring}, and will use it to draw exact samples for training in \cref{sec:positive_ebm} or evaluation in \cref{sec:eval}.

\vheader
\subsection{Connections with Reinforcement Learning}\label{sec:soft_rl_connection}
\vheader
Twisted \gls{SMC} shares close connections with (soft) reinforcement learning \citep{levine2018reinforcement, piche2018probabilistic, lawson2018twisted, heng2020controlled}, which we develop with detailed discussion in \cref{app:soft_rl} and \cref{app:mudgal_decoding_all}.
In particular, we consider language modeling as a \gls{MDP} with states $x_t \coloneqq \bs_{1:t-1}$, actions $a_t \coloneqq s_t$, and deterministic transitions $p(x_{t+1}|x_t, a_t)= 
\delta(\bs_{1:t} = \mathtt{concat}({s_{t}, \bs_{1:t-1})})$.
We describe two different definitions of the reward function in relation to the 
\phishort
function $\finaltwist$ below.   In \cref{app:int_reward}, we further extend our \gls{SMC} framework to capture settings with intermediate 
\phishorts
$\phi_t(\bs_{1:t})$ or rewards over partial sequences.
\vheader
\paragraph{Base Model Policy Evaluation} 
Viewing the final 
\phishort
$\finaltwist$ as the reward function, 
the optimality condition $\optimaltwist{t}(\bs_{1:t}) = \sum_{\bs_{t+1:T}} \base(\bs_{t+1:T}|\bs_{1:t}) \finaltwist$ in \cref{eq:optimal_twists}
corresponds to exact \textit{policy evaluation} of the future reward under the \textit{fixed} base model policy $\base(\bs_{t+1:T}|\bs_{1:t})$. 
\citet{mudgal2023controlled} adopt this perspective for 
controlled
decoding, and refer to the twist functions as `prefix scorers'.

\paragraph{Soft RL with KL Regularization}
Alternatively, we may consider the 
 soft or KL-regularized \gls{RL} target distributions commonly used in language modeling \citep{levine2018reinforcement, korbak2022rl} as a special case of our twisted \gls{SMC} framework.
For a regularization strength $\beta$, define the terminal 
\phishort
as 
\begin{align}
 \finaltwist = e^{\beta r(\bs_{1:T})} .\label{eq:rl_phi}
\end{align}
In this case, the intermediate twist functions in \cref{def:twist_targets} correspond to state-action $Q$-values, $\psi_t(\bs_{1:t}) = e^{\beta Q(s_t, \bs_{1:t-1})}$ (\cref{app:soft_rl}).   
In particular, consider the recursion for the optimal twists in \cref{eq:psi_recursion}.  
Taking the logarithm of both sides and recalling the definition of the soft value function $V^*(\bs_{1:t})$ \citep{levine2018reinforcement}, we obtain
\begin{equation}
\hspace*{-.2cm}
\resizebox{.9\columnwidth}{!}{\ensuremath{
Q^*(s_t, \bs_{1:t-1}) =  \underbrace{ \frac{1}{\beta} \log \sum_{s_{t+1}} \base(s_{t+1}|\bs_{1:t} )e^{\beta Q^*(s_{t+1}, \bs_{1:t})}}_{V^*(\bs_{1:t})}, 
}}
\label{eq:q_bellman_main}
\end{equation}
\normalsize
which is a soft Bellman recursion with no intermediate reward.
From the soft \gls{RL} perspective, the twist functions are analogous to a critic, while the proposal plays the role of an actor \citep{levine2018reinforcement, haarnoja2018soft}.  We provide detailed discussion of the soft RL case in \cref{app:soft_rl}, and review RL-inspired losses for twist learning in \cref{sec:consistency_twist}.

\vheader
\paragraph{Benefits of the Probabilistic Perspective}   
While soft RL is a natural special case of our framework which gives intuition for the role of the twist functions,
our approach allows for general target distributions without reference to \gls{RL} objectives and suggests principled probabilistic resampling using \gls{SMC}.    
Further, 
we develop
twist learning techniques 
inspired by 
density ratio estimation, 
including our novel CTL method or the \textsc{SIXO} objective from \citep{lawson2022sixo}, 
which are more naturally motivated from a probabilistic perspective.
Finally, we 
leverage
our probabilistic perspective to propose novel language model evaluation techniques inspired by Bidirectional Monte Carlo (\citet{grosse2015sandwiching, grosse2016measuring}) in \cref{sec:eval}.

\vheader
\section{Learning the Twist Functions}\label{sec:learning}
\vheader
We next consider methods to learn twist functions $\vartwist{t}$ parameterized by neural networks, presenting a novel 
\gls{CTL}
approach 
in \cref{sec:ebm_roger}.  
We summarize twist learning methods from related work in \cref{sec:related_twists}.

\vheader
\subsection{Contrastive Twist Learning} 
\label{sec:ebm_roger}\label{sec:ctl}
\vheader
To match our approximate %
$\pitwisttheta{t}$ to the target marginals, we propose to 
minimize $T$ separate KL divergences,
\begin{align}
 \min\limits_{\thb} \cL_{\text{CTL}}(\thb) \coloneqq \min\limits_{\thb}\sum_{t=1}^{T}\DKL\pV{\target(\bs_{1:t})}{\pitwisttheta{t}\left(\bs_{1:t}\right)}\,.
 \label{eq:separate_ebms}
\end{align}
\normalsize
While other divergences could be used to learn $\pitwisttheta{t}(\bs_{1:t})$, we argue that the mass-covering behavior of \cref{eq:separate_ebms} is a desirable property for twist learning.  Since we 
separately match
each 
$\sigma(\bs_{1:t})$, 
our hope is that
suboptimal learning in early timesteps does not lead to aggressive pruning of partial sequences that would achieve high final target likelihood.

Using \cref{eq:filtering}, the gradient of \cref{eq:separate_ebms} at each $t$ becomes %
{%
\begin{equation}
\hspace*{-.2cm}
\resizebox{.90\columnwidth}{!}{\ensuremath{ 
\EEE_{\target(\bs_{1:t})}\left[\gr_{\thb}\log\vartwist{t}\left(\bs_{1:t}\right)\right]-\EEE_{\pitwisttheta{t}\left(\bs_{1:t}\right)}\left[\gr_{\thb}\log\vartwist{t}\left(\bs_{1:t}\right)\right]
}},
\label{eq:grad_separate_ebm}
\end{equation}
}%
which allows us to learn from exact target samples of $\si(\s_{1:t})$ in the first term when they are available.

We note the similarity of the objective in \cref{eq:separate_ebms} and gradient in \cref{eq:grad_separate_ebm} to maximum likelihood training of \gls{EBM}s.  Due to the form of the gradient update, we refer to this method as \textit{contrastive twist learning} (\gls{CTL}).   
We proceed to describe approximate techniques for positive sampling  (first term) and negative sampling 
 (second term) in the next subsections.

\vheader
\subsubsection{Approximate Negative Sampling}\label{sec:negative_ebm}  
\vheader
A 
common challenge in energy-based modeling is that the second term in %
\cref{eq:grad_separate_ebm} 
involves sampling from the target $\pitwist{t}$ with intractable normalization constant $\zfilter{t}$. %
We proceed to estimate the expectation
using 
\gls{SIS}
as in \cref{eq:sis_expectation}, using a proposal $\prop(\bs_{1:t})$ such as the base model or the twist-induced proposal from \cref{sec:proposals}. 
{Note that \gls{SMC} resampling with learned intermediate twist functions could also be used.}

\begin{table*}[t]
\newcommand{\samplingdistfig}{\samplingdist}
\newcommand{\linesep}{\\[1.5ex]}
\vheader
\label{table:losses}
\centering
\resizebox{\textwidth}{!}{
\footnotesize
\begin{tabular}{clll}
\toprule
  \textbf{\textit{Name}} %
  & \textbf{\textit{Loss}} &  
  & \textbf{\textit{Learning Principle}} %
  \\ \toprule
  CTL & 
  \multicolumn{2}{l}{$\sum_{t=1}^T ~ \mathbb{E}_{\samplingdist(o_T)} \Big[ \DKL\pV{\sigma(\bs_{1:t}|o_T)}{\pitwisttheta{t}(\bs_{1:t}|o_T)} \Big] 
  $  \qquad %
  \textit{({Gradient:})}~ $\EEE_{\samplingdist(o_T)}\left[\EEE_{\target(\bs_{1:t}|o_T)}\left[\gr_{\thb}\log\vartwist{t}\left(\bs_{1:t}, o_T\right)\right]-\EEE_{\pitwisttheta{t}\left(\bs_{1:t}|o_T\right)}\left[\gr_{\thb}\log\vartwist{t}\left(\bs_{1:t}, o_T\right)\right]\right]$
  }
  &  
  Marginal Matching with MLE \linesep
  RL & 
 \multicolumn{2}{l}{$
\sum_{t=1}^{T-1} \mathbb{E}_{
\samplingdistfig(\bs_{1:t}, o_T)
}\Big[ \Big( \log \sum_{s_{t+1}} \base(s_{t+1}|\bs_{1:t}) \text{sg}\big(\vartwist{t+1}(\bs_{1:t+1},o_T) \big) -  \log \vartwist{t}(\bs_{1:t}, o_T)    \Big)^2 \Big] + \mathbb{E}_{
\samplingdistfig(\bs_{1:T}, o_T)
} \left[ \left( \log \phi(\bs_{1:T},o_T) - \log \vartwist{T}(\bs_{1:T},o_T)  \right)^2 \right]$  }
& Twist Consistency / Soft Q-Learning
\linesep
  SIXO & 
  \multicolumn{2}{l}{$\sum_{t=1}^T ~ \mathbb{E}_{\samplingdist(o_T)\target(\bs_{1:t}\condzero| o_T)}\left[ \log \texttt{sigmoid}\big( \log \vartwist{t}(\bs_{1:t}, o_T) \big)\right] + \mathbb{E}_{\base(\bs_{1:t}\condzero)\samplingdist(o_T)}\left[ \log \big(1- \texttt{sigmoid}\big(\log \vartwist{t}(\bs_{1:t}, o_T)\big)\big)\right]$}
  & Noise Contrastive Estimation
  \linesep
  FUDGE & 
  $\sum \limits_{t=1}^T 
 - 
 \mathbb{E}_{\samplingdistfig(\bs_{1:t}, o_T)}\mathbb{E}_{\base(\bs_{t+1:T}|\bs_{1:t})}
 \Big[ \sigmacond{T}(o_T |\bs_{1:T}) \log 
 \vartwist{t}(\bs_{1:t},o_T)
+
\Big(1-\sigmacond{T}(o_T|\bs_{1:T}) \Big)
\log \Big(1- \vartwist{t}(\bs_{1:t}, o_T)\Big)
\Big)
\Big]$ & 
  & Binary Classification 
  \\  \midrule
  DPG &  $ \mathbb{E}_{\samplingdistfig(o_T)}\Big[\DKL\pV{\sigma(\bs_{1:T}|o_T)}{\prop^{\qparam}(\bs_{1:T}|o_T)} \Big]$ &
  & Maximum Likelihood (MLE) %
  \linesep 
  PPO &  $ \mathbb{E}_{\samplingdistfig(o_T)}\Big[\DKL\pV{\prop^{\qparam}(\bs_{1:T}|o_T)}{\sigma(\bs_{1:T}|o_T)}\Big] $ &
  & Variational Inference %
  \\
 \bottomrule
\end{tabular}
}\vspace{-.2cm}\caption{Losses for twist (top) and proposal (bottom) learning, 
where $\samplingdist(\cdot)$ indicates an arbitrary sampling distribution.}   
\label{tab:losses}
\vspace{-.25cm}
\normalsize
\end{table*}

\vheader
\subsubsection{(Approximate) Positive Sampling}\label{sec:positive_ebm}
\vheader
In contrast to traditional \gls{EBM} settings, we do not necessarily have exact samples available from a `data' distribution.
We describe several settings related to availability of positive samples, which are explored in our experiments in \cref{sec:experiments}.

\vheader
\paragraph{Exact Target Samples}
If exact posterior samples are available, for example using the BDMC trick in \cref{sec:conditional_twist}, we may use them directly in the gradient update in 
\cref{eq:grad_separate_ebm}.

\vheader
\paragraph{Rejection Sampling}
Rejection sampling can yield exact target samples $\bs_{1:T}^{\target}$ when an upper bound on the likelihood ratio $\frac{\ttarget(\bs_{1:T})}{q(\bs_{1:T})} \leq M$ is known.   
In cases where the target $\ttarget(\bs_{1:T})$ is defined by thresholding or an indicator function $\base(\bs_{1:T}) \bbI(\bs_{1:t} \in \cC)$ or joint distribution $\base(\bs_{1:T}) \sigma(o_T|\bs_{1:T})$, 
we can clearly take $M=1$
for 
the base model proposal $\base(\bs_{1:T})$.  If the base model yields posterior samples in reasonable time, we can obtain exact samples for training using rejection sampling, and use our twist learning procedures to greatly improve sampling efficiency at generation time. 

While an improved proposal $q$ should more efficiently draw samples meeting the target conditions, exact rejection sampling 
would require
estimation of $M$.   Approximate or quasi rejection sampling might be used in this case, as analysed in \citet{eikema2022approximate}.

\vheader
\paragraph{Approximate Positive Sampling using SIS or SMC}
In cases where exact samples are unavailable and rejection sampling is inefficient or inexact, we leverage \gls{SMC} sampling with twist targets $\{\pitwisttheta{t}\}_{t=1}^T$ and any proposal $q(\bs_{1:T})$ to first draw a set of $K$ full sequences $\bs_{1:T}^{1:K}$.   
As in \cref{eq:sis_expectation}, we can use the normalized \gls{SMC} weights since the last resampling step to
estimate the expected gradient in the first term of \cref{eq:grad_separate_ebm}.
Without resampling, we recover \gls{SIS} estimation. 

While both our approximate positive and negative sampling for estimating the expectations in \cref{eq:grad_separate_ebm} rely on \gls{SMC} or \gls{SIS} weights (often with the same proposal), the crucial distinction is that weights for \textit{positive} sampling are based on the \textit{true target \phishort}$\finaltwist$ over \textit{full} sequences. %

\vheader
\paragraph{Truncation to Partial Sequences} 
For an exact positive sample,
we use its truncation to a partial sequence of length $t$ (which corresponds to a sample from the desired marginal $\sigma_t$) to perform the gradient update in \cref{eq:grad_separate_ebm}. 
For approximate positive sampling, we use the same set of $K$ final 
weights to estimate the expected
gradient 
at each timestep.

\vheader
\subsection{Twist Learning Methods from Related Work}\label{sec:related_twists}
\vheader
We briefly describe alternative approaches for twist learning, 
with detailed discussion in \cref{app:twist} and a summary of the loss functions for methods used in our experiments in \cref{tab:losses}. 

\vheader
\paragraph{Soft Q-Learning (RL)}
Enforcing the recursion in \cref{eq:psi_recursion} using a squared error loss is analogous to soft $Q$-learning in the \gls{RL} literature (see \cref{eq:q_bellman_main}), and has been used for twisted SMC in \citet{lioutas2022critic}.
\citet{mudgal2023controlled} derive a similar squared-error loss, viewing $\finaltwist$ as the reward.
Finally, we interpret path consistency losses \citep{nachum2017bridging}, which were derived in the soft \gls{RL} setting 
and have been used for language modeling in \citet{guo2021efficient, hu2023amortizing}, from an importance sampling 
perspective in \cref{sec:consistency_twist} and \ref{app:pcl}.

\vheader
\paragraph{SIXO} The \textsc{SIXO} loss proposed by \citet{lawson2022sixo} 
learns twist functions using a binary classification task to distinguish samples from the target marginal $\sigma(\bs_{1:t}|o_T)$ and base model $\base(\bs_{1:t})$ at each step, which {corresponds to noise contrastive estimation \citep{gutmann2010noise} for learning energy-based models}.   See \cref{app:sixo}.

\vheader
\paragraph{FUDGE }
\citet{yang2021fudge} learn twists by constructing a 
binary classification task 
to instead learn the conditional likelihood
$\sigma(o_T|\bs_{1:t})$ (\cref{eq:cond_lkd_main}).
{This may be viewed as enforcing the $T-t$ step optimality equation in \cref{eq:optimal_twists} or \cref{eq:cond_lkd_main}, where 
rollouts should be obtained using the base model $\base(\bs_{t+1:T}|\bs_{1:t})$ (see \cref{tab:losses} or \cref{app:fudge}).}  
{\citet{mudgal2023controlled, deng2023reward} similarly propose to enforce the $T-t$ step optimality condition 
using
a squared-error loss, 
$\sum_t \mathbb{E}_{\base(\bs_{t+1:T}|\bs_{1:t})}[(\finaltwist - \psi_t(\bs_{1:t}))^2]$.   }
\newcommand{\infevaldist}{q}
\vheader
\section{Evaluating Inference in Language Models}\label{sec:eval}
\vheader
 Our \gls{SMC} framework yields a rich set of tools for evaluating inference techniques in language models, using well-studied quantities such as the log partition function $\log \cZ_\sigma$ and \textsc{KL} divergence {to the target distribution.}
Remarkably, with access to a single exact sample from the target distribution, we show in \cref{prop:smc_proposition} that we can obtain \textit{upper} bounds on 
$\log \cZ_\sigma$ in addition to 
lower bounds.   
These bounds can tightly sandwich $\log \cZ_\sigma$ with increasing $K$, thereby 
ensuring 
reliable conclusions regarding inference quality.
\vheader
\subsection{Applications of $\log \cZ_\sigma$ Estimation}\label{sec:eval_applications}
\paragraph{Evaluating Fine-Tuned Models}
To motivate this section and present an important application of our \gls{SMC} methods, consider evaluating how well a given $\infevaldist(\bs_{1:T})$ matches a target distribution for controlled generation or fine-tuning.  
Assume that $\infevaldist$ is tractable to sample and evaluate.
To calculate the KL divergence to $\sigma$ in either direction, 
we also require an estimate of the $\log$ partition function $\log \cZ_\sigma$,

\small
\begin{align}
\DKL\pV{\infevaldist(\bs_{1:T})}{\sigma(\bs_{1:T})
}  =\EEE_{\infevaldist}\left[\log\fr{\infevaldist(\bs_{1:T})}{\base(\bs_{1:T})\phi(\bs_{1:T})}\right]+{\log\Z_{\si} }
\nonumber %
\\
\DKL\pV{\sigma(\bs_{1:T})
}{\infevaldist(\bs_{1:T})}  =\EEE_{\si}\left[\log\fr{\base(\bs_{1:T})\phi(\bs_{1:T})}{\infevaldist(\bs_{1:T})}\right]-{\log\Z_{\si}} 
\label{eq:fwd_rev_KL_est} \raisetag{3pt}
\end{align} 
\normalsize
For $\DKL\pV{\sigma}{\infevaldist} $, note that we also require samples from the target $\sigma$, {as may be readily available using the BDMC trick when $\sigma$ is defined as a Bayesian posterior} (\cref{sec:conditional_twist}). %
In such cases, we argue that \gls{SMC} can be used to accurately bound the value of $\log \cZ_\sigma$ and estimate
each KL divergence above. 
Estimation of $\DKL\pV{\sigma}{\infevaldist}$ may be particularly important to diagnose mode-dropping in
inference techniques such as \gls{PPO} which optimize the mode-seeking $\DKL\pV{\prop}{\sigma}$ during fine-tuning \citep{korbak2022rl}.

\vheader%
\paragraph{Evaluating Twisted SMC Sampling} 
After running \gls{SIS} or \gls{SMC} with $K$ samples, we can sample a single index as in \cref{eq:snis} to return
a single approximate target sample $\bs_{1:T}^\target$.
{However, the marginal distribution of this sample, which we denote as
$\bs_{1:T}^\target \sim \smcprop(\bs_{1:T})$,}
is not tractable due to the need to sum over all possible sets of $K$ samples.  
Nevertheless, we will show below that the tightness of our $\log \cZ_\sigma$ lower or upper bounds in \cref{prop:smc_proposition} provides upper bounds on the KL divergences $\DKL\pV{\smcprop(\bs_{1:T})}{\sigma(\bs_{1:T})}$ or $\DKL\pV{\sigma(\bs_{1:T})}{\smcprop(\bs_{1:T})}$, respectively.

Alternatively, we can also use the {single-sample} KL divergences in \cref{eq:fwd_rev_KL_est} for the twist-induced proposal $\proptwist{}$ in \cref{eq:twist_prop} to
evaluate a set 
of twist functions $\approxtwist{t}$ (\cref{sec:experiment_kls}).

\vheader
\subsection{Bidirectional SMC Bounds on $\log \cZ_\sigma$}
\label{sec:bi-SMC}

Given the importance of $\log \Z_\si$ estimation as motivated above, we propose a \emph{bidirectional SMC} {stochastic} upper bound which is novel (to the best of our knowledge), and may be of interest outside of the language modeling setting.

Recall from \cref{sec:smc_bg} that \gls{SMC} admits an interpretation as \gls{SIS} in an extended state space $\bS \coloneqq \{s_t^k, \omega_t^k \}_{k=1,t=1}^{K,T}$ which includes all tokens and resampling indices.  
We derive lower and upper bounds on $\log \cZ_\sigma$ in \cref{prop:smc_proposition} below, with proof and detailed description of the extended state space target $\smctgt(\bS)$ and proposal $\smcprop(\bS)$ distributions in  \cref{app:smc}.

\begin{restatable}{proposition}{bdsmc}
    \propheader{(Bidirectional SMC Bounds)} \label{prop:smc_proposition}
The log partition function $\log \cZ_\sigma$ of a target distribution $\sigma(\bs_{1:T})$ can be lower and upper bounded by
\begin{align}
\begin{split}
 &\EEE_{\smcprop\left(
 \smcargs
 \right)}\left[\log\prod_{t=1}^{T}\fr 1K\sum_{i=1}^{K}w_{t}\left(\s_{1:t}^{i}\right)\right]\leq\log\Z_{\si}\\
 &\qquad  \log\Z_{\si}\le\EEE_{\smctgt\left(
 \smcargs
 \right)}\left[\log\prod_{t=1}^{T}\fr 1K\sum_{i=1}^{K}w_{t}\left(\s_{1:t}^{i}\right)\right]\,.
 \end{split} \raisetag{10pt} \label{eq:smc_lb} %
\end{align}
\normalsize
The gap in the lower bound is $\DKL\pV{\smcprop(\bS)}{\smctgt(\bS)}$,
 and the gap in the upper bound is $\DKL\pV{\smctgt(\bS) }{\smcprop(\bS) }$.   
\end{restatable}
See \cref{app:bounds_smc} for a detailed discussion and derivations.
The proof proceeds 
by adapting a general approach for extended state space 
log partition function bounds
from \citet{brekelmans2022improving} using the probabilistic interpretation of \gls{SMC} from \citet{andrieu2010particle, maddison2017filtering}.  %
With no resampling, the \gls{SIS} case recovers the \gls{IWAE} lower \citep{burda2015importance} and upper \citep{sobolev2019importance, brekelmans2022improving} bounds.  %

\vspace{-.3cm}
\paragraph{Sampling from $\smctgt$ for SMC Upper Bounds}
We now discuss sampling from $\smctgt(\bS)$ for the expectation in the upper bound, which requires a single, \textit{exact} sample from the
target distribution $\sigma(\bs_{1:T})$.   This sample may be obtained, for example, using the BDMC trick in \cref{sec:conditional_twist}.
Note that \cref{sec:smc_bg} and \cref{alg:smc} describe sampling from $\smcprop(\bS)$, which is used for 
the expectation in the lower bound.

Sampling from  $\smctgt(\bS)$ differs from
sampling from $\smcprop(\bS)$ by its treatment of the exact target sample.
In particular,
the partial sequence corresponding to the exact target sample is {guaranteed} to be cloned once %
at each resampling step. 
In other indices, resampling proceeds as in \cref{sec:smc_bg}, where the 
exact sample may be cloned additional times based on its incremental importance weights.  Finally, we sample $K-1$ next tokens from the proposal, while the value of the remaining chain is fixed by the exact target sample.
See \cref{app:smc} and \cref{alg:smc_target} for detailed discussion.

\vheader
\paragraph{Tightness of the Bidirectional Bounds}
Since the bounds in \cref{prop:smc_proposition} become exact as $K \rightarrow \infty$ for any proposal \citep{burda2015importance, maddison2017filtering}, we can use \gls{SMC} or \gls{IWAE} with large $K$ to sandwich the $\log$ partition function when $\sigma$ samples are available.   

For a given $K$, the gap in the extended state space $\log \cZ_\sigma$ bounds in \cref{prop:smc_proposition} provides further insight into the quality of twisted \gls{SMC} sampling via the distribution of the marginal sample $\bs_{1:T}^\sigma$ (\cref{sec:eval_applications}).
In particular, the data processing inequality suggests that 
${\DKL\pV{\smcprop(\bs_{1:T})}{\sigma(\bs_{1:T})} \leq \DKL\pV{\smcprop(\bS)}{\smctgt(\bS)}}$ and ${\DKL\pV{\sigma(\bs_{1:T})}{\smcprop(\bs_{1:T})} \leq \DKL\pV{\smctgt(\bS)}{\smcprop(\bS)}}$ \citep{grosse2015sandwiching, grosse2016measuring}.
Thus, if the difference between upper and lower bounds on $\log \cZ_\sigma$ is small, then we can conclude that the $K$-sample \gls{SMC} or \gls{SIS} procedures in \cref{sec:smc_bg} yield a single approximate sample $\bs_{1:T}^\sigma$ whose distribution $\smcprop(\bs_{1:T})$ is close to the target $\sigma(\bs_{1:T})$ in symmetrized KL divergence.\footnote{Note that the difference between upper and lower bound yields $\DKL\pV{\smctgt(\bS)}{\smcprop(\bS)}+\DKL\pV{\smcprop(\bS)}{\smctgt(\bS)}$.}

\vheader
\section{Related Work} \label{sec:related}
\vheader
In the previous sections, we have discussed related work as it fit within our \gls{SMC} framework for language modeling.   Note that \citet{lew2023sequential} consider \gls{SMC} sampling for language models, but do not learn twist functions or proposals. %

Decoding from language models 
to obtain diverse \citep{holtzman2019curious, vilnis2023arithmetic} or controlled generation \citep{zhang2023survey,  dathathri2019plug, krause2020gedi, yang2021fudge, guo2021efficient, qin2022cold, snell2022offline, hu2023amortizing}
is an active area of research.
Our \gls{SMC} resampling approach may be viewed as a principled \textit{probabilistic} extension of best-of-$K$ decoding methods.
\citet{mudgal2023controlled} propose a $K$-way $\argmax$ decoding scheme based on `prefix scorers' $\approxtwist{t}$ learned using \cref{eq:psi_recursion}, but also consider using these twists as logits for softmax sampling in the proposal. 
However, neither of these decoding schemes are 
aligned with our proposed \gls{SMC} framework, as we discuss in detail in \cref{app:mudgal_decoding_all}.  
For example, greedy $\argmax$ decoding with respect to the optimal twists in \cref{prop:optimal_twists} does not yield samples from the target distribution $\sigma(\bs_{1:T})$.

Finally, \gls{RL}-based 
methods such as \gls{PPO} maintain both a policy or proposal network and value network or advantage estimator during training.   From the soft \gls{RL} perspective 
in 
\cref{sec:soft_rl_connection} and
\cref{app:soft_rl}, the soft values play a similar role as
our twist functions for \gls{SMC} resampling.  
\citet{liu2023don} consider using Monte Carlo Tree Search (MCTS) based on \gls{PPO} value estimates to improve decoding, while \citet{chaffin2022ppl} consider discriminator-driven MCTS.

\newcommand{\experimentsubsection}[1]
{\subsubsection{#1}}

\vheader
\section{Experiments} \label{sec:experiments}
\vheader

We now 
illustrate empirically how our framework can be used to evaluate inference through $\log \zsigma$ bounds and KL divergences between the sampling and target distributions, providing meaningful quantitative comparison between various learning methods.
We consider a range of tasks throughout this section, including toxic story generation (as an example of uncovering rare undesirable behavior), generating reviews with varied sentiment, and infilling.
For the toxicity and infilling tasks, we consider the TinyStories model \cite{eldan2023tinystories}\footnote{\scriptsize\url{https://huggingface.co/roneneldan/TinyStories-33M}} as a small-scale model where the generation is coherent, and use the prompt of `Once upon a time, there was a'. For the toxicity task, we elicit responses judged to be toxic by the classifier from \citet{nicholas22aira}\footnote{\scriptsize\url{https://huggingface.co/nicholasKluge/ToxicityModel}}. For the sentiment task, we consider the GPT2-Medium\footnote{\scriptsize\url{https://huggingface.co/gpt2-medium}} model and a classifier trained on Amazon reviews.\footnote{\scriptsize\url{https://huggingface.co/LiYuan/amazon-review-sentiment-analysis}} Our code is available at \url{https://github.com/Silent-Zebra/twisted-smc-lm}~.  

\begin{figure}[t]
\vskip -0.1in
\begin{center}
\centerline{\includegraphics[width=1\columnwidth, trim = 0mm 2mm 0mm 7mm, clip]{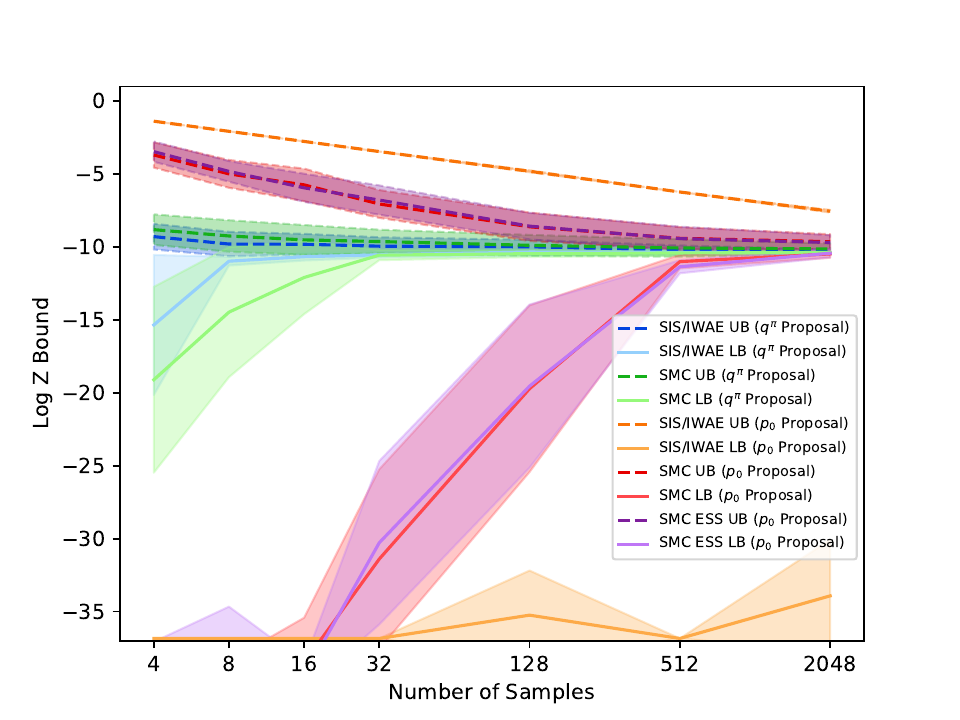}}
\vspace{-.1cm}
\caption{\small{Comparison of SIS (IWAE) and SMC bounds on $\log \zsigma$ for base proposal $\base$ and twist-induced proposal $\proptwist{}$, with twists learned with CTL.  With the twist-induced proposal, both SIS and SMC bounds are tight; with the base proposal, resampling with learned twists is needed. Resampling based on ESS instead of every-step resampling yields similar results.
}}
\vheader
\vheader
\label{fig:toxthresh}
\end{center}
\vheader
\end{figure}

\subsection{Comparing SIS and SMC for $\log \cZ_\sigma$ Estimation}\label{sec:toxthresh}
We first use our $\log \zsigma$ bounds to test how 
twisted \gls{SMC} can improve upon \gls{SIS} and efficiently sample rare events.
We consider the task of toxic story generation.
The target is defined as
$\sigma(\bs_{1:T}) \propto \base(\bs_{1:T})  \mathbb{I}[\bs_{1:T} \in \cC]  $ where $\cC \coloneqq \{\bs_{1:T} ~ | r(\bs_{1:T}) \leq {\eta} \}$, $r(\bs_{1:T})$ is the non-toxic logit, and the
threshold ${\eta} = -5$ corresponds to a greater than 99\% chance of being toxic. 
Rejection sampling under $\base$ yields exact samples for $\log \cZ_\sigma$ UB estimation, but can require hundreds of thousands of samples.  
Thus, this setting also allows us to test the effectiveness of approximate positive sampling for twist training when target samples are rare.

\cref{fig:toxthresh} demonstrates that training twists with \gls{CTL} and approximate positive sampling can significantly improve log partition function estimation and sampling efficiency.
We first note that both upper and lower bounds tighten as $K$ increases, as expected, for both \gls{SIS} and \gls{SMC}.
Using $\base$ as proposal, the \gls{SIS} LB (orange) generally fails to draw any samples meeting the threshold.  By contrast, \gls{SMC} resampling (red) with $\base$ proposal
eventually achieves \textit{tight} $\log \cZ_{\sigma}$ upper and lower bounds, yielding near-exact target samples (small KL divergence between the distribution over samples and the target distribution) by the reasoning in \cref{sec:eval}.
\begin{figure*}
\vheader
\begin{minipage}{.3\textwidth}
\centering
\resizebox{\columnwidth}{!}{%
\begin{tabular}{cccc}
\toprule
Proposal $q$ & Twist Learning & $\DKL\pV q{\si}$ & $\DKL\pV{\si}q$\tabularnewline
\midrule\midrule
Twisted & Contrastive & $1.11 \pm 0.05$ & \boldmath{$1.07 \pm 0.02$} \tabularnewline  \midrule
Twisted & RL & $1.52 \pm 0.09$ & $1.42 \pm 0.03$ \tabularnewline  \midrule
Twisted & SIXO & $1.71 \pm 0.06$ & $1.98 \pm 0.04$ \tabularnewline  \midrule
Twisted & FUDGE & $3.24 \pm 0.26$ & $2.00 \pm 0.13$ \tabularnewline  \midrule \midrule
DPG & -- & $1.09 \pm 0.05$ & $1.12 \pm 0.03$ \tabularnewline  \midrule
PPO & -- & \boldmath{$0.98 \pm 0.01$} & $1.32 \pm 0.04$ \tabularnewline
\bottomrule
\end{tabular}%
}
\vspace{-0.3cm}
\captionof{table}{\label{table:res_toxc} Toxicity (\cref{sec:toxclass})}
\end{minipage}%
\hfill %
\begin{minipage}{.3\textwidth}
\centering
\resizebox{\columnwidth}{!}{%
\begin{tabular}{cccc}
\toprule
Proposal $q$ & Twist Learning & $\DKL\pV q{\si}$ & $\DKL\pV{\si}q$\tabularnewline
\midrule\midrule
Twisted & Contrastive & \boldmath{$0.55 \pm 0.03$} & \boldmath{$0.47 \pm 0.01$} \tabularnewline  \midrule
Twisted & RL & $0.94 \pm 0.04$ & $0.81 \pm 0.02$ \tabularnewline  \midrule
Twisted & SIXO & $0.73 \pm 0.03$ & $0.59 \pm 0.02$ \tabularnewline  \midrule
Twisted & FUDGE & $1.01 \pm 0.07$ & $0.77 \pm 0.07$ \tabularnewline  \midrule \midrule
DPG & -- & $0.72 \pm 0.04$ & $0.57 \pm 0.01$ \tabularnewline  \midrule
PPO & -- & $1.04 \pm 0.31$ & $0.87 \pm 0.20$ \tabularnewline
\bottomrule
\end{tabular}%
}
\vspace{-0.3cm}
\captionof{table}{\label{table:res_sent} Sentiment (\cref{sec:sent})}
\end{minipage}%
\hfill %
\begin{minipage}{.385\textwidth}
\centering
\resizebox{\columnwidth}{!}{%
\begin{tabular}{cccc}
\toprule
Proposal $q_{o_T}$ & Twist Learning & $\E_{o_T}[\DKL\pV {q_{o_T}}{\sigmaot}]$ & $\E_{o_T}[\DKL\pV{\sigmaot}{q_{o_T}}]$\tabularnewline
\midrule\midrule
Twisted & Contrastive & $23.93 \pm 0.34$ & $8.87 \pm 0.05$ \tabularnewline  \midrule
Twisted & RL & $31.35 \pm 2.33$ & $14.96 \pm 1.69$ \tabularnewline  \midrule
Twisted & SIXO & $20.34 \pm 0.36$ & $7.43 \pm 0.04$ \tabularnewline  \midrule
Twisted & FUDGE & $60.93 \pm 2.82$ & $19.85 \pm 0.51$ \tabularnewline  \midrule \midrule
DPG & -- & \boldmath{$13.27 \pm 0.44$} & \boldmath{$4.90 \pm 0.03$} \tabularnewline  \midrule
PPO & -- & $19.37 \pm 0.41$ & $14.07 \pm 0.50$ \tabularnewline
\bottomrule
\end{tabular}
}
\vspace{-0.3cm}
\captionof{table}{\label{table:res_plast15_10} Infilling (\cref{sec:infilling})}
\end{minipage}
\vspace{-0.2cm}
\caption*{Forward and reverse KL divergences between the SMC or variational proposal distributions and the true target $\sigma$.}
\vheader
\end{figure*}

However, both \gls{SMC} and \gls{SIS} with the twist-induced proposal 
achieve
tight estimation and near-exact sampling of the target toxic outputs with orders of magnitude lower $K$.
Resampling does not appear to help or hurt these bounds, as the effect of the twists has been incorporated in the proposal $\proptwist{}$ in \cref{eq:twist_prop}. 
Thus, we conclude that 
using
the twist-induced proposal can provide significant efficiency gains over base model sampling.
\vheader
\subsection{Evaluating Twist-Induced or Variational Proposals}
\label{sec:experiment_kls}
We next leverage our $\log \cZ_\sigma$ bounds to evaluate 
single-sample inference using $\DKL\pV q{\si}$ and $\DKL\pV{\si}q$, as in \cref{sec:eval_applications}.
Across settings, we consider two \gls{SIS} proposal-learning methods: 
\gls{PPO} \citep{schulman2017proximal} which minimizes
$\DKL\pV{q}{\sigma}$ during optimization, and \gls{DPG}, which minimizes $\DKL\pV{\sigma}{q}$ \citep{parshakova2019distributional} (see \cref{app:prop}).  %
We consider four twist learning methods, including \gls{CTL} and \gls{RL} from \cref{sec:learning}, \textsc{SIXO} \citep{lawson2022sixo}, and \textsc{FUDGE} \citep{yang2021fudge} (see \cref{app:twist}).   For each, we measure KL divergences involving the twist-induced proposal $\proptwist{}$. 
Thus, \textit{these experiments  showcase two complementary applications of SMC}:  as a novel inference method yielding a tractable $\proptwist{}$, and as an evaluation method for any other inference method (such as \gls{PPO}) using $K$-sample bounds on $\log \cZ_\sigma$ to estimate the KL divergence.  

\vheader
\experimentsubsection{Generating Toxic Stories}\label{sec:toxclass}
\vheader
We consider toxic story generation as in \cref{sec:toxthresh}, but using a target $\sigma(\bs_{1:T}) \propto p_0(\bs_{1:T}) p(a=1| \bs_{1:T})
$, where $p(a=1| \bs_{1:T})$ denotes the probability of the text being judged as toxic by a classifier.  Compared to the thresholding target, this task provides a smoother gradient signal for learning (see \cref{app:expmt_choices}) but still allows for exact sampling via rejection sampling.   We train using approximate positive sampling, but provide an ablation with exact positive sampling results in \cref{sec:exact_vs_appr}.

We report KL divergences in \cref{table:res_toxc}.
We observe that \gls{PPO} learns the best proposal with respect to $\DKL\pV{q}\si$ while our \gls{CTL} method performs best with respect to
$\DKL\pV{\si}q$, which is consistent with the divergences minimized during training.
Finally, in 
\cref{app:qualitative}
we provide a qualitative example of a toxic story generated with \gls{CTL} for $\sigma(\bs_{1:T}) \propto p_0(\bs_{1:T}) p(a=1| \bs_{1:T})^{\beta}$ with $\beta = 10$, a case where no exact samples are available. 
\vheader
\experimentsubsection
{Generation with Varied Sentiment}\label{sec:sent}
\vheader
For the sentiment setting described earlier, we consider a prompt `I bought this' and target $\sigma(\bs_{1:T}) \propto  p_0(\bs_{1:T}) p(a=1| \bs_{1:T}) 
$, where $a=1$ indicates a 1-star review and exact samples are available by rejection sampling.
We train using approximate positive sampling
(see \cref{sec:exact_vs_appr} for comparison with exact).
While all methods achieve low KL divergences in \cref{table:res_sent}, \gls{CTL} performs best for both directions.

\vheader
\subsubsection{Infilling}\label{sec:infilling}
\vheader  

In this section, we demonstrate a \textit{conditional twist function} parameterization, where $\vartwist{t}(\bs_{1:t}, o_T)$ takes input $o_T$ which identifies the target distribution $\sigma(\bs_{1:T}|o_T)$ 
as in \cref{sec:conditional_twist}. We consider an infilling task \citep{lew2023sequential, hu2023amortizing}, where the observation variables
$o_T \coloneqq \bs_{T+1:T+c}$ correspond to continuation tokens, and their likelihood $\sigma(o_T|\bs_{1:T}) \coloneqq \base(\bs_{T+1:T+c} | \bs_{1:T})$ is evaluated under the base model, given generated $\bs_{1:T}$.   The target distribution corresponds to the posterior $\sigma(\bs_{1:T}|o_T)$.  Instead of training separate $\{\vartwist{t}\}$ for each $o_T$, we amortize learning of a conditional twist network $\vartwist{t}(\bs_{1:t}, o_T)$.

A second distinctive feature of this setting is that we \textit{train} from exact posterior or target samples,  %
which are readily available using the BDMC trick in \cref{sec:conditional_twist}.   In particular, we
may sample sequences of length $T+c$ from the base model $\bs_{1:T+c} \sim \base(\bs_{1:T+c}) = \sigma(\bs_{1:T}, o_T = \bs_{T+1:T+c})$, and interpret the prefix $\bs_{1:T} \sim \sigma(\bs_{1:T}|o_T =\bs_{T+1:T+c} )$ as a target sample.   Note that we do not 
explicitly 
control the continuations
tokens 
$o_T$ defining the tasks.
We evaluate average KL divergences over 2000 different 
$o_T =\bs_{T+1:T+c}$,
with $T=15$ and $c=10$, and report results in \cref{table:res_plast15_10}.

We find that \gls{DPG} performs best for both directions of the KL divergence in this setting, likely due to its ability to 
leverage exact positive samples 
by minimizing $\DKL\pV{\sigmaot}{q_{o_T}}$.   While \gls{CTL} also learns from exact positive samples, it requires approximate negative sampling and only performs comparably to \textsc{SIXO}, which uses exact positive samples and performs exact negative sampling under $\base$. Finally, \gls{PPO} trains from $q_{o_T}$ samples only, and performs relatively poorly with respect to $\DKL\pV{\sigmaot}{q_{o_T}}$.  We show qualitative results in 
\cref{app:qualitative}
to 
correlate
KL divergence results with sample 
quality.

{Using our KL divergence evaluation methods,
    we conclude DPG may be preferable when exact target samples are available (\cref{sec:infilling}, \cref{sec:exact_vs_appr}), while CTL may be preferable with approximate positive sampling (\cref{sec:toxclass}, \cref{sec:sent}).}

\vheader
\section{Conclusion}
\vheader
In this work, we have presented twisted \gls{SMC} as a principled probabilistic inference framework for solving numerous capability and safety tasks in LLMs. After discussing 
different design choices for twisted \gls{SMC} and their relation to related work, we proposed a novel contrastive method for twist learning. Furthermore, we have proposed novel bidirectional \gls{SMC} bounds for evaluating LLM inference methods. We demonstrated the effectiveness of our methods quantitatively and qualitatively in both sampling and evaluation across a variety of experimental settings.
\vheader
\section*{Acknowledgments}
\vheader
AM and RG acknowledge support from the Canada CIFAR AI Chairs program and from Open Philanthropy.
SZ thanks Juhan Bae for helping debug memory issues in the code.
Resources used in this research were provided, in part, by the Province of Ontario, the Government of
Canada, and companies sponsoring the Vector Institute. We thank the anonymous 
reviewers for helpful comments on earlier versions of this paper.

\vheader
\section*{Impact Statement}
\vheader
This paper is motivated by the social consequences of recent advances in the field of machine learning. Controlled generation from language models has the potential to improve safety 
through better steering of generation to human preferences, more efficient automated red-teaming, and the ability to estimate or bound probabilities of rare behaviors.
Any such work is inherently a double-edged sword; the same techniques used to generate samples from a harmless distribution of text could, with a single sign change, be repurposed for generating samples from a harmful distribution of text. Thus, better controlled generation (in our framework, better sampling from target distributions) can provide 
benefits in the hands of responsible users but can also magnify 
harms in the hands of malevolent users (who have access to model weights). 

Overall, we believe the potential positive social benefits of our work in evaluation and steering language model output towards desired target distributions outweigh the potential negatives stemming primarily from misuse.

\bibliography{refs}
\bibliographystyle{icml2024}

\appendix
\clearpage
\onecolumn
\part{Appendix} %
\parttoc
\newcommand{\finalpost}{\sigma(\bs_{1:T})}
\newcommand{\finalbase}{\base(\bs_{1:T})}
\vheader
\begin{table*}[t]
\caption{Examples of Target Posteriors in Language Model Finetuning and Controlled Generation
}
\label{table:posteriors}
\resizebox{\textwidth}{!}{
\small 
\begin{tabular}{lll}
\toprule
  \textbf{\textit{Type}} %
  & Target & References / Examples \\ \toprule
     Reward & $\finalpost \propto \finalbase e^{\pm \beta r(\bs_{1:T})}$ & RLHF  \citep{ziegler2019fine, ouyang2022training, korbak2022rl} \\ \midrule
     Continuation
     & $\finalpost \propto \finalbase \base(\bs_{T+1:T+c} | \bs_{1:T})^\beta$ &
     Generates tokens based on likelihood of future tokens $p(\bs_{T+1:T+c} | \bs_{1:T})$ \\ & &
     For $\beta = 1$, this is in-filling \citep{lew2023sequential}. \\
     & & 
     As $\beta \rightarrow \infty$, disregard $\finalbase$, focus on $\argmax$ of continuation prob.  \\ 
     & & \qquad - similar to adversarial prompt generation \citep{zou2023universal}
      \\ \midrule
     Indicator & $\finalpost \propto \finalbase \mathbb{I}[\bs_{1:T} \in \cC]  $& 
     Generations $\bs_{1:T}$ from this target must satisfy the properties of set $\cC$.  \\ 
        & ~~\text{where $\bbI$ is indicator of set $\cC$:} & - Meeting reward threshold $\cC_{r\leq{\eta}} \coloneqq \{\bs_{1:T} ~ | \pm r(\bs_{1:T}) \leq {\eta} \}$ \\
        & ~~ $\mathbb{I}[\bs_{1:T} \in \cC] = 1$ if $[\bs_{1:T} \in \cC]$ & 
     \text{\small - Containing topical or specific words in $\bs_{1:T}$}  \\
     & ~~ $\mathbb{I}[\bs_{1:T} \in \cC] = 0$ if $[\bs_{1:T} \notin \cC]$ &
     - Having certain structure or rhyme \citep{yang2021fudge}, \\ & & 
     - Valid output according to verifier \citep{cobbe2021training, dohan2022language}) 
      \\ \midrule
    Classifier &  $\finalpost \propto \finalbase p(y|\bs_{1:T})^\beta  $ & Class $y$ can be a binary (e.g. toxicity) or multinomial (e.g. 1-5 star reviews) \\ 
    & & Bayesian posterior for $\beta = 1$: $\finalpost = p(\bs_{1:T}|y) \propto \base(\bs_{1:T})p(y|\bs_{1:T})$ \\
    & & \citep{dathathri2019plug, krause2020gedi, liu2021dexperts}
    \\ \midrule
    Global  &  $\finalpost \propto \finalbase^{\beta}  $ &  Tempering on entire sequences (long-horizon) vs. per-token (myopic) \\ 
    Temperature & & \quad - yields higher quality generation in \citet{shih2023long}
    \\ \midrule
    \text{\small Distributional} & $\finalpost \propto  \finalbase e^{ \bm{\beta} \cdot \bm{T}(\bs_{1:T}) }$ &  KL minimization subj. expectation constraints on $\bm{T} = \{ T_i\}$   \\
    & & \text{\small $\argmin \DKL(\optdist(\bs_{1:T})\| p_0(\bs_{1:T}))$ s.t. $\mathbb{E}_{\optdist}[\bm{T}]=\bm{\eta}_{\bm{\beta}}$}\\
    & & ($\bm{\beta}$ = optimal Lagrange multipliers for constraints $\bm{\eta}$) \\
   & & \text{\small \qquad e.g. gender roles/references 
   \citep{khalifa2020distributional}} 
    \\ \bottomrule 
    \\
    \textbf{\textit{Intermediate}}   &  & References / Examples \\ \toprule
    Indicator & $\reward{t} =  \mathbb{I}[s_t \in \cC]$ \text{ or }  $\mathbb{I}[\bs_{1:t} \in \cC]$
    &  words of specific length, or specific sets of tokens 
    \\ 
    & & \text{\small ~~\citep{khalifa2020distributional,lew2023sequential}}\\
    Product of \\ Experts & 
    $\sigma(\bs_{1:T}) \propto \prod_{m=1}^M \prod_{t=1}^T \base(s_t | \bs_{1:t-1}, \bs_0^{(m)})
$ 
    & \text{\small prompt intersection \citep{lew2023sequential}} \\
     \bottomrule
\end{tabular}
}
\normalsize
\vheader
\end{table*}

\vheader
\section{Proofs}\label{app:proofs}
\vheader

In this section, we present the sense in which the target marginals correspond to the \textit{optimal} intermediate distributions in twisted SMC.   We defer proof of \cref{prop:optimal_twists} from the main text to slightly more general version in \cref{app:int_reward} \cref{prop:opt_twists_int_reward}, although \cref{prop:opt_intermediate} provides the analogous statement in terms of the intermediate target distributions $\pi_t^*(\bs_{1:t}) = \sigma(\bs_{1:t})$ instead of the optimal twists $\psi_t^*$.

We also prove \cref{prop:transition} from the main text in \cref{app:proof-twist-ind-prop} and derive the gradient of the \gls{CTL} loss (\cref{eq:grad_separate_ebm}) in \cref{app:ctl_gradient}. %

\subsection{Proof for Optimal Intermediate Target Distributions}
\label{app:optimality}
In order to achieve sampling from the full joint distribution $\sigma(\bs_{1:T})$, each intermediate target $\sigma(\bs_{1:t})$ must match the intermediate marginal $\sigma(\bs_{1:t})$.   
To formalize this notion, we provide the following definition of optimality, justified by the fact that it yields an exact partition function estimator.

To do so, we will consider the multi-step importance weights
\begin{align}
\hspace*{-.2cm} w_{t:t+c-1}(\bs_{1:t+c-1}) = \prod \limits_{\tau=t}^{t+c-1} w_{\tau}(\bs_{1:\tau}) 
= \prod \limits_{\tau=t}^{t+c-1} \frac{\tpitwist{\tau}(\bs_{1:\tau})}{\tpitwist{\tau-1}(\bs_{1:\tau-1})q(s_{\tau}|\bs_{1:\tau-1})} = \frac{\tpitwist{t+c-1}(\bs_{1:t+c-1})}{\tpitwist{t-1}(\bs_{1:t-1})q(\bs_{t:t+c-1}|\bs_{1:t-1})}
\label{eq:c-step-weights}   \raisetag{-5pt} \tag{$c$-Step SMC Weights}
\end{align}
using a telescoping cancellation in the final equality.  The one-step weights correspond to $c=1$, denoted simply as $w_t$.

\begin{definition}[\propheader{Optimal Twisted SMC Sampling}]\label{definition:optimality}
For a given target distribution $\target(\bs_{1:T}\condzero) \propto \base(\bs_{1:T}) \finaltwist$,  we refer to a twisted SMC procedure, $\textsc{SMC}(\{\pi_t\}_{t=1}^T, q, K)$ or $\textsc{SMC}(\base, \{\psi_t\}_{t=1}^T, q, K)$ (with $\pi_T = \sigma$ or $\psi_T = \phi$), as \textup{optimal} if 
$c$-step importance weights $w_{t:t+c-1}(\bs_{1:t+c-1}) = \zfilter{t+c-1} / \zfilter{t-1}$
for all $1\leq t \leq T$ and $0\leq c \leq T - t + 1$.   
\end{definition}
Note, that 
the role of $\psi_t$ and $\zfilter{t}$ is specified in \cref{def:twist_targets}.   We assume $\pi_T = \target$ for the goal of estimating $\cZ_\sigma$, and show below that an optimal twisted SMC procedure yields an exact partition function estimator.

\vheader
\begin{proposition}[\propheader{Optimal SMC yields Exact Partition Function Estimation}]\label{proposition:exact}
For any optimal twisted SMC procedure, the resulting estimator of the partition function $\cZ_\sigma$ has zero bias and zero variance.  
\end{proposition}
\vheader 
\begin{proof}
As in \cref{footnote:smc_bound} or \cref{app:smc} \cref{alg:smc_target}, consider $\left\{ t_{r}\right\} _{r=1}^{R}$ index timesteps where resampling occurs and fix $t_{0}=0$ and $t_{R}=T$.   The SMC estimator of $\cZ_\sigma = \zfilter{T}$ becomes 
$ \hat{\cZ}_\sigma^{\textsc{smc}} =  
\prod_{r=1}^{R}\fr 1K\sum_{i=1}^{K}\left(\prod_{t=t_{r-1}+1}^{t_{r}}w_{t}\left(\s_{1:t}^{i}\right)\right)
$ for $\bS \sim q_{\textsc{SMC}}(\bS)$.
Using the optimality definition in \cref{definition:optimality}, we have $w_t(\bs_{1:t}) = \zfilter{t}/\zfilter{t-1}$ for all partial sequences $\bs_{1:t}$.
Noting the telescoping multiplicative cancellation and the fact that $ w_t(\bs_{1:t}^i) = \zfilter{t}/\zfilter{t-1}$ is constant with respect to indices $i \in [1,K]$, we have the following estimator for a single run of an optimal SMC procedure,
\small 
\begin{align}
    \hat{\cZ}^{\textsc{smc}}_\sigma =\prod_{r=1}^{R}\fr 1K\sum_{i=1}^{K}\left(\prod_{t=t_{r-1}+1}^{t_{r}}w_{t}\left(\s_{1:t}^{i}\right)\right) =     \prod_{r=1}^{R} \frac{\zfilter{t_r}}{\zfilter{t_{r-1}}}  = \frac{\zfilter{t_R}}{\zfilter{t_0}} = \frac{\zfilter{T}}{\zfilter{0}} = \cZ_\sigma
\end{align}
\normalsize
 as desired, assuming $\zfilter{0} = 1$.   Since $\hat{\cZ}^{\textsc{smc}}_\sigma = \cZ_\sigma$ is independent of $\bS$, we conclude
$\hat{\cZ}^{\textsc{smc}}_\sigma$
 has zero bias and zero variance.

Note that we could also define optimality in \cref{definition:optimality} using the 
condition that $w_{t:t+c-1}(\bs_{1:t+c-1}) = \text{const}$ for all $1 \leq t \leq T$ and $0 \leq c \leq T-t+1$.   Following similar derivations as above would yield $\hat{\cZ}^{\textsc{smc}}_\sigma = \text{const}$.   As we will show in \cref{app:smc},  $\hat{\cZ}^{\textsc{smc}}_\sigma$ is unbiased with $\mathbb{E}[\hat{\cZ}^{\textsc{smc}}_\sigma]=\cZ_\sigma$.   We thus conclude that $\hat{\cZ}^{\textsc{smc}}_\sigma = \cZ_\sigma$ with zero variance, and thus \cref{proposition:exact} holds.
\end{proof}

With this notion of optimality in mind, we demonstrate the following necessary and sufficient conditions.

\begin{proposition}[\propheader{Optimality Conditions}]  \label{prop:optimality_conditions}
The following conditions are necessary and sufficient for twisted SMC optimality,
    \begin{align}
    \begin{split}
    (i): \phantom{q^\optornot_t(s_t|\bs_{1:t-1})} \pi^\optornot_t(\bs_{1:t}) &= \sigma(\bs_{1:t}) \qquad \qquad \qquad \qquad \, \hfill \forall \quad 1 \leq t \leq T \\
    (ii): \phantom{\pi^\optornot_t(\bs_{1:t})} q^\optornot_t(s_t|\bs_{1:t-1}) &= \sigma(s_t|\bs_{1:t-1}) \qquad \qquad \qquad  \hfill \forall \quad 1 \leq t \leq T \,.
    \end{split}
    \end{align}
\end{proposition}
\begin{proof}
   \textit{(Necessary) Optimal Twisted SMC $\implies (i),(ii)$:}
    We begin by writing 
    the marginalization of the unnormalized density $\tpi^\optornot_{t+c}$ over prefixes of length $t$ as
\small    \begin{align}
    \tpi^\optornot_{t+c}(\bs_{1:t}) = \sum_{\bs_{t+1:t+c}} \tpi^\optornot_{t+c}(\bs_{1:t+c}) = \sum_{\bs_{t+1:t+c}} \base(\bs_{1:t+c}) \approxtwist{t+c}(\bs_{1:t+c}) = \base(\bs_{1:t})\sum_{\bs_{t+1:t+c}} \base(\bs_{t+1:t+c}|\bs_{1:t}) \approxtwist{t+c}(\bs_{1:t+c}) \nonumber
    \end{align}
    \normalsize
    The normalization constant of $\tpi^\optornot_{t+c}(\bs_{1:t})$ can easily be seen to be $\zfilterstar{t+c}$ after summing over $\bs_{1:t}$ above, which yields $\pi^\optornot_{t+c}(\bs_{1:t}) = \tpi^\optornot_{t+c}(\bs_{1:t})/\zfilterstar{t+c}$.
    We now factorize the $c$-step incremental importance weights (at step $t+1$, see \cref{eq:c-step-weights}) using the above identities, which imply that $\tpi^\optornot_{t+c}(\bs_{1:t+c}) = \zfilterstar{t+c}\pi^\optornot_{t+c}(\bs_{1:t+c}) = \zfilterstar{t+c}
    \pi^\optornot_{t+c}(\bs_{1:t})\pi^\optornot_{t+c}(\bs_{t+1:t+c}|\bs_{1:t})$ and
\small
\begin{align}
w_{t+1:t+c}(\bs_{1:t+c}) 
= \frac{\tpi^\optornot_{t+c}(\bs_{1:t+c})}{\tpi^\optornot_{t}(\bs_{1:t})q^\optornot(\bs_{t+1:t+c}|\bs_{1:t})}  
= \frac{\zfilterstar{t+c}}{\zfilterstar{t}} \frac{
    \pi^\optornot_{t+c}(\bs_{1:t})}{\pi^\optornot_{t}(\bs_{1:t})} \frac{\pi^\optornot_{t+c}(\bs_{t+1:t+c}|\bs_{1:t})}{q^\optornot(\bs_{t+1:t+c}|\bs_{1:t})} 
\end{align}
\normalsize
In order to have $ w_{t+1:t+c}(\bs_{1:t+c}) = {\zfilterstar{t+c}}/{\zfilterstar{t}}  $ in general, we thus require $ \pi^\optornot_{t+c}(\bs_{1:t}) = \pi^\optornot_{t}(\bs_{1:t})$ and $ {\pi^\optornot_{t+c}(\bs_{t+1:t+c}|\bs_{1:t})}={q^\optornot(\bs_{t+1:t+c}|\bs_{1:t})}$ for all $t$ and $c \leq T-t$.  

 \textit{ (Sufficient) $(i),(ii) \implies $ Optimal Twisted SMC:}
Consider the incremental importance weights using $(i)$ and $(ii)$, 
\small 
\begin{align}
 w_{t}(\bs_{1:t}) = \frac{\tpitwist{t}^*(\bs_{1:t})}{\tpitwist{t-1}^*(\bs_{1:t-1})\propopttwist{t}(s_t|\bs_{1:t-1})} = 
 \frac{\zfilter{t} \sigma(\bs_{1:t})}{\zfilter{t-1} \sigma(\bs_{1:t-1}) \sigma(s_t|\bs_{1:t-1})}  = \frac{\zfilter{t}}{\zfilter{t-1}}
\end{align}
\normalsize
which matches the optimality definition in \cref{definition:optimality}.
 \end{proof}

\begin{proposition}[\propheader{Optimal Intermediate Target Distributions}]\label{prop:opt_intermediate}
For a given target distribution $\target(\bs_{1:T}\condzero)$ (\cref{eq:tgt_rewards}), 
the following conditions are equivalent, and are necessary 
for optimality of a twisted SMC procedure involving $\{\pi^\optornot_{t} \}_{t=1}^T$,
\begin{align}
\begin{split}
(i)
: \qquad \pi^\optornot_t(\bs_{1:t}) &= \sum_{s_{t+1}} \pi^\optornot_{t+1}(\bs_{1:t+1}) \qquad  \qquad \hfill \forall \quad 1 \leq t \leq T-1 \,, \\
(ii)
: \qquad \pi^\optornot_t(\bs_{1:t}) &= \sum_{\bs_{t+1:t+c}} \pi^\optornot_{t+c}(\bs_{1:t+c}) \qquad ~~~  \hfill \forall \quad 1 \leq t \leq T-1, ~ 1 \leq c \leq T-t \,,
\\ (iii): \qquad \pi^\optornot_t(\bs_{1:t}) &= \sigma(\bs_{1:t}) \qquad \qquad \qquad \qquad \hfill \forall \quad 1 \leq t \leq T \,.
\end{split}
\end{align}
Conditions $(i)$ and $(iii)$ directly correspond to the recursions for the optimal twist functions given in \cref{prop:optimal_twists} and \cref{prop:opt_twists_int_reward}.
\end{proposition}
\begin{proof}
\textit{$(i) \iff (ii)$:}  It is clear that $(ii) \implies (i)$ as a special case for $c = 1$.   To show $(i) \implies (ii)$,  we have
\small
\begin{align*}
\pi^\optornot_t(\bs_{1:t}) = \sum_{s_{t+1}} \pi^\optornot_{t+1}(\bs_{1:t+1}) =\sum_{s_{t+1}}\sum_{s_{t+2}} \pi^\optornot_{t+2}(\bs_{1:t+2}) = ... = \sum_{\bs_{t+1:t+c}} \pi^\optornot_{t+c}(\bs_{1:t+c}).
\end{align*}
\normalsize
\textit{$(i) \implies (iii):$}  Recursively applying $(i)$ until time $T$ suggests that
\small 
\begin{align*}
\pi^\optornot_t(\bs_{1:t}) = \sum_{s_{t+1}} \pi^\optornot_{t+1}(\bs_{1:t+1}) =\sum_{s_{t+1}}\sum_{s_{t+2}} \pi^\optornot_{t+2}(\bs_{1:t+2}) = ... = \sum_{\bs_{t+1:T}} \pi^\optornot_{T}(\bs_{1:T}) = \sum_{\bs_{t+1:T}} \sigma(\bs_{1:T})  = \sigma(\bs_{1:t}).
\end{align*}
\normalsize
\textit{$(iii) \implies (i):$}  The target marginals clearly satisfy the recursion %
\small 
\begin{align*}
    \sigma(\bs_{1:t})\coloneqq \sum_{\bs_{t+1:T}} \sigma(\bs_{1:T}) = 
\sum_{s_{t+1}}\sum_{\bs_{t+2:T}} \sigma(\bs_{1:T}) = \sum_{s_{t+1}} \sigma(\bs_{1:t+1}).
\end{align*} 
\normalsize
\end{proof}

\vheader
\subsection{Proof of Twist-Induced Proposal}\label{app:proof-twist-ind-prop}
\vheader 
\twistinduced*
\begin{proof}
We seek to minimize the variance of the resulting importance weights, subject to a constraint on the proposal probabilities summing to 1.   Introducing a Lagrange multiplier $\lambda(\bs_{1:t-1})$, we have
\small
\begin{equation}
\resizebox{\textwidth}{!}{\ensuremath{
    \min \limits_{q(\args)} \mathbb{E}_{q(\args)}\left[ \left(\frac{\tpitwist{t}(\argstone)}{\tpitwist{t-1}(\argst)q(\args)}\right)^2 \right] -  \left( \mathbb{E}_{q(\args)}\left[ \left(\frac{\tpitwist{t}(\argstone)}{\tpitwist{t-1}(\argst)q(\args)}\right) \right]\right) ^2  + \lambda(\argst) \left( \sum_{s_{t}} q(\args) - 1 \right) 
    }}
    \nonumber
\end{equation}
\normalsize
Taking $\frac{\delta}{\delta q}(\cdot) = 0$ implies
\small
\begin{align}
    0&= \left(\frac{\tpitwist{t}(\argstone)}{\tpitwist{t-1}(\argst)q(\args)}\right)^2 
    - 2 q(\args) \left(\frac{\tpitwist{t}(\argstone)}{\tpitwist{t-1}(\argst)q(\args)}\right)\frac{\tpitwist{t}(\argstone)}{\tpitwist{t-1}(\argst)q(\args)^2} + \lambda(\argst) \nonumber 
\end{align}
\normalsize
where the derivative in the second term is zero since the $q(\args)$ cancel.  Finally, we have
\small
\begin{align}
 \frac{\tpitwist{t}(\argstone)^2}{\tpitwist{t-1}(\argst)^2 q(\args)^2}
    &= \lambda(\argst)  \nonumber \\
    q^*(\args) &= \frac{1}{\sqrt{\lambda(\argst)}} \frac{\tpitwist{t}(\argstone)}{\tpitwist{t-1}(\argst)} \quad = \frac{1}{\zonestep{t}{t-1}} \base(s_{t}|\bs_{1:t-1}) \psi_{t}(\bs_{1:t}) \nonumber
\end{align}
\normalsize
where $\zonestep{t}{t-1}$ (or $\lambda$) is chosen to enforce normalization.
\end{proof}

We focused on the one-step twist-induced proposal in \cref{prop:transition}.   However, this proposal is \textit{not optimal} for resampling every $c$ steps (as would also occur, for example, with adaptive resampling).
\begin{proposition}[\propheader{Multi-Step Twist Induced Proposal (Generalization of \cref{prop:transition})}]\label{prop:multi-step-optimal-proposal}
For resampling $c$-steps ahead, the optimal proposal (over $\bs_{t+1:t+c-1}$) which minimizes the variance of the importance weights $ w_{t:t+c-1}(\bs_{1:t+c-1})$ is given by 
\begin{align}
\proptwist{}{(\bs_{t:t+c-1}|\bs_{1:t-1})} &= 
\frac{ \base(\bs_{t:t+c-1}|\bs_{1:t-1}) \approxtwist{t+c-1}(\bs_{1:t+c-1})}{\sum \limits_{\bs_{t:t+c-1}} \base(\bs_{t:t+c-1}|\bs_{1:t-1}) \approxtwist{t+c-1}(\bs_{1:t+c-1})}. \nonumber 
\end{align}
\normalsize
\end{proposition}
The proof follows the same reasoning as in the proof of \cref{prop:transition} above, using the multistep weights $ w_{t:t+c-1}(\bs_{1:t+c-1})=\frac{\tpitwist{t+c-1}(\bs_{1:t+c-1})}{\tpitwist{t-1}(\bs_{1:t-1})q(\bs_{t:t+c-1}|\bs_{1:t-1})}$ from \cref{eq:c-step-weights}.

Note that the denominator is not usually tractable for $c > 1$ in language modeling applications.   

\vheader
\subsection{Derivation of CTL Gradient}\label{app:ctl_gradient}

\begin{lemma}[\propheader{Derivation of CTL Gradient}]
For the \gls{CTL} loss $ \min\limits_{\thb} \cL_{\text{CTL}}(\thb)$$ \coloneqq \min\limits_{\thb}\sum_{t=1}^{T}\DKL\pV{\target(\bs_{1:t})}{\pitwisttheta{t}\left(\bs_{1:t}\right)}  $,
the (negative) gradient with respect to the parameters $\thb$ is given by
\small{
\begin{equation}
\hspace*{-.2cm}
- \nabla_{\thb} \cL_{\text{CTL}}(\thb) = 
\sum \limits_{t=1}^T \EEE_{\target(\bs_{1:t})}\left[\gr_{\thb}\log\vartwist{t}\left(\bs_{1:t}\right)\right]-\EEE_{\pitwisttheta{t}\left(\bs_{1:t}\right)}\left[\gr_{\thb}\log\vartwist{t}\left(\bs_{1:t}\right)\right]
\label{eq:grad_separate_ebm_app}
\end{equation}
} \normalsize
\end{lemma}
\begin{proof}
Consider expanding the form of $\pitwisttheta{t}(\bs_{1:t})$ using \cref{eq:filtering}, noting that the normalization $\log \zfilter{t}$ is independent of $\bs_{1:t}$.  Taking the gradient with respect to $\thb$ using the log derivative identity $\nabla_{\thb} f(\thb) = f(\thb) \nabla_{\thb} \log f(\thb)$,  we have
\small
\begin{align*}
   - \nabla_{\thb} \cL_{\text{CTL}}(\thb) &= - \nabla_{\thb} \left( \sum_{t=1}^{T} \mathbb{E}_{\target(\bs_{1:t})}\left[ \log \target(\bs_{1:t}) - \log \base(\bs_{1:t}) - \log \vartwist{t}(\bs_{1:t}) \right] + \log \sum_{\bs_{1:t}}\base(\bs_{1:t})\vartwist{t}(\bs_{1:t})
   \right) \\
   &=   \sum_{t=1}^{T} \mathbb{E}_{\target(\bs_{1:t})}\left[ \nabla_{\thb} \log \vartwist{t}(\bs_{1:t}) \right]   - \sum_{t=1}^{T}  \sum_{\bs_{1:t}} \frac{\hphantom{\sum_{\bs_{1:t}}}\base(\bs_{1:t})\vartwist{t}(\bs_{1:t})}{\sum_{\bs_{1:t}}\base(\bs_{1:t})\vartwist{t}(\bs_{1:t})}  \nabla_{\thb} \Big(\log  \base(\bs_{1:t}) + \log \vartwist{t}(\bs_{1:t}) \Big)   \\
   &= \sum_{t=1}^{T} \left( \mathbb{E}_{\target(\bs_{1:t})}\left[ \nabla_{\thb} \log \vartwist{t}(\bs_{1:t}) \right]   -  \mathbb{E}_{\pitwisttheta{t}(\bs_{1:t})}\left[  \nabla_{\thb}\log \vartwist{t}(\bs_{1:t})   \right]\right)
\end{align*}
\normalsize
\end{proof}

\vheader
\vheader
\section{SMC with Intermediate \bigphishorts and Connection with Soft Reinforcement Learning}\label{app:framework}
\vheader

In the main text, we focused 
on settings where the 
target distribution is defined by a 
\phishort %
$\finaltwist$ depending on full sequences only, as in \cref{eq:posterior}.  This setting highlights the need for (learned) twist functions to summarize the future expected value of the \phishort 
in the absence of intermediate target information.  

In this appendix, we generalize our exposition to show how our twisted SMC framework can accommodate settings with intermediate 
\phishortsnospace,
which is evocative of connections with soft reinforcement learning 
\citep{levine2018reinforcement}. 
We leverage intuition from soft RL while introducing our general probabilistic interpretation, by using $\eqrl$ to instantiate the soft RL special case.  In particular, soft RL will correspond to the terminal 
\phishort 
\begin{align}
\reward{t} \eqrl e^{\beta ~ r_t(\bs_{1:t})} \label{eq:soft_reward}\tag{soft RL $\phi_t$ Definition}
\end{align}
which corresponds to $\finaltwist = e^{\beta r_T(\bs_{1:T})}$ if the 
\phishort
is given at the final step only (as in RLHF, \citet{korbak2022rl}).
However, we defer detailed discussion of soft RL to \cref{app:soft_rl}.    See \cref{table:posteriors} for several examples of intermediate 
\phishortsnospace.

Finally, we formalize a notion of conditional target distributions and twist functions in \cref{app:conditional_twist_expo}, which generalizes the exposition in the main text and captures our conditional twist learning experiments in \cref{sec:infilling}.

\vheader
\subsection{Twisted SMC with Intermediate \bigphishorts}\label{app:int_reward}
\vheader

To generalize the exposition in the main text, we might consider defining the target as
\begin{align}
    \sigma(\bs_{1:T}) \coloneqq \frac{1}{\cZ_\sigma} \base(\bs_{1:T}) \left( \prod \limits_{t=1}^{T} \reward{t} \right) 
    \eqrl \frac{1}{\cZ_\sigma} \base(\bs_{1:T}) e^{\beta \sum_{t=1}^T r_t(\bs_{1:t})}
    \label{eq:tgt_rewards}
\end{align}
where \cref{eq:posterior} and the main text exposition corresponds to $\reward{t} = 1$ for $t < T$.   

\paragraph{Optimal Twists with Intermediate \bigphishorts}
Using \cref{eq:tgt_rewards}, the marginal distribution $\sigma(\bs_{1:t}) = \sum_{\bs_{t+1:T}} \sigma(\bs_{1:T})$ over $t$ tokens becomes
\begin{align}
      \sigma(\bs_{1:t}) = & \frac{1}{\cZ_\sigma} \base(\bs_{1:t}) \left( \prod \limits_{\tau=1}^{t} \reward{\tau} \right) \left(\sum \limits_{\bs_{t+1:T}} \base(\bs_{t+1:T}|\bs_{1:t})  \prod \limits_{\tau=t+1}^{T} \reward{\tau} \right) \label{eq:int_reward_marginal} \\
      \eqrl &
      \frac{1}{\cZ_\sigma}  \base(\bs_{1:t}\condzero) \expof{ \beta \sum\limits_{\tau=1}^{t} r_\tau(\bs_{1:\tau}) } \left( \sum_{\bs_{t+1:T}} \base(\bs_{t+1:T}|\bs_{1:t}) \expof{ \beta \sum\limits_{\tau=t+1}^{T} r_\tau(\bs_{1:\tau})} \right) \tag{soft RL special case}
\end{align}
As in \cref{prop:optimal_twists}, the goal of the optimal twist functions is to facilitate sampling from the intermediate marginals $\sigma(\bs_{1:t})$ of the target distribution $\sigma(\bs_{1:T})$.

We consider two different quantities involved in defining the optimal twists, which differ in their treatment of the intermediate reward.   For the soft RL setting, this corresponds to the natural distinction between $Q$-values and (soft) value functions $V_t$.  
\begin{align}
\begin{split}
      \sigma(\bs_{1:t}) = &\frac{1}{\cZ_\sigma} \base(\bs_{1:t}) \left( \prod \limits_{\tau=1}^{t-1} \reward{\tau} \right) \underbrace{ \reward{t}\Big( \underbrace{ \sum \limits_{\bs_{t+1:T}} \base(\bs_{t+1:T}|\bs_{1:t})  \prod \limits_{\tau=t+1}^{T} \reward{\tau} }_{
      \optimaltwistphi{t}(\bs_{1:t}) \colonpropto  } \Big) }_{\psi_t^*(\bs_{1:t})\colonpropto} \label{eq:int_reward_marginal_qv} \\[2ex]
     \eqrl  &\frac{1}{\cZ_\sigma} \base(\bs_{1:t}) \left(
     \expof{ \beta \sum\limits_{\tau=1}^{{t-1}} r_\tau(\bs_{1:\tau}) } \right)
       \underbrace{ 
    \expof{ \beta ~ r_t(\bs_{1:t})} 
    \Big( \underbrace{ \sum \limits_{\bs_{t+1:T}} \base(\bs_{t+1:T}|\bs_{1:t}) 
    \expof{ \beta \sum\limits_{\tau=t+1}^{T} r_\tau(\bs_{1:\tau})}  \Big) }_{\optimaltwistphi{t}(\bs_{1:t}) \colonpropto  ~ \expof{\beta V_t^*(\bs_{1:t}) } =}}_{\psi_t^*(\bs_{1:t}) \colonpropto ~ \expof{\beta r_t(\bs_{1:t}) + \beta V_t^*(\bs_{1:t}) } = }
    \end{split}
\end{align}
where 
$\colonpropto$ means `defined to be proportional to' and
$Q_t^*(s_t, \bs_{1:t-1})= r_t(\bs_{1:t}) + V_t^*(\bs_{1:t})$ in RL notation.  See \cref{app:soft_rl} for detailed derivations in the soft RL special case.
In general, 
$\vartwistphi{t}$ captures the expectation of \textit{future} \phishorts from $t+1:T$,
 analogous to the (soft) value function. 
The twists $\approxtwist{t}$ play a role analogous to a $Q$-value, estimating both the immediate $\phi_t$ and future value $\vartwistphi{t}$. 
In particular,
\begin{align}
\psi_t^*(\bs_{1:t}) \propto \reward{t} \optimaltwistphi{t}(\bs_{1:t}) \qquad \text{where} \quad \optimaltwistphi{t}(\bs_{1:t}) \colonpropto \sum \limits_{\bs_{t+1:T}} \base(\bs_{t+1:T}|\bs_{1:t})  \prod \limits_{\tau=t+1}^{T} \reward{\tau}
\end{align}
We continue to refer to $\approxtwist{t}$ as the \textit{twist functions} and focus on 
probabilistic interpretations based 
on 
$\approxtwist{t}$ instead of $\optimaltwistphi{t}$ (see \cref{app:final_timestep} for additional discussion).

To show that this notation is consistent with the main text, consider the optimal twists $\psi_t^*(\bs_{1:t}) =\reward{t} \optimaltwistphi{t}(\bs_{1:t})$ with no intermediate 
\phishortsnospace,
$\reward{t} = 1$ for $t< T$.  
For $t<T$, $\psi_t^*(\bs_{1:t}) = \optimaltwistphi{t}(\bs_{1:t})$ reflect the future expected 
\phishort
and for $t=T$, the terminal 
\phishort
is $\psi_T^*(\bs_{1:T}) = \reward{T}$, 
with no future 
\phishorts after step $T$, i.e. %
$\vartwistphi{T}=1$.

Building on  \cref{eq:int_reward_marginal}-(\ref{eq:int_reward_marginal_qv}) above, the following generalization of \cref{prop:optimal_twists} defines the `optimal' twists so as to obtain the intermediate target marginals $\sigma(\bs_{1:t})$ (see \cref{prop:opt_intermediate}). 

\begin{proposition}[\propheader{Optimal Twists}]\label{prop:opt_twists_int_reward}
For a given target distribution $\target(\bs_{1:T}\condzero)$ in \cref{eq:tgt_rewards},
the optimal twist functions yield intermediate $\{\pi_t\}_{t=1}^{T-1}$ which match the target marginals.   
In regions where $\base(\bs_{1:t})>0$, the optimal twists are given by
\begin{align}
\pitwist{t}^*(\bs_{1:t}) = 
\target(\bs_{1:t}) 
&= \frac{1}{\zfilterstar{t}}~  \base(\bs_{1:t}) \left( \prod \limits_{\tau=1}^{t-1} \reward{\tau} \right) 
\optimaltwist{t}(\bs_{1:t}) &= \frac{1}{\cZ_{{t}}^{\optimaltwistphi{}}}~  \base(\bs_{1:t}) \left( \prod \limits_{\tau=1}^{t-1} \reward{\tau} \right) \reward{t}
\optimaltwistphi{t}(\bs_{1:t}). \label{eq:optimal_filtering_app}
\end{align}
Up to a constant $c_t$ independent of $\bs_{1:t}$, the optimal twists $\optimaltwist{t}$ are given by
\begin{align}
 \optimaltwist{t}(\bs_{1:t}) &= c_t ~ \reward{t} \sum \limits_{\bs_{t+1:T}} \base(\bs_{t+1:T}|\bs_{1:t})  \prod \limits_{\tau=t+1}^{T} \reward{\tau} 
 \label{eq:optimal_twists_app} 
\end{align}
where $c_t$ is absorbed into the normalization constant $\zfilterstar{t}$.  
The optimal twists satisfy the recursion
\begin{align}
 \optimaltwist{t}(\bs_{1:t}) &= 
 \frac{\zfilterstar{t}}{\zfilterstar{t+1}}
 \reward{t}
 \sum \limits_{s_{t+1}} \base(s_{t+1}|\bs_{1:t})  
 \optimaltwist{t+1}(\bs_{1:t+1}) . \label{eq:app_recursion}
\end{align}
\end{proposition}

\begin{remark}[\propheader{Equivalence Class of $\psi_{t}$ and $\varPhi_{t}$}]
Note that any rescaling of $\psi_{t} \gets c_t \bar{\psi}_t$ by a constant with respect to $\bs_{1:t}$ will yield the same intermediate marginals $\pi_t(\bs_{1:t})$, due to the normalization constant $\cZ_{{t}}^{\psi}$ which scales with $\psi_{t}$.   This defines an equivalent class in the space of functions. The same statement holds for $\varPhi_{t}$. We express results such as \cref{eq:optimal_twists_app} using proportionality $\propto$. We define $\psi_{t}$ and $\varPhi_{t}$ as the members of their equivalent classes whose normalization $\cZ_{{t}}^{\psi}$ and $\cZ_{{t}}^{\varPhi}$ are equal. Thus, we have $\psi_{t}(\bs_{1:t}) = \reward{t} \varPhi_{t}(\bs_{1:t})$.
\end{remark}

\begin{proof}
Substituting \cref{eq:optimal_twists_app} into \cref{eq:optimal_filtering_app}, we obtain the desired marginal \cref{eq:int_reward_marginal},
\small 
\begin{align}
\pitwist{t}^*(\bs_{1:t}) 
= \frac{c_t}{\zfilterstar{t}}~  p_0(\bs_{1:t}\condzero) ~ 
\prod \limits_{\tau=1}^{t} \reward{\tau} 
\left( \sum \limits_{\bs_{t+1:T}} \base(\bs_{t+1:T}|\bs_{1:t})  \prod \limits_{\tau=t+1}^{T} \reward{\tau} \right) = 
\target(\bs_{1:t}) \nonumber
\end{align}
\normalsize
where the final equality 
follows from absorbing the constant $c_t$ into $\zfilterstar{t}$, 
with $\frac{1}{\cZ_\sigma} = \frac{c_t}{\zfilterstar{t}}$ and $\cZ_\sigma$ which normalizes $\tilde{\sigma}(\bs_{1:t})$.
We will now use 
$c_t = \frac{\zfilterstar{t}}{\cZ_\sigma}$ to show the recursion in \cref{eq:app_recursion}. 
Note that \cref{eq:optimal_twists_app} implies
\small 
\begin{align}
 \optimaltwist{t}(\bs_{1:t}) &= c_t ~ \reward{t} ~ \sum \limits_{s_{t+1}} \base(s_{t+1}|\bs_{1:t})  \Big( \underbrace{ \reward{t+1} \sum \limits_{\bs_{t+2:T}} \base(\bs_{t+2:T}|\bs_{1:t+1}) \prod \limits_{\tau=t+2}^{T} \reward{\tau} }_{ \frac{1}{c_{t+1}} \optimaltwist{t+1}(\bs_{1:t+1}) } \Big) \nonumber \\
 &= 
 \frac{\zfilterstar{t}}{\zfilterstar{t+1}} \reward{t}
 \sum \limits_{s_{t+1}} \base(s_{t+1}|\bs_{1:t})  
 \optimaltwist{t+1}(\bs_{1:t+1}) \nonumber
\end{align}\normalsize
 where the second line follows from $\frac{c_t}{c_{t+1}} = \frac{\zfilterstar{t}/\cZ_\sigma}{\zfilterstar{t+1}/\cZ_\sigma}$.   This demonstrates \cref{eq:app_recursion}.  
\end{proof}
This leads to the following definition of the intermediate twisting targets (we defer the soft RL special case to \cref{app:soft_rl}).
\begin{definition}[\propheader{Twisted Intermediate Targets }]\label{def:filtering_app}
Using approximate twist functions $\{ \approxtwist{t}\}_{t=1}^{T-1}$, we define the twisted intermediate target distributions
\begin{empheq}[
  left= {
  \pitwist{t}(\bs_{1:t}) 
   = \empheqlbrace }]
  {align}
  & \frac{1}{\zfilter{t}}~  p_0(\bs_{1:t}\condzero)  \left( \prod \limits_{\tau=1}^{t-1} \reward{\tau} \right) ~
\approxtwist{t}(\bs_{1:t}) \qquad (t < T) \tag{Twist Targets ($\approxtwist{}$) 
} \label{eq:filtering_app_psi}  \\
  & \frac{1}{\cZ_\sigma} \base(\bs_{1:T}) \prod_{t=1}^T \reward{t} \qquad \quad ~~ \qquad \qquad (t = T) \nonumber
\end{empheq}
\end{definition}

\vheader
\paragraph{One-Step Twist-Induced Proposal}
Using 
\cref{prop:transition} and
\cref{def:filtering_app} and noting that $\reward{t-1}$ is independent of $s_t$, we have the optimal one-step proposal
\begin{align}
\begin{split}
\proptwist{t}{(s_{t}|\bs_{1:t-1})} \propto \frac{\pitwist{t}(\bs_{1:t})}{\pitwist{t-1}(\bs_{1:t-1})} 
&= \frac{\zfilter{t-1}}{\zfilter{t}}  \base(s_{t}|\bs_{1:t-1}) \frac{\reward{t-1} \approxtwist{t}(\bs_{1:t}) }{\approxtwist{t-1}(\bs_{1:t-1})} \\
&\eqqcolon \frac{1}{\zonestep{t}{t-1}} \base(s_{t}|\bs_{1:t-1}) 
\approxtwist{t}(\bs_{1:t})\\
&= \frac{ \hphantom{\sum \limits_{s_t}} \base(s_{t}|\bs_{1:t-1}) 
\approxtwist{t}(\bs_{1:t})}{\sum \limits_{s_t} \base(s_{t}|\bs_{1:t-1}) 
\approxtwist{t}(\bs_{1:t}) }
\end{split}
\label{eq:prop_twist}\tag{Twist-Induced Proposal ($\approxtwist{}$) } 
\end{align}
where in the second line, we absorb terms which depend only on $\bs_{1:t-1}$ (and not $s_t$) into the normalization.
 In the soft RL special case, we have $\proptwist{t}{(s_{t}|\bs_{1:t-1})} \propto  \base(s_{t}|\bs_{1:t-1})  e^{\beta Q_t(s_t,\bs_{1:t-1})}$ (see \cref{eq:twist_induced_rl_app} below).

\vheader
\subsection{Conditional Twisted SMC}%
\label{app:conditional_twist_expo}
\vheader
To formalize our notion of conditional twists in the infilling experiments (\cref{sec:infilling}), we generalize our above framework to explicitly depend on
`observation' random variables $\{ \myo_t \}_{t=1}^T$.  This matches the common setting of SMC in state-space models \citep{briers2010smoothing, gu2015neural, lawson2022sixo, chopin2020introduction}.
Our derivations in this section also emphasize that the optimal twist functions in \cref{prop:opt_twists_int_reward} learn functions proportional to \textit{conditional likelihoods} of the future observation 
variables given the current sequence (see \cref{eq:cond_twist_lkd} below)).
We recover the unconditional targets in the main text for fixed $o_T= 1$.

Consider a target distribution $\sigmacond{t}(\bs_{1:T}| \ocond )$ conditioned on particular 
observation random variables $\ocond \coloneqq \{ o_{t}\}_{t=1}^{T}$.  
We define a probabilistic model over observations  $\sigmacond{t}(\myo_t|\bs_{1:t}) = \phi_t(o_t,\bs_{1:t})$~ as the intermediate \phishortnospace,\footnote{Note, 
rescaling $\rewardone{t}$ by a constant $c$ with respect to $o_t, \bs_{1:t}$ does not affect the target posterior in \cref{eq:cond_optimality}.  For example, 
with terminal \phishort only:  
$\sigmacond{T}(\bs_{1:T}|o_{T}) = \frac{ \base(\bs_{1:T})  ~ \phi_T(\bs_{1:T},o_T) /  
c
}{\sum_{\bs_{1:T}} \base(\bs_{1:T}) ~ \phi_T(\bs_{1:T},o_T) /  
c %
} = \frac{1}{\cZ_\sigma(o_T)} \base(\bs_{1:T}) \phi_T(\bs_{1:T},o_T)$ as long as the scaling factor 
is independent of $o_T$ and $\bs_{1:T}$. 
}  which yields the target posterior
\small 
\begin{align}
\begin{split}
\sigma(\bs_{1:T}|\bo_{1:T}) &= \frac{\phantom{\sum_{\bs_{1:T}} } \base(\bs_{1:T}) \left(\prod\limits_{t=1}^{T} \sigmacond{t}(\myo_t |\bs_{1:t}) \right) %
}{ \sum \limits_{\bs_{1:T}} \base(\bs_{1:T}) \left( \prod\limits_{t=1}^{T} \sigmacond{t}(\myo_t |\bs_{1:t}) \right)%
} 
= \frac{1}{\cZ_\sigma(\bo_{1:T})} \base(\bs_{1:T}) \left(\prod\limits_{t=1}^{T}\phi_t(o_t,\bs_{1:t}) \right) 
= \frac{ \base(\bs_{1:T}) \sigmacond{T}(\bo_{1:T} | \bs_{1:T})
}{\sigmacond{T}(\bo_{1:T} )}
\end{split}
\label{eq:cond_optimality}
\end{align}
\normalsize
where we interpret $\sigmacond{T}(\bo_{1:T} | \bs_{1:T}) = \prod_{t=1}^T \sigmacond{T}(\myo_t |\bs_{1:t})$ and $\cZ_\sigma(\bo_{1:T})=\sigmacond{T}(\bs_{1:T})$ to make the Bayesian posterior explicit in the last equality. 
  Note, we now seek to estimate a different partition function $\cZ_\sigma(\bo_{1:T})$ for each set of observation variables.

Using our infilling experiments in \cref{sec:infilling} as an example, consider (a sequence of) subsequent tokens $o_T = \bs_{T+1:T+c}$ as observation variables, where the observation model is simply the base language model $\sigmacond{}(o_T|\bs_{1:T}) \coloneqq  \base(\bs_{T+1:T+c} | \bs_{1:T})$.

Using \cref{eq:cond_optimality}, the intermediate marginals become 
\begin{align}
    \sigmacond{t}(\bs_{1:t}|\bo_{1:T}) &= \sum_{\bs_{t+1:T}} \sigmacond{T}(\bs_{1:T}|\bo_{1:T}) \nonumber \\
&= \sum_{\bs_{t+1:T}} \frac{1}{\sigmacond{T}(\bo_{1:T})} \base(\bs_{1:t})\base(\bs_{t+1:T}|\bs_{1:t}) \Big( \prod\limits_{t=1}^{T} \sigmacond{t}(\myo_t |\bs_{1:t}) \Big)   \nonumber \\
&= \frac{1}{\cZ_\sigma(\bo_{1:T})}
\base(\bs_{1:t})\left( \prod \limits_{\tau=1}^{t} \phi_\tau(o_\tau,\bs_{1:\tau}) \right)
\sum \limits_{\bs_{t+1:T}}  \base(\bs_{t+1:T}|\bs_{1:t}) \Big( \prod \limits_{\tau=t+1}^{T} \phi_\tau(o_\tau,\bs_{1:\tau})  \Big) %
\nonumber \\
&= \frac{1}{\cZ_\sigma(\bo_{1:T})}
\base(\bs_{1:t})\left( \prod \limits_{\tau=1}^{t} \phi_\tau(o_\tau,\bs_{1:\tau}) \right)
\sigma(\bo_{t+1:T}|\bs_{1:t})\,,
\label{eq:filtering_intermediate_optimality_conditional} 
\end{align}
noting that $\sigma(\bo_{t+1:T}|\bs_{1:t}) = \sum_{\bs_{t+1:T}} \sigma(\bo_{t+1:T}, \bs_{t+1:T}|\bs_{1:t})$ matches the second to last line.    

The optimal twists take a similar form as \cref{prop:opt_twists_int_reward}, but now as a function of the future observation or conditioning information.  Further, 
the optimal twists is proportional to the
conditional likelihoods (e.g., $\sigma(\bo_{t+1:T} |\bs_{1:t})$) of future observations given $\bs_{1:t}$, 
which marginalize over future tokens (e.g., $\bs_{t+1:T}$),
\begin{align}
\begin{split}
\optimaltwistphi{t}(\bs_{1:t}, \bo_{t+1:T}) &\proptoo{\bo_{t+1:T}} \sigma(\bo_{t+1:T} |\bs_{1:t}) = \sum \limits_{\bs_{t+1:T}}  \base(\bs_{t+1:T}|\bs_{1:t}) \Big( \prod \limits_{\tau=t+1}^{T} \phi_\tau(o_\tau,\bs_{1:\tau})  \Big)\,, \\
\qquad \qquad 
\optimaltwist{t}(\bs_{1:t}, \bo_{t:T} ) &\proptoo{\bo_{t:T}} \sigma(\bo_{t:T} |\bs_{1:t}) = \sum \limits_{\bs_{t+1:T}} \base(\bs_{t+1:T}|\bs_{1:t}) \Big( \prod \limits_{\tau=t}^{T} \phi_\tau(o_\tau,\bs_{1:\tau})  \Big)\,, %
\end{split}\label{eq:cond_twist_lkd}
\end{align}
where $f(x,\bo) \proptoo{\bo} g(x,\bo)$ denotes proportionality
up to a constant which depends on $\bo$ only: 
$\exists c(\bo) \colon f(x,\bo)=c(\bo) g(x,\bo)$. These equations can be confirmed by comparing \cref{prop:opt_twists_int_reward} with the last two lines in \cref{eq:filtering_intermediate_optimality_conditional}.

The intermediate marginals over partial sequences can finally be rewritten as either %
\begin{align}
\begin{split}
\sigma(\bs_{1:t}|\bo_{1:T}) &\proptoo{\bo_{1:T}}  \base(\bs_{1:t}) \left(\prod_{\tau=1}^{t} \phi_\tau( o_\tau, \bs_{1:\tau} )  \right) \optimaltwistphi{t}(\bs_{1:t}, \bo_{t+1:T})\,, \\
&= \base(\bs_{1:t}) \left(\prod_{t=1}^{t-1} \phi_\tau( o_\tau, \bs_{1:\tau} )  \right) \optimaltwist{t}(\bs_{1:t}, \bo_{t:T})\,.
\end{split}
\end{align}
We discuss the choice of parameterization using $\approxtwist{t}$ versus $\vartwistphi{t}$ in \cref{app:final_timestep}.

The conditional twist learning formulation matches the setting of \citet{lawson2022sixo}, to which we refer the reader for additional discussion.  We use this conditional perspective to derive classification losses for twist learning in \cref{app:sixo}-\ref{app:fudge}.

\vheader
\paragraph{Unconditional Targets as a Special Case}
In cases where we are only learning twists for a single set of conditioning information such as a single classifier label or a reward model, note that we can omit explicit conditioning information in $\psi_{t}(\bs_{1:t},o_t)$ and consider setting $\{o_t = 1\}_{t=1}^T$.  
  With terminal \phishort only as in the main text, we write $\sigmacond{T}(o_T = 1 | \bs_{1:T}) = \finaltwist$ and the overall target distribution as $\sigma(\bs_{1:T}) = \sigmacond{T}(\bs_{1:T}|o_T=1) \propto \base(\bs_{1:T})\reward{T}$.
Thus, the formulation in \cref{eq:cond_optimality}-\cref{eq:cond_twist_lkd} strictly generalizes our exposition in the main text and \cref{app:int_reward}.
With intermediate \phishortsnospace, we set $\sigma(\myo_{1:T}=\boldsymbol{1} | \bs_{1:T}) = \prod_{t=1}^T \reward{t}$.

Our notation also matches the exposition in \citet{levine2018reinforcement} for the soft RL case with a binary observation or `optimality' random variable $\sigma(\myo_t =1 | \bs_{1:t-1}, s_t) = e^{ \beta r_t(\bs_{1:t-1}, s_t) }$, where the reward is a function of the state $x_t = \bs_{1:t-1}$ and action $a_t = s_t$ pair (see the MDP interpretation in \cref{app:soft_rl}).

\vheader
\subsection{Connection with Soft Reinforcement Learning}\label{app:soft_rl}
\vheader
In this section, we more explicitly describe the soft reinforcement learning setting \citep{levine2018reinforcement} commonly used in \gls{RLHF} \citep{korbak2022rl} as a special case of our probabilistic framework.   
Again, we use notation $\eqrl$ to indicate that the expressions in this section correspond to a particular instance of our SMC framework where $\finaltwist = e^{\beta r(\bs_{1:T})}$. 

\vheader
\paragraph{Summary of Soft RL Notation}  To summarize the below derivations, we state the relevant assignments for the soft RL case.   We focus on the optimal case for simplicity, but note that approximate versions play identical roles,
\begin{align}
\reward{t} = e^{\beta ~ r_t(\bs_{1:t})} \qquad  \optimaltwist{t}(\bs_{1:t}) = e^{\beta r_t(\bs_{1:t}) + \beta V_t^*(\bs_{1:t})} = e^{\beta Q_t^*(s_t,\bs_{1:t-1})} \qquad \optimaltwistphi{t}(\bs_{1:t}) = e^{\beta V_t^*(\bs_{1:t})}
\tag{Twist to Soft RL}\label{eq:twist_to_soft_rl}
\end{align}
where $\optimaltwist{t}(\bs_{1:t}) = \reward{t} \optimaltwistphi{t}(\bs_{1:t})$ or $Q_t^*(s_t,\bs_{1:t-1}) = r_t(\bs_{1:t}) + V_t^*(\bs_{1:t})$.   In the other direction, we have
\begin{align}
 r_t(\bs_{1:t}) = \frac{1}{\beta} \log \reward{t}
\qquad Q_t^*(s_t,\bs_{1:t-1}) = \frac{1}{\beta}\log \optimaltwist{t}(\bs_{1:t}) \qquad  V_t^*(\bs_{1:t}) = \frac{1}{\beta}\log \optimaltwistphi{t}(\bs_{1:t}) 
\tag{Soft RL to Twist} \label{eq:soft_rl_t\myo_twist}
\end{align}
\vheader
\paragraph{MDP Interpretation}
To draw connections with soft RL, we view language model controlled decoding as a 
\gls{MDP}, where the prompt is drawn from an initial state distribution $\bs_0 \sim \nu_0$, an action policy $\pi(a_{t} | x_t) = \prop(s_{t} | \bs_{1:t-1})$ selects the next token $a_t = s_{t}$ given a partial sequence $x_t = \bs_{1:t-1}$ as the state, and deterministic environment transitions $P(x_{t+1} = \bs_{1:t} | a_t= s_{t}, x_t = \bs_{1:t-1}) = \delta(x_{t} = \text{concat}({s_{t}, \bs_{1:t-1})})$ append the selected token to update the state.  Discounting may also be included without difficulty.
The reward is given by $r_t(\bs_{1:t})$.

\vheader
\paragraph{Final Target Distribution}
We define the target distribution 
as the solution to the following variational optimization which solves the regularized MDP described above,
\begin{align}
\begin{split}
\target(\bs_{1:T}\condzero)
\eqrl \frac{1}{\cZ_\sigma}  \base(\bs_{1:T}\condzero) \expof{ \beta \sum\limits_{t=1}^{T} r_t(\bs_{1:t}) } %
&=\medmath{ \argmax %
\limits_{\optdist(\bs_{1:T}\condzero)} \mathbb{E}_{\optdist(\bs_{1:T}\condzero)}\Big[   
\sum\limits_{t=1}^{T} r_t(\bs_{1:t})\Big] 
-
\frac{1}{\beta}
\klof{ \optdist(\bs_{1:T})}{ \base(\bs_{1:T}\condzero) } }
\end{split}\label{eq:variational_post}
\intertext{
which corresponds to the choice 
$\reward{t} = e^{\beta ~ r_t(\bs_{1:t})}$ as in \cref{eq:twist_to_soft_rl}.
The soft value is defined as the maximum value of the above optimization for optimal $\optdist^*(\bs_{1:T})$, and corresponds to the scaled log partition function }
\begin{split}
\vopt_0(\bs_0) \coloneqq \frac{1}{\beta} \log \cZ_\sigma = \frac{1}{\beta} \log \sum\limits_{\bs_{1:T}} \base(\bs_{1:T}\condzero) \expof{\beta \sum\limits_{t=1}^{T}  r_t(\bs_{1:t})  }  
&=
\medmath{
\max \limits_{\optdist(\bs_{1:T}\condzero)} \mathbb{E}_{\optdist(\bs_{1:T}\condzero)}\Big[  
\sum\limits_{t=1}^{T} r_t(\bs_{1:t})\Big] 
-
\frac{1}{\beta}
\klof{ \optdist(\bs_{1:T})}{ \base(\bs_{1:T}\condzero) }
}
\end{split}\label{eq:variational_value}
\end{align}
which can be confirmed by substituting $ \optdist(\bs_{1:T}) = \target(\bs_{1:T})
$ from \cref{eq:variational_post} into the maximization on the right side of \cref{eq:variational_value}. %
Although we omit the dependence of $\cZ_\sigma(\bs_0)$ on the prompt $\bs_0$ for notational simplicity (see \cref{eq:posterior}), note that $\vopt_0 \coloneqq \vopt(\bs_0)$ naturally corresponds to the soft value of the prompt as the initial state in the MDP. 

\vheader
\paragraph{Optimal Intermediate Marginals and Soft Value}
Decomposing the maximization in \cref{eq:variational_value} into optimizations over each $\optdist(s_{t+1}|\bs_{1:t})$, we define the intermediate soft value $\vopt_{t}(\bs_{1:t})$ as the maximum of the expected future regularized reward 
\begin{align} %
\vopt_t(\bs_{1:t})= \frac{1}{\beta} \log \optimaltwistphi{t}(\bs_{1:t})  &\eqrl \frac{1}{\beta} \log \sum \limits_{\bs_{t+1:T}} \base(\bs_{t+1:T} | \bs_{1:t})~ \expof{ \beta 
 \sum \limits_{\tau=t+1}^{T} r_\tau(s_{1:\tau})} \label{eq:soft_value_app} \tag{Optimal Intermediate Soft Value} \\
 &= 
\medmath{ \max \limits_{\optdist(\bs_{t+1:T} | \bs_{1:t})} ~ \mathbb{E}_{\optdist(\bs_{t+1:T} | \bs_{1:t})}\Big[ \sum \limits_{\tau=t+1}^{T}  r_\tau(\bs_{1:\tau})
 \Big] - 
 \frac{1}{\beta} 
 \DKL\pV{\optdist(\bs_{t+1:T} | \bs_{1:t}) }{ \base(\bs_{t+1:T} | \bs_{1:t})}} \nonumber  \\
 &= \medmath{ \max \limits_{\optdist(s_{t+1}|\bs_{1:t})} \mathbb{E}_{\optdist(s_{t+1}|\bs_{1:t})}\Big[  
r_{t+1}(\bs_{1:t+1}) + \vopt_{t+1}(\bs_{1:t+1}) \Big] - \frac{1}{\beta} \klof{ \optdist(s_{t+1}|\bs_{1:t})}{ \base(s_{t+1}|\bs_{1:t}) } } \nonumber %
\end{align}
\normalsize
where, in the third line, we isolate the optimization over $q(s_t|\bs_{1:t-1})$ by (i) assuming optimality at $\tau < t$ and (ii) substituting the optimal value 
$\vopt_{t+1}(\bs_{1:t+1})=\max_{\optdist(\bs_{t+2:T} | \bs_{1:t+1})}[...]$
of the maximization over $ \optdist(\bs_{t+2:T} | \bs_{1:t+1})$ (i.e. recursively applying the second line). 

The optimal intermediate marginal can be written using either $\vopt_t(\bs_{1:t})$ or $Q_t^*(s_t, \bs_{1:t-1})$  form (as in \cref{eq:int_reward_marginal_qv} above, or by substituting the optimal $\vopt_t$ or $Q_t^*$ into the twist targets below).   

 \vheader
\paragraph{Twisted Intermediate Targets} 
We state the approximate twisting targets for \textit{both} $V_t$ or $Q_t$ parameterizations in order to make connections with soft RL losses in \cref{app:twist}.   
For approximate $V_t(\bs_{1:t})$ or $Q_t(s_t, \bs_{1:t-1})$, we have %
\begin{align}
\hspace*{-.2cm}
      \pitwist{t}(\bs_{1:t})  \eqrl  &\frac{1}{\cZ_t^{V}} \base(\bs_{1:t}) \expof{ \beta \sum\limits_{\tau=1}^{t-1} r_\tau(\bs_{1:\tau}) } \expof{\beta r_t(\bs_{1:t}) + \beta V_t(\bs_{1:t}) }   \tag{Twist Targets (Soft RL V)
      }  \hphantom{\expof{\beta Q_t^*(s_t, \bs_{1:t-1})}}  \hspace*{-.8cm} (t < T)   \label{eq:filtering_app_v} \\
      = &\frac{1}{\cZ_t^{Q}} \base(\bs_{1:t}) \expof{ \beta \sum\limits_{\tau=1}^{t-1} r_\tau(\bs_{1:\tau})} \expof{ \beta Q_t(s_t, \bs_{1:t-1}) } \hphantom{\expof{ \beta r_t(\bs_{1:t}) +  \beta V_t^*(\bs_{1:t})}} \hspace*{-.8cm} (t < T)  
       \label{eq:filtering_app_q}
       \tag{Twist Targets (Soft RL Q)
       }
\end{align}
where the final twisting target is given by \cref{eq:variational_post} and
the optimal $Q$-values are defined as
\begin{align}
Q_t^*(s_t, \bs_{1:t-1}) = r_t(\bs_{1:t}) +  V_t^*(\bs_{1:t}) 
\end{align}

\vheader
\paragraph{One-Step Proposal}
Finally, the optimal one-step proposal (e.g. in $V_t$ form) can be derived either (i) as the twist-induced proposal from \cref{eq:filtering_app_v} and \cref{prop:opt_twists_int_reward} or (ii) as the solution to the one-step optimization in the third line of \cref{eq:soft_value_app}.  As in \cref{eq:prop_twist},
\begin{align}
\optdist_{t}^\pi(s_{t}|\bs_{1:t-1}) \eqrl \frac{\hphantom{\sum_{s_{t}}} \base(s_{t}|\bs_{1:t-1}) \expof{\beta \left( r_{t}(\bs_{1:t}) + V_t(\bs_{1:t}) \right)} }{ \sum_{s_{t}} \base(s_{t}|\bs_{1:t-1}) \expof{\beta \left( r_{t}(\bs_{1:t}) + V_t(\bs_{1:t}) \right)} } \propto \base(s_{t}|\bs_{1:t-1}) \expof{\beta Q_t(s_t, \bs_{1:t-1})}. 
\label{eq:twist_induced_rl_app} \tag{Twist-Induced Proposal (soft RL)}
\end{align} 
We define the one-step log normalization constant induced by an approximate $V_t$ or $Q_t$ as $V_{V_t}$ or $V_{Q_t}$, respectively,
\begin{align}
V_{V_t}(\bs_{1:t-1}) \coloneqq \frac{1}{\beta} \log \sum_{s_{t}} \base(s_{t}|\bs_{1:t-1}) \expof{\beta \left( r_{t}(\bs_{1:t}) + V_t(\bs_{1:t}) \right)} \qquad  
V_{Q_t}(\bs_{1:t-1}) \coloneqq \frac{1}{\beta} \log \sum_{s_{t}} \base(s_{t}|\bs_{1:t-1}) \expof{\beta  Q_{t}(s_t, \bs_{1:t-1})} \label{eq:induced_value}
\end{align}
such that, for example, $\optdist_{t}^\pi(s_{t}|\bs_{1:t-1})  =  \base(s_{t}|\bs_{1:t-1}) \expof{\beta Q_t(s_t, \bs_{1:t-1}) - \beta V_{Q_t}(\bs_{1:t-1})}$.

\vheader
\paragraph{RLHF Minimizes $\DKL\pV{q}{\sigma}$}
Note that, for a given suboptimal $\optdist(\bs_{1:T}\condzero)$, the value %
of the variational optimization in \cref{eq:variational_post} is a lower bound on the (scaled) log partition function {$V^*_0 = \frac{1}{\beta} \log \cZ_\sigma$}.   Similarly to the standard Evidence Lower Bound, the gap in this lower bound is given by the \textsc{KL} divergence
\begin{align}
\frac{1}{\beta} \log \cZ_\sigma = 
    \underbrace{ \frac{1}{\beta} \DKL\pV{q(\bs_{1:T}\condzero)}{ \target(\bs_{1:T}\condzero) } \vphantom{\sum\limits_{t=1}^{T}} }_{ \text{ELBO gap } (\geq 0) } + \Big(  
    \underbrace{ \mathbb{E}_{\optdist(\bs_{1:T}\condzero)}\Big[ 
\sum\limits_{t=1}^{T} r_t(\bs_{1:t})\Big]- \frac{1}{\beta}
\klof{ \optdist(\bs_{1:T})}{ \base(\bs_{1:T}\condzero) } }_{\text{`ELBO': 
 \cref{eq:variational_post}}} \Big) \label{eq:elbo_gap}
\end{align}

In this sense, we consider soft RL or policy gradient methods such as \gls{PPO} which optimize \cref{eq:variational_post} as targeting $ \target(\bs_{1:T}\condzero)$
by minimizing $\DKL\pV{q(\bs_{1:T}\condzero)}{ \target(\bs_{1:T}\condzero)}$ \citep{korbak2022rl}.
\vheader
\subsection{
Remarks on Parameterization
}\label{app:final_timestep}
\vheader
 While the twisting targets (\cref{eq:filtering_app_psi}) and twist-induced proposal (\cref{eq:prop_twist}) may equivalently be parameterized using approximate $\vartwistphi{t}$, we focus on the $\approxtwist{t}$ parameterization to match the main text.   In particular, recall that the optimal twists satisfy $\optimaltwist{t}(\bs_{1:t}) = \reward{t} \optimaltwistphi{t}(\bs_{1:t})$ 
for all $t$.   With no intermediate \phishort ($\phi_t = 1$ for $t < T$), 
our approximate twists estimate $\approxtwist{t}(\bs_{1:t}) \approx \optimaltwistphi{t}(\bs_{1:t}) \propto \sum_{\bs_{t+1:T}} \base(\bs_{t+1:T} | \bs_{1:t}) \reward{T}$ for $t < T$.   
In this section, we describe how the presence of intermediate \phishorts may affect the choice of twist parameterization.

The 
twist-induced proposal may not be tractable to evaluate at the final timestep, since it may be costly to evaluate the terminal \phishort $\reward{T}$ for all $s_T \in \cV$ given a context $\bs_{1:T-1}$ (as described in \cref{sec:proposals}).   Thus, we learn an approximate $\approxtwist{T}(\bs_{1:T}) \approx \reward{T}$ for proposal sampling, which can be easily evaluated over $|\cV|$ next tokens.   The final $\pi_T(\bs_{1:T}) = \sigma(\bs_{1:T})$ is defined using $\finaltwist$ in order to preserve unbiased estimation.  However, after sampling the proposal according to $\approxtwist{T}$, we only need to evaluate $\finaltwist$ over $K$ full sequences to calculate the importance weights at the final step (\cref{eq:locally_normalized_is_weights}).   See \textit{Intermediate 
\bigphishort
Tractable over $K$ Sequences Only} paragraph below.

\vheader
\paragraph{Intermediate \bigphishorts Tractable over $|\cV|$ Sequences}  
However, in settings where the intermediate \phishorts $\reward{t}$ \textit{are} tractable to calculate for all $s_t \in \cV$ given $\bs_{1:t-1}$ (e.g. using an indicator function or forward pass in a transformer architecture, as in \cref{table:posteriors}), it may be useful to use a $\vartwistphi{t}$ parameterization of the twist targets and twist-induced proposal.   This allows us to use the \textit{exact} immediate \phishorts $\reward{t}$ alongside an estimated $\vartwistphitheta{t}$, instead of an approximate $\vartwist{t} \approx \phi_t \optimaltwistphi{t}$ which estimates both the immediate $\phi_t$ and future expected value of \phishorts $\optimaltwistphi{t}$.  Using notation established in \cref{eq:int_reward_marginal_qv} and \cref{prop:opt_twists_int_reward}, the twisting targets in \cref{eq:filtering_app_psi} can be rewritten using a $\vartwistphitheta{t}$ parameterization
\begin{align}
  \pitwisttheta{t}(\bs_{1:t}) 
   = 
  & \frac{1}{\zfilter{t}}~  p_0(\bs_{1:t}\condzero)  \left( \prod \limits_{\tau=1}^{t-1} \reward{\tau} \right) ~ \reward{t}
\vartwistphitheta{t}(\bs_{1:t}) \qquad (t < T) \tag{Twist Targets ($\vartwistphi{}$)
} \label{eq:filtering_app_phi}  
\end{align}
with $\pitwist{T}(\s_{1:T}) = \sigma(\s_{1:T})$ as before. 
The twist-induced proposal $\proptwist{t}{(s_{t}|\bs_{1:t-1})} \propto  \base(s_{t}|\bs_{1:t-1}) \reward{t} \vartwistphitheta{t}(\bs_{1:t})$ and its normalization constant are tractable in this case, by evaluating both the given $\reward{t}$ and parameterized $\vartwistphitheta{t}(\bs_{1:t})$ in a single forward pass and normalizing over the discrete vocabulary of next tokens.

\vheader
\paragraph{Intermediate \bigphishorts Tractable over $K$ Sequences Only}  
In cases where the intermediate 
\phishorts are
difficult to evaluate, we would like to limit evaluation of $\reward{t}$ to only $K$ partial sequences.   
In this case, parameterizing the twisted targets $\pitwist{t}$ using $\vartwist{t}$ or $Q_t^{\thb}$ (\cref{eq:filtering_app_psi}, \cref{eq:filtering_app_q}) instead of $ \vartwistphitheta{t}$ or $V_t^{\thb}$ may be preferable to ensure a tractable twist-induced proposal.
Separate parameterizations of the proposal (using $\psi_{t}^{\qparam}$) and targets ($\phi_{t} \vartwistphitheta{t}$) might also be considered.

In the case of the final timestep described above or in \cref{sec:proposals}, note that we use a learned $\psi_{T}^{\qparam}$ to parameterize a tractable variational proposal $q_T(s_T|\bs_{1:T-1})$.   In this case, we have no future value $\vartwistphi{T}(\bs_{1:T}) = 1$ and only need to evaluate the terminal \phishort $\finaltwist$ for calculating importance weights over $K$ sequences.

\newcommand{\sqlogloss}{{\log\text{Cons}}}
\newcommand{\policyevalloss}{{\text{Cons}}}

\vheader
\section{Twist Learning Losses}\label{app:twist}
\vheader
In this section, we describe various methods for twist learning beyond our proposed contrastive twist learning (CTL) procedure from \cref{sec:learning}.   In \cref{sec:consistency_twist}, we first describe several losses from the soft RL literature from a probabilistic perspective, building closely on our developments in \cref{app:int_reward}.   We then proceed to describe SIXO \citep{lawson2022sixo} and FUDGE \citep{yang2021fudge} in \cref{app:sixo}-\ref{app:fudge}.    

We emphasize losses found in related work or used as experimental baselines using equation tags (e.g. \cref{eq:sixo_nce}), where equations \cref{eq:rl_baseline}, \cref{eq:sixo_nce}, \cref{eq:fudge} are used in our experiments.
We consider settings with intermediate \phishorts in \cref{sec:consistency_twist}, but focus on the ($\phi_t=1$ for $t<T$) setting in the remainder of the section, as in the main text.

\vheader
\subsection{
Soft Q-Learning (RL) and Path Consistency Losses
from Log Importance Weights}\label{sec:consistency_twist}
\vheader

From the probabilistic perspective of %
the SMC log importance weights, we can derive several losses for twist learning, including soft Q-learning and path consistency learning (PCL) \citep{nachum2017bridging} losses from the soft RL literature. 

A general principle for deriving loss functions would be to minimize the variance of the (log) importance weights under some sampling distribution $\samplingdist{}$, which leads to constant importance weights at optimality.   To draw connections with previous work, we also consider minimizing the square of the log weights, which at optimality, ensures that $\log w = 0$ and $w=1$ are equal to a \textit{particular} constant.
We will proceed to parameterize the twist functions using parameters $\thb$
and consider 
loss terms which
minimize
the variance or square of $c$-step log weights at time $t$,
\begin{align}
   \cL_{\log\text{Var}}^{(t,c)}(\thb) 
   \coloneqq
   \text{Var}_{\samplingdist{}}\bigg[ \sum \limits_{\tau=t}^{t+c-1} \log w_\tau(\bs_{1:\tau}) \bigg] \qquad \qquad   
   \cL_{\sqlogloss}^{(t,c)}(\thb) \coloneqq 
   \mathbb{E}_{\samplingdist{}}\bigg[ \bigg( \sum \limits_{\tau=t}^{t+c-1} \log w_\tau(\bs_{1:\tau}) \bigg)^2 ~ \bigg].
\end{align}
\normalsize
$\cL_{\sqlogloss}^{(t,c)}(\thb)$ indicates `consistency' in $\log$-weight space for $c$-step-ahead weights at time $t$ (see \cref{eq:c-step-weights}).

We will consider various choices of parameterization and proposal in the following subsections.  For example, 
let 
$\cL_{\sqlogloss}^{(t,c)}(\thb ; \{ \psi_t, q_t^{\pi} \})$ denote the log-consistency loss corresponding to twisting targets parameterized by $\vartwist{t}$ and the twist induced proposal $q_t^{\pi}$ (note, our notation for the one-step weights $w_t(\bs_{1:t})$ does not make these choices explicit).

For reference, we derive the log importance weights with intermediate 
\phishorts
and arbitrary $q$ as
\small
\begin{align}
\begin{split}
\log w_t(\bs_{1:t})= \log \frac{\tpitwist{t}(\bs_{1:t})}{\tpitwist{t-1}(\bs_{1:t-1})q(s_{t}|\bs_{1:t-1})} 
&= \log \frac{  p_0(\bs_{1:t}\condzero)  \left( \prod \limits_{\tau=1}^{t-1} \reward{\tau} \right) ~
\approxtwist{t}(\bs_{1:t})}{ p_0(\bs_{1:t-1}\condzero)  \left( \prod \limits_{\tau=1}^{t-2} \reward{\tau} \right) ~
\approxtwist{t-1}(\bs_{1:t-1}) q(s_{t}|\bs_{1:t-1})} \nonumber \end{split} 
\\[1.5ex]
\implies \quad \log w_t(\bs_{1:t}) &= \log \reward{t-1} +  \log \approxtwist{t}(\bs_{1:t}) - \log \approxtwist{t-1}(\bs_{1:t-1}) -\log \frac{q(s_{t}|\bs_{1:t-1})}{p_0(s_t|\bs_{1:t-1})}  \qquad 
\label{eq:log_w_soft_q}%
\end{align}
\normalsize
Various special cases arise from choices of twist parameterizations and proposals in the following subsections.  

\vheader
\subsubsection{Soft Q-Learning and RL Baseline}\label{app:soft_q}
\vheader
For single-step log-weights, the $\approxtwist{}$-\textit{parameterization} of the targets (\cref{eq:filtering_app_psi}, \cref{eq:filtering_app_q}), and the \textit{twist-induced proposal} (\cref{eq:prop_twist}, \cref{eq:twist_induced_rl_app}), we have
\begin{align}
 \log w_t(\bs_{1:t}) &= \medmath{
 \log \reward{t-1} +  \log \approxtwist{t}(\bs_{1:t}) - \log \approxtwist{t-1}(\bs_{1:t-1}) - \Big( \cancel{ \log \frac{\base(s_{t}|\bs_{1:t-1})}{\base(s_t|\bs_{1:t-1})}}  + \log \approxtwist{t}(\bs_{1:t}) - %
 \log \sum_{s_{t}} \base(s_t|\bs_{1:t-1}) \approxtwist{t}(\bs_{1:t})  
 \Big) } \nonumber \\
 &= \log \reward{t-1} +   \log \sum_{s_{t}} \base(s_t|\bs_{1:t-1}) \approxtwist{t}(\bs_{1:t}) - \log \approxtwist{t-1}(\bs_{1:t-1})
\end{align}
\normalsize
where the second term $\log \zonestep{t}{t-1} = \log \sum_{s_{t}} \base(s_t|\bs_{1:t-1}) \approxtwist{t}(\bs_{1:t}) $ normalizes the twist-induced proposal (\cref{eq:transition}).

Minimizing the sum of \textit{one-step log consistency losses} (i.e. squared log weights in \cref{eq:log_w_soft_q}) will yield the familiar soft $Q$-learning loss (e.g. \citet{lioutas2022critic} Eq. (4)-(5)). 
Adjusting indexing from \cref{eq:log_w_soft_q} and introducing
a stop-gradient within $\log \zonestep{t}{t-1}$,
we have
\begin{align} 
\min \limits_{\thb} \cL_{\textsc{softQ}}({\thb}) &\coloneqq 
\min \limits_{\thb} \sum \limits_{t=1}^{T} \cL_{\sqlogloss}^{(t+1,1)}(\thb; \{ \approxtwist{t}, q_t^\pi \} ) \label{eq:soft_q} \tag{Soft Q Learning}  \\
&= \min \limits_{\thb}
\sum \limits_{t=1}^{T} ~ \mathbb{E}_{\samplingdist(\cdot)}\Big[ \Big( \log \reward{t} +  \log \sum_{s_{t+1}} \base(s_{t+1}|\bs_{1:t}) \text{sg}\big(\vartwist{t+1}(\bs_{1:t+1}) \big) -  \log \vartwist{t}(\bs_{1:t})    \Big)^2 \Big] \nonumber \\
& \eqrl \min \limits_{\thb}
\sum \limits_{t=1}^{T} ~ \mathbb{E}_{\samplingdist(\cdot)}\Big[ \Big( \beta r_{t}(\bs_{1:t}) +  \log \sum_{s_{t+1}} \base(s_{t+1}|\bs_{1:t}) e^{\beta \texttt{sg}\big( Q_t^{\thb}(s_{t+1}, \bs_{1:t}) \big) } -  \beta Q_t^{\thb}(s_{t}, \bs_{1:t-1})  \Big)^2 \Big] \nonumber 
\end{align}
\normalsize
In the final line, we rewrite the loss for the soft RL special case, $\reward{t} = e^{\beta r_t(\bs_{1:t})}$ using the substitutions in \cref{eq:twist_to_soft_rl}.  
Note that the $\log$-normalization term is analogous to an induced soft value 
$
V_{Q^{\thb}_t}(\bs_{1:t-1})
= \frac{1}{\beta} \log \sum_{s_{t}} \base(s_{t}|\bs_{1:t-1}) \expof{\beta %
Q_t^{\thb}(s_t, \bs_{1:t-1})
}
$, so that each squared error loss has the form $\mathbb{E}[\beta^2 ( r_t + V_t - Q_t)^2]$.   Hence, we refer to this loss as \textit{Soft Q-learning} loss.

The $\log$-normalization term, which arises from normalizing the twist-induced proposal, is analogous to the `target' value 
in deep $Q$-learning. 
\citet{lioutas2022critic} consider the soft-Q learning loss to SMC sampling in self-driving applications where interaction with the environment is expensive.  
\citet{lawson2018twisted} adopt a similar loss function (using a parameterization of the value $V_t^{\thb}$) in the setting of state-space models with tractable intermediate rewards.

\vheader
\paragraph{RL Baseline with no Intermediate Reward}
The soft Q-learning loss in \cref{eq:soft_q} simplifies nicely in the case of no intermediate rewards, as in the main text ($\reward{t} = 1$ for $t < T$ and $\vartwistphi{T}=1$).

Written in terms of twist functions,
we separate the terms at $t < T$ and $t = T$ for purposes of exposition
\begin{align}
&  \min \limits_{\thb} \cL_{\textsc{rl}}({\thb})   \coloneqq 
\min \limits_{\thb} \sum \limits_{t=1}^{T} \cL_{\sqlogloss}^{(t+1,1)}(\thb; \{ \approxtwist{t}, q_t^\pi, \phi_t=1 \} )   \label{eq:rl_baseline} \tag{RL Baseline}  \\
&=  \min \limits_{\thb}
\sum \limits_{t=1}^{T-1} ~ \mathbb{E}_{\samplingdist(\cdot)}\Big[ \Big( \log \sum_{s_{t+1}} \base(s_{t+1}|\bs_{1:t}) \text{sg}\big(\vartwist{t+1}(\bs_{1:t+1}) \big) -  \log \vartwist{t}(\bs_{1:t})    \Big)^2 \Big]
+ \mathbb{E}_{\samplingdist(\cdot)}\Big[ \Big( \log \finaltwist -  \log \vartwist{T}(\bs_{1:T})    \Big)^2 \Big]
\nonumber
\end{align}
 For intermediate timesteps, note that \cref{eq:rl_baseline} enforces the recursion $\vartwist{t-1}(\bs_{1:t-1}) = \sum_{s_t} \base(s_{t}|\bs_{1:t-1}) \vartwist{t}(\bs_{1:t})$ in \cref{eq:psi_recursion} of the main text, albeit in log space.   
 In \cref{app:mudgal_loss} below, we consider the one-step squared error loss enforcing this recursion directly (without logarithms), i.e. $\mathbb{E}_{\samplingdist}[ ( \vartwist{t-1}(\bs_{1:t-1}) - \sum_{s_t} \base(s_{t}|\bs_{1:t-1}) \vartwist{t}(\bs_{1:t}))^2]$ ,

\vheader
\subsubsection{Path Consistency Learning (for Twist Learning)}\label{app:pcl_twist_learning}
\vheader
Using the \textit{value}  \textit{parameterization} of the targets ($\vartwistphi{t}$ or $V_t$, see \cref{eq:filtering_app_phi}, \cref{eq:filtering_app_v}), the one-step log consistency loss with arbitrary proposal $q$ recovers the path-consistency loss (PCL) from \citet{nachum2017bridging}.

Switching to a $\vartwistphitheta{t}$ parameterization of the twisting targets, we substitute $\vartwist{t}(\bs_{1:t}) = \reward{t} \vartwistphitheta{t}(\bs_{1:t})$ into the log importance weights in 
\cref{eq:log_w_soft_q}.   The log-consistency loss becomes,
\small
\begin{align}
\hspace*{-.2cm} \min \limits_{\thb} \cL_{\textsc{pcl}}(\thb) \coloneqq &\min \limits_{\thb} \sum \limits_{t=1}^{T} \cL_{\sqlogloss}^{(t,1)}(\thb; \{ \vartwistphi{t} , \text{any } q \} ) \label{eq:pcl} \tag{PCL}  \\
= & \min \limits_{\thb} \sum \limits_{t=1}^{T} \mathbb{E}_{\samplingdist{}}\left[ \left( 
  \log \reward{t} + \log \vartwistphitheta{t}(\bs_{1:t}) - \log \vartwistphitheta{t-1}(\bs_{1:t-1}) 
- \log \frac{q(s_{t}|\bs_{1:t-1})}{ \base(s_{t}|\bs_{1:t-1})} 
\right)^2 \right]  \nonumber \\
\eqrl & \min \limits_{\thb} \sum \limits_{t=1}^{T} \mathbb{E}_{\samplingdist{}}\left[ \left( 
\beta \Big(  r_t(\bs_{1:t}) + V_{t}^{\thb}(\bs_{1:t}) - V_{t-1}^{\thb}(\bs_{1:t-1}) \Big)
- \log \frac{q(s_{t}|\bs_{1:t-1})}{ \base(s_{t}|\bs_{1:t-1})} 
\right)^2 \right]  \nonumber 
\end{align}
\normalsize
In particular, substituting the soft RL 
\phishort
terms from \cref{eq:twist_to_soft_rl}, \cref{eq:pcl} recovers the path consistency loss from \citet{nachum2017bridging}.
Note that we derived PCL from an importance sampling perspective, whereas PCL was originally derived by enforcing KKT conditions of the soft RL problem. 

We might also consider multi-step losses for various $c$.
Minimizing the square of the multi-step log weights 
with {arbitrary $q$}
recovers the multi-step PCL loss \citep{nachum2017bridging},
\small 
\begin{align}
\min \limits_{\thb} \cL_{\textsc{pcl}}^{(t,c)} (\thb) \coloneqq & \min \limits_{\thb}  \cL_{\sqlogloss}^{(t,c)}(\thb; \{ \vartwistphi{t} , \text{any } q \} ) \label{eq:pcl_multi}\tag{multi-step PCL} \\
= & \min \limits_{\thb} \mathbb{E}_{\samplingdist{}}\left[ \left(  \sum_{\tau=t}^{t+c} \log \reward{\tau} + \log \vartwistphitheta{t+c}(\bs_{1:t+c}) - \log \vartwistphitheta{t-1}(\bs_{1:t-1}) 
- \sum_{\tau=t}^{t+c} \log \frac{q(s_{\tau}|\bs_{1:\tau-1})}{ \base(s_{\tau}|\bs_{1:\tau-1})}
\right)^2 \right] \nonumber\\
= & \min \limits_{\thb} \mathbb{E}_{\samplingdist{}}\left[ \left(  \sum_{\tau=t-1}^{t+c-1} \log \reward{\tau} + \log \vartwist{t+c}(\bs_{1:t+c}) - \log \vartwist{t-1}(\bs_{1:t-1}) 
- \sum_{\tau=t}^{t+c} \log \frac{q(s_{\tau}|\bs_{1:\tau-1})}{ \base(s_{\tau}|\bs_{1:\tau-1})}
\right)^2 \right] \label{eq:pcl_multi_psi} \\
\eqrl & \min \limits_{\thb} \mathbb{E}_{\samplingdist{}}\left[ \left( \beta  \sum_{\tau=t}^{t+c} r_\tau(\bs_{1:\tau}) + \beta ~ V_{t+c}^{\thb}(\bs_{1:t+c}) - \beta ~ V_{t-1}^{\thb}(\bs_{1:t-1}) 
- \sum_{\tau=t}^{t+c} \log \frac{q(s_{\tau}|\bs_{1:\tau-1})}{ \base(s_{\tau}|\bs_{1:\tau-1})}
\right)^2 \right]  \nonumber %
\end{align}
\normalsize
where we write the $\vartwist{t}$ parameterization in \cref{eq:pcl_multi_psi} explicitly for use in \cref{app:mudgal_decoding}.
While PCL considers learned a proposal or policy $q$ with the goal of approximating the solution of a regularized MDP,  we leave joint learning of proposals $\{q^{\qparam}(s_t|\bs_{1:t-1})\}_{t=1}^T$ and SMC target twists $\{\vartwist{t}(\bs_{1:t})\}_{t=1}^T$ or $\{V_t^{\thb}(\bs_{1:t})\}_{t=1}^T$ to future work.

In \cref{app:prop}, we describe using PCL to learn the proposal \textit{only} \citep{guo2021efficient}, with the values $V_{Q_t}(\bs_{1:t})
$
induced from learned proposal twists $Q_t^{\qparam}(s_{t+1},\bs_{1:t})$
which define $\{q_{Q_t}^{\qparam}(s_{t+1}|\bs_{1:t})\}_{t=0}^{T-1}$ (in similar fashion to \cref{eq:twist_induced_rl_app}, but without reference to twisting targets).

\vheader
\subsection{Controlled Decoding Losses via Optimal Twist Identities \citep{mudgal2023controlled}}\label{app:mudgal_loss}
\vheader

In \cref{prop:opt_twists_int_reward} (or \cref{prop:optimal_twists} and \cref{eq:psi_recursion} in the main text), we noted that the optimal twists satisfy the following relationships
\begin{align}
 \optimaltwist{t}(\bs_{1:t}) = & c_t ~ \reward{t} \sum \limits_{\bs_{t+1:T}} \base(\bs_{t+1:T}|\bs_{1:t})  \prod \limits_{\tau=t+1}^{T} \reward{\tau}
 \phantom{\overset{(\phi_t=1)}{=}} =  \frac{c_t}{c_{t+1}} \reward{t} \sum \limits_{s_{t+1}} \base(s_{t+1}|\bs_{1:t})  
 \optimaltwist{t+1}(\bs_{1:t+1}) \nonumber \\
 \overset{(\phi_t=1)}{=} & c_t \sum \limits_{\bs_{t+1:T}} \base(\bs_{t+1:T}|\bs_{1:t})  \finaltwist \hphantom{\prod \limits_{\tau=t+1}= ~ \reward{t} \phi_t} \overset{(\phi_t=1)}{=}  \frac{c_t}{c_{t+1}} \sum \limits_{s_{t+1}} \base(s_{t+1}|\bs_{1:t}) 
 \optimaltwist{t+1}(\bs_{1:t+1})  
 \label{eq:app_recursions} 
\end{align}
We proceed to describe two `controlled decoding' (CD) losses from \citet{mudgal2023controlled} 
as using a \textit{squared error} loss to
enforce the optimality conditions in \cref{eq:app_recursions}, for settings with no intermediate \phishorts ($\reward{t} = 1$ for $t < T$).   \citet{mudgal2023controlled} also propose two ways to use the learned `twists' at inference time, which we discuss in relation to our proposed SMC framework in \cref{app:mudgal_decoding}.   %

\vheader
\paragraph{CD-Q} 
The CD-Q loss from \citet{mudgal2023controlled} corresponds to minimizing the one-step recursion in \cref{eq:app_recursions} using the expected squared error under a (possibly off-policy) sampling distribution $\samplingdist{}$.   
Assuming \textit{no intermediate reward} and an additional squared error loss to approximate the terminal \phishort $ \vartwist{T}(\bs_{1:T}) \approx \finaltwist $, we have
\begin{align}
\begin{split}
\hspace*{-.2cm} \min \limits_{\thb} \cL_{\textsc{cd-q}}(\thb) \coloneqq \min \limits_{\thb} \sum \limits_{t=1}^{T-1} \mathbb{E}_{\samplingdist(\cdot)}\Big[ \Big( \sum_{s_{t+1}} \base(s_{t+1}|\bs_{1:t}) \vartwist{t+1}(\bs_{1:t+1}) -  \vartwist{t}(\bs_{1:t}) \Big)^2 \Big]  
+  \mathbb{E}_{\samplingdist{(\cdot)}}\left[\left( \finaltwist -   \vartwist{T}(\bs_{1:T}) \right)^2 \right]  
    \end{split} \label{eq:cd_q} \tag{CD-Q}
\end{align}
\cref{eq:cd_q} enforces the same optimality condition as the \cref{eq:rl_baseline} loss (i.e. $\vartwist{t}(\bs_{1:t}) =\sum_{s_{t+1}} \base(s_{t+1}|\bs_{1:t}) \vartwist{t+1}(\bs_{1:t+1})$), without log scaling of each term inside the squared error.   
At optimality, we have zero-variance one-step importance weights ($w(\bs_{1:t}) = 1$ in \cref{eq:incremental_twisted}) for the {twist-induced proposal}.

At optimality, we in fact also have $\vartwist{t}(\bs_{1:t}) = \sum \limits_{\bs_{t+1:T}} \base(\bs_{t+1:T}|\bs_{1:t}) \reward{T}$ (as in \cref{eq:app_recursions} and the proof of \cref{prop:opt_twists_int_reward}).

\vheader
\paragraph{CD-FUDGE}
While we might naively like to consider the loss
$\mathbb{E}_{\samplingdist(\cdot)}\big[ \big(\vartwist{t}(\bs_{1:t}) -  \sum_{\bs_{t+1:T}} \base(\bs_{t+1:T}|\bs_{1:t})\finaltwist \big)^2 \big]$ to enforce \cref{prop:optimal_twists} or \cref{eq:app_recursions}, note that marginalization over multiple steps is not tractable in general.

Instead, the CD-FUDGE loss\footnote{Note, we reserve the naming convention FUDGE \citep{yang2021fudge} for a binary cross entropy loss described in \cref{app:fudge}, as opposed to the CD-FUDGE squared error loss from \citet{mudgal2023controlled}.} defined as 
\begin{align}
    \min \limits_{\thb} \cL_{\textsc{cd-fudge}}(\thb) \coloneqq \min \limits_{\thb} \sum \limits_{t=1}^{T} \mathbb{E}_{\samplingdist(\bs_{1:t})}\left[ \mathbb{E}_{\base(\bs_{t+1:T}|\bs_{1:t})} \left[ \Big(    \vartwist{t}(\bs_{1:t})  - \finaltwist \Big)^2 \right] \right] \label{eq:cd_fudge} \tag{CD-FUDGE}
\end{align} 
can be shown to have the same gradient as the desired (but intractable) squared error loss above \citep{mudgal2023controlled}.

Since the minimizer of the expected squared error (under $\base(\bs_{t+1:T}|\bs_{1:t})$) to a single function $\vartwist{t}(\bs_{1:t})$ (which is independent of $\bs_{t+1:T}$) is given by the conditional expectation \citep{banerjee2005optimality}, we can also see that \cref{eq:cd_fudge} has the desired minimum $\vartwist{t}(\bs_{1:t}) = \sum_{\bs_{t+1:T}} \base(\bs_{t+1:T}|\bs_{1:t})\finaltwist $. 
Note, it is crucial that the inner expectation samples rollouts under the base model $\base(\bs_{t+1:T}|\bs_{1:t})$ to obtain the desired conditional expectation as the minimizer.   While it appears that any prefix sampling distribution can be used, $\samplingdist = \base$ allows for losses to be calculated at all $t$ in a single sampling run.

\citet{mudgal2023controlled} also propose two decoding-time usages for the learned twist functions $\vartwist{t}$:  stochastic token-by-token sampling and argmax decoding of partial sequences.   {We discuss their inconsistencies with our SMC framework in \cref{app:mudgal_decoding_all}.}

\paragraph{CD-FUDGE for $\log \vartwist{t}$} 
We can also compare \cref{eq:cd_fudge} with the multi-step PCL loss in \cref{eq:pcl_multi_psi}, choosing $\phi_t=1$ for $t < T$ and the proposal equal to the base model $q = \base$ so that the proposal terms cancel.    Noting that $\approxtwist{T}(\bs_{1:T}) = \finaltwist$ is fixed to the exact terminal \phishort and choosing the $c= T-t+1$-step PCL loss for each $t$, note that \cref{eq:pcl_multi_psi} would reduce to $\sum_t \mathbb{E}[ \left( \log \finaltwist  +  0  - \log \vartwist{t}(\bs_{1:t}) - 0 \right)^2]$.
\citet{deng2023reward} optimize this loss with reweighting of terms based on timestep (higher weight for $t \approx T$).
\cref{eq:cd_fudge} optimizes the squared error of the difference \textit{without log scaling of each term}, under appropriate sampling of rollouts.  \footnote{Note the difference in choice of proposal between \cref{eq:cd_q} (twist-induced $q= q_t^\pi$) and \cref{eq:cd_fudge} (base $q=\base$).}

\vheader
\subsection{SIXO:  Smoothing Inference with Twisted Objectives \citep{lawson2022sixo}}
\label{app:sixo}
\vheader

\citet{lawson2022sixo} adopt a noise-contrastive estimation loss \citep{gutmann2010noise} to learn the target twist functions using binary classification.   For state space models, \citet{lawson2022sixo} 
adopt our setting in \cref{app:conditional_twist_expo}
with observation variables $o_t$ emitted based on the sampling state $\bs_{1:t}$ (or simply $s_t$) and a known likelihood $\phi_t(o_t, s_t) = \sigma(o_t|s_{t})$. %
As discussed in \cref{app:final_timestep}, in these settings with easily evaluable intermediate \phishortsnospace, it may be preferable to parameterize $\vartwistphitheta{t}(\bs_{1:t}, \boldsymbol{o}_{t+1:T})$ as in \cref{eq:filtering_app_phi}.   
\citet{lawson2022sixo} indeed use this parameterization (see their Eq. 5).

\newcommand{\clf}{p}

Recall from \cref{eq:filtering_intermediate_optimality_conditional} that the optimal twists or future values amount to conditional likelihoods,
\begin{align}
\optimaltwistphi{t}(\bs_{1:t}, \bo_{t+1:T}) \proptoo{\bo_{t+1:T}} \sigma(\bo_{t+1:T}|\bs_{1:t}\condzero)\,, \qquad \qquad  \optimaltwist{t}(\bs_{1:t},\bo_{t:T}) \proptoo{\bo_{t:T}} \sigma(\bo_{t:T}|\bs_{1:t}\condzero)\,, \label{eq:sixo_twist_lkd}
\end{align}
where $\proptoo{\bo}$ denotes proportionality
up to a constant which depends on $\bo$ only. Using Bayes rule, we have
\begin{align}
\sigma(\bo_{t+1:T}|\bs_{1:t}) =  \frac{\sigma(\bs_{1:t}|\bo_{t+1:T})\sigma(\bo_{t+1:T})}{\base(\bs_{1:t})}   \proptoo{\bo_{t+1:T}} \frac{\sigma(\bs_{1:t}|\bo_{t+1:T})}{\base(\bs_{1:t})} \,, \qquad
\sigma(\bo_{t:T}|\bs_{1:t}) \proptoo{\bo_{t:T}} \frac{\sigma(\bs_{1:t}|\bo_{t:T})}{\base(\bs_{1:t})} \,,
\end{align}
noting that $\sigma(\bo_{t+1:T})$ and $\base(\bs_{1:t})$ are marginals of $\sigma(\bs_{1:t}, \bo_{t+1:T})$ by definition.
The above reasoning suggests that we may learn the twists, or likelihood ratio $\optimaltwistphi{t}(\bs_{1:t}, \bo_{t+1:T}) \propto \sigmacond{t}(\bo_{t+1:T}|\bs_{1:t}) \propto \sigma(\bs_{1:t}|\bo_{t+1:T})/\base(\bs_{1:t})$, using a 
classifier which seeks to distinguish samples from 
$\sigma(\bs_{1:t}|\bo_{t+1:T})$ and $\base(\bs_{1:t})$ \citep{gutmann2010noise, lawson2022sixo}.   In particular, at each $t$, we classify the event $y=1$, indicating that $\bs_{1:t} \sim \sigma(\bs_{1:t}|\bo_{t+1:T})$, or $y=0$, indicating that $\bs_{1:t} \sim \base(\bs_{1:t})$.  

Consider a given $\bo_{t+1:T}$, which can be either $\bo_{t+1:T} = \bone$ in the unconditional case or $ \bo_{t+1:T} \sim \samplingdist(\bo_{t+1:T})$ drawn from a behavioral policy as discussed below.   
The \textsc{SIXO} loss becomes
\small 
\begin{align}
\begin{split}
\cL_{\text{SIXO}}( 
\bo_{1:T};
\thb) 
&= %
\sum \limits_{t=1}^{T-1} \mathbb{E}_{\sigmacond{t}(\bs_{1:t}|\bo_{t+1:T} \condzero)}\left[ \log \texttt{sigmoid}\big( \log \vartwistphitheta{t}(\bs_{1:t}, \bo_{t+1:T})\big)\right] + \mathbb{E}_{\base(\bs_{1:t}\condzero)}\left[ \log \big(1- \texttt{sigmoid}\big(\log \vartwistphitheta{t}(\bs_{1:t}, \bo_{t+1:T})\big)\big)\right] 
\\
&=  %
\sum \limits_{t=1}^{T} \mathbb{E}_{\sigmacond{t}(\bs_{1:t}|\bo_{t:T} \condzero)}\left[ \log \texttt{sigmoid}\big( \log \vartwist{t}(\bs_{1:t},\bo_{t:T})\big)\right] + \mathbb{E}_{\base(\bs_{1:t}\condzero)}\left[ \log \big(1- \texttt{sigmoid}\big(\log \vartwist{t}(\bs_{1:t},\bo_{t:T})\big)\big)\right]
\\
&=  %
\sum \limits_{t=1}^T \mathbb{E}_{\sigmacond{t}(\bs_{1:t}|\bo_{t:T} \condzero)}\left[ \log
\frac{\vartwist{t}(\bs_{1:t},\bo_{t:T})}{1+\vartwist{t}(\bs_{1:t},\bo_{t:T})}
\right] + \mathbb{E}_{\base(\bs_{1:t}\condzero)}\left[ \log \frac{1}{1+\vartwist{t}(\bs_{1:t},\bo_{t:T})}\right] 
\end{split} 
\label{eq:sixo_nce} \tag{SIXO}\raisetag{20pt}
\end{align}
\normalsize
Note that we can perform approximate positive sampling as in \cref{sec:learning} to estimate expectations in the first term.  

\vheader
\paragraph{Exact Conditional Sampling}
However, we can also use the BDMC trick in \cref{sec:conditional_twist} to obtain exact target samples for general observation variables.  In order to facilitate tractable sampling, we optimize the \cref{eq:sixo_nce} loss over a sampling distribution 
$\samplingdist(\bo_{1:T})  = \sigma(\bo_{1:T})$
for all $t$, such that the objective becomes
\begin{align}
\mathbb{E}_{\sigma(\bo_{1:T})}\left[ \cL_{\text{SIXO}}( \bo_{1:T} ; 
\thb) \right]  &=  %
\sum \limits_{t=1}^T \mathbb{E}_{\sigmacond{t}(\bs_{1:t}, \bo_{t+1:T} \condzero)}\left[ \log
\frac{\vartwist{t}(\bs_{1:t},\bo_{t:T})}{1+\vartwist{t}(\bs_{1:t},\bo_{t:T})}
\right] + \mathbb{E}_{\base(\bs_{1:t}\condzero)\sigma(\bo_{t+1:T})}\left[ \log \frac{1}{1+\vartwist{t}(\bs_{1:t},\bo_{t:T})}\right] \nonumber 
\end{align}
With this choice, note that we may sample once from $\sigma(\bs_{1:T}, \bo_{1:T})= \prod_{t=1}^T \base(s_t|\bs_{1:t-1}) \sigma(o_t|\bs_{1:t})$ using ancestral sampling and use the appropriate truncations for positive sampling terms involving $\sigmacond{t}(\bs_{1:t}, \bo_{t+1:T} \condzero)$.
By shuffling observation variables across a batch of $K$ samples, we may obtain samples from the product of marginals $\base(\bs_{1:T})\sigma(\bo_{1:T})$ or $\base(\bs_{1:t})\sigma(\bo_{t+1:T})$ in the negative sampling term.

In the main text, note that we condition on $o_T = 1$ or $o_{T} = \bs_{T+1:T+c}$ (for infilling).
\vheader
\paragraph{Gradient and Comparison with CTL} Proceeding with the $\vartwist{t}$ parameterization for the target $\target(\bs_{1:T}|o_T\condzero)= \sigma(\bs_{1:T})$ with fixed $o_T$ and unconditional twists $\vartwist{t}(\bs_{1:t})$, the gradient of \cref{eq:sixo_nce} with respect to $\thb$ is
\begin{align}
\nabla_{\thb} \cL_{\text{SIXO}}(\thb) &= \sum \limits_{t=1}^T \mathbb{E}_{\target(\bs_{1:t}\condzero)}\left[ \nabla_{\thb} \log \vartwist{t}(\bs_{1:t}) - \frac{\vartwist{t}(\bs_{1:t})}{1+\vartwist{t}(\bs_{1:t})} \nabla_{\thb} \log \vartwist{t}(\bs_{1:t})
\right] - \mathbb{E}_{\base(\bs_{1:t}\condzero)}\left[ \frac{\vartwist{t}(\bs_{1:t})}{1+\vartwist{t}(\bs_{1:t})} \nabla_{\thb} \log \vartwist{t}(\bs_{1:t})\right] \nonumber \\
&= \sum \limits_{t=1}^T \mathbb{E}_{\target(\bs_{1:t}\condzero)}\left[ \frac{1}{1+\vartwist{t}(\bs_{1:t})}\nabla_{\thb} \log \vartwist{t}(\bs_{1:t})
\right] - \mathbb{E}_{\base(\bs_{1:t}\condzero)}\left[ \frac{\vartwist{t}(\bs_{1:t})}{1+\vartwist{t}(\bs_{1:t})} \nabla_{\thb} \log \vartwist{t}(\bs_{1:t})\right]  
\tag{SIXO Gradient}
\end{align}
The SIXO gradient is superficially similar to our CTL gradient in \cref{sec:ebm_roger}, in that it involves $\nabla_{\thb} \log \vartwist{t}$ under positive and negatives samples. 
However, viewing $\tpitwisttheta{t}(\bs_{1:t}\condzero) = \base(\bs_{1:t}\condzero) \vartwist{t}(\bs_{1:t})$ as the unnormalized density of our intermediate twisting target, we can see that the second term in the \textsc{sixo} update includes $\tpitwisttheta{t}(\bs_{1:t}\condzero)$.   
Rewriting to highlight differences with our CTL gradient, we have
\begin{align}
\hspace*{-.2cm} 
\begin{split}
\nabla_\thb \cL_{\text{SIXO}}  &= \sum \limits_{t=1}^T \left( \sum_{\bs_{1:t}} \target(\bs_{1:t}\condzero)  \hl{ \frac{1}{1+\vartwist{t}(\bs_{1:t})} } \nabla_{\thb} \log \vartwist{t}(\bs_{1:t})
 - \sum_{\bs_{1:t}} \tpitwisttheta{t}(\bs_{1:t}) \hl{\frac{1}{1+\vartwist{t}(\bs_{1:t})}} \nabla_{\thb} \log \vartwist{t}(\bs_{1:t}) \right) \\
\nabla_\thb \cL_{\text{CTL}}  &= \sum \limits_{t=1}^T \left( \sum_{\bs_{1:t}} \target(\bs_{1:t}\condzero)  \phantom{ \frac{1}{1+\vartwist{t}(\bs_{1:t})} } \nabla_{\thb} \log \vartwist{t}(\bs_{1:t})
- \sum_{\bs_{1:t}} \tpitwisttheta{t}(\bs_{1:t}) ~~~\quad \hl{\frac{1}{\zfilter{t}}} \quad~~~   \nabla_{\thb} \log \vartwist{t}(\bs_{1:t}) \right) 
 \end{split}\tag{SIXO vs. CTL}
\end{align}
To compare the two, first note that the positive sampling gradient in SIXO is scaled by a factor of $\frac{1}{1+\vartwist{t}(\bs_{1:t})}$ factor (which reflects the misclassification probability under $\vartwist{t}$).   For the negative sampling terms, note that $\tpitwisttheta{t}(\bs_{1:t}\condzero)$ is divided by 
 a factor of
$\frac{1}{1+\vartwist{t}(\bs_{1:t})}$ in the SIXO gradient,
 instead of the true normalization constant $\zfilter{t}$ for the gradient of our CTL loss \cref{eq:grad_separate_ebm}.

\vheader
\subsection{FUDGE: Future Discriminators \citep{yang2021fudge}}\label{app:fudge}
\vheader
In contrast to \textsc{SIXO}, the FUDGE method from \citet{yang2021fudge} seeks to directly learn a discriminative classifier to match the conditional likelihood $\optimaltwist{t}(\bs_{1:t}, o_T) \propto \sigmacond{t}(o_T | \bs_{1:t})$ or $\optimaltwist{t}(\bs_{1:t}, \bo_{t:T}) \propto \sigmacond{t}(\bo_{t:T} | \bs_{1:t})$ (see \cref{app:conditional_twist_expo}).

As before, we define the joint distribution $\target(\bs_{1:T}, o_T) = \base(\bs_{1:T}\condzero) \sigmacond{t}(o_T|\bs_{1:T})$ with $\phi(\s_{1:T}, o_T) = \sigmacond{t}(o_T|\bs_{1:T})$.   
From \cref{eq:sixo_twist_lkd} above or \cref{app:conditional_twist_expo} \cref{eq:cond_twist_lkd}, we have
\begin{align}
\optimaltwist{t}(\bs_{1:t}, o_T) \propto
\sigma(o_T|\bs_{1:t}) \coloneqq \sum_{\bs_{t+1:T}} \base(\bs_{t+1:T}|\bs_{1:t}) \sigmacond{T}(o_T|\bs_{1:T})
\label{eq:future_discriminator}
\end{align}
\citet{yang2021fudge} consider training a `future discriminator' 
$
\vartwist{t}(\bs_{1:t}, o_T)
 \approx 
 \sigma(o_T|\bs_{1:t})$ 
 which, as in \cref{eq:future_discriminator} marginalizes over future tokens to predict the expected probability that a full sequence with prefix $\bs_{1:t}$ emits $o_T$ (e.g., let $o_T=a$ be the probability of a classifier for class $a$, or the probability that $\bs_{1:T}$ satisfies a desired attribute indicated by a boolean $o_T=1$).

In similar fashion to \textsc{SIXO} in the previous section, we define a binary random variable $y$ such that 
\small 
\begin{align}
    \sigma(y | \bs_{1:t}, o_T) = \begin{cases} \sigma(o_T |\bs_{1:t}) \phantom{1 - }~~ \qquad y=1 \\
    1-\sigma(o_T |\bs_{1:t}) \qquad y=0 
    \end{cases} \qquad \qquad 
    p_{\vartwist{t}}(y |\bs_{1:t}, o_T) = \begin{cases} \vartwist{t}(\bs_{1:t}, o_T) \phantom{1 - }~~ \qquad y=1 \\
    1-\vartwist{t}(\bs_{1:t}, o_T) \qquad y=0 
    \end{cases}\label{eq:clf_def}
\end{align}  
\normalsize
where we directly parameterize $p_{\vartwist{t}}(y |\bs_{1:t}, o_T) = \vartwist{t}(\bs_{1:t}, o_T)$ to be a probability distribution 
 (e.g. using a sigmoid or softmax activation).
 For a given observation random variable $o_T$ and partial sequence $\bs_{1:t}$, we can define the \textsc{FUDGE} loss
\small
\begin{align}
\sum \limits_{t=1}^T \cL_{\text{FUDGE}}(\bs_{1:t}, o_T;\thb) &\coloneqq \sum \limits_{t=1}^T %
\DKL\pV{
\sigmacond{t}(y|\bs_{1:t}, o_T)
}{ p_{\vartwist{t}}(y |\bs_{1:t}, o_T)
} 
\label{eq:fudge} \tag{FUDGE} \\
&=  \sum \limits_{t=1}^T 
 -\left[ 
 \sigmacond{T}(y=1 |\bs_{1:t}, o_T) \log 
 p_{\vartwist{t}}(y=1 |\bs_{1:t}, o_T)
 + \sigmacond{T}(y=0|\bs_{1:t}, o_T) \log 
 p_{\vartwist{t}}(y=0 |\bs_{1:t}, o_T)  
\right]+ \text{const} . \nonumber
\\ 
&=  \sum \limits_{t=1}^T 
 - %
\mathbb{E}_{\base(\bs_{t+1:T}|\bs_{1:t})
 } \bigg[ \sigmacond{T}(o_T |\bs_{1:T}) \log 
 \vartwist{t}(\bs_{1:t},o_T)
+
\Big(1-\sigmacond{T}(o_T|\bs_{1:T}) \Big)
\log \Big(1- \vartwist{t}(\bs_{1:t}, o_T)\Big)
\Big)
\bigg]  + \text{const} . \nonumber
\end{align}
\normalsize
where, in moving from the second to the third line, we have used the fact that $
\sigmacond{T}(y=1|\bs_{1:t}, o_T)= \sigma(o_T|\bs_{1:t}) 
= \sum_{\bs_{t+1:T}} \base(\bs_{t+1:T}|\bs_{1:t}) \sigma(o_T|\bs_{1:T})  $ from \cref{eq:future_discriminator} and \cref{eq:clf_def}.
At the optimum,  $ p_{\vartwist{t}}(y=1 |\bs_{1:t}, o_T) = \sigma(y = 1 | \bs_{1:t},o_T)$ implies $\vartwist{t}(\bs_{1:t}, o_T) = \sigma(o_T |\bs_{1:t})$, as desired.

While sampling may be done using an arbitrary distribution over prefixes $\bs_{1:t}$ and observation $o_T$,  \cref{eq:fudge} \textit{requires} that rollouts be sampled under the base model $\base(\bs_{t+1:T}|\bs_{1:t})$ 
in order to ensure sampling from the appropriate distribution 
$\sigma(y = 1 | \bs_{1:t},o_T)$. 
This restriction is similar to what we required in \cref{eq:cd_fudge}, although the loss in \cref{eq:fudge} is based on cross entropy classification rather than a squared error.   We discuss the choices made in our experiments below.

\vheader
\paragraph{\citet{yang2021fudge} Setting}  In the original \textsc{FUDGE} paper, \citet{yang2021fudge} consider learning from a dataset of labelled examples $(\bs_{1:T}, o_T)$ or $(\bs_{1:t}, o_T)$ for a binary observation variable $o_T = 1$ which defines the target distribution.

\vheader
\paragraph{Unconditional Twist Setting}  For the unconditional twist experiments in \cref{sec:toxclass}-\ref{sec:sent}, we sample under the base model proposal $\samplingdist(\bs_{1:t}) = \base(\bs_{1:t})$  where the target distribution conditions on $o_T = 1$ and $\sigma(o_T = 1 | \bs_{1:T}) = \finaltwist = \sigma(y=1|\bs_{1:T}, o_T=1)$.  In particular, we optimize %
\begin{align*}
\min \limits_{\theta} \sum \limits_{t=1}^T \mathbb{E}_{\base(\bs_{1:t})}\left[\cL_{\text{FUDGE}}(\bs_{1:t}, o_T = 1;\thb)\right]
\end{align*}
\normalsize
\vheader
\paragraph{Conditional Twist Setting} For conditional twist learning, we can consider amortizing learning the twists $\psi_t(\bs_{1:t},o_T)$ over some distribution of observation variables $\pi_{s}\left(\s_{1:t},o_{T}\right)$.
In particular, in our infilling experiments in \cref{sec:infilling}, we
consider sampling under the following joint distribution,
\begin{align}
\pi_{s}\left(\s_{1:t},o_{T}\right)=p_{0}\left(\s_{1:t}\right)\si\pv{o_{T}}{\s_{1:t}}\,,
\nonumber 
\end{align}
which we can sample from by first sampling from $p_{0}\left(\s_{1:T}\right)\si\pv{o_{T}}{\s_{1:T}}$ and then dropping $\s_{t+1:T}$ subsequence. 
Therefore, the overall objective becomes
\begin{align}
\min \limits_{\thb} & ~ \mathbb{E}_{\samplingdist(\bs_{1:t}, o_T)}\left[\cL_{\text{FUDGE}}(\bs_{1:t}, o_T ; \thb)\right]%
\\ &
= \min \limits_{\thb} \sum \limits_{t=1}^T 
 - %
\mathbb{E}_{p_{0}\left(\s_{1:T}\right)\si\pv{o_{T}}{\s_{1:t}}
 } \bigg[ \sigmacond{T}(o_T |\bs_{1:T}) \log 
 \vartwist{t}(\bs_{1:t},o_T)
+
\Big(1-\sigmacond{T}(o_T|\bs_{1:T}) \Big)
\log \Big(1- \vartwist{t}(\bs_{1:t}, o_T)\Big)
\Big)
\bigg]\,, \nonumber
\end{align}
where the expectation $\base(\bs_{1:T})$ includes the expectation under $\base(\bs_{t+1:T}|\bs_{1:t})$ from \cref{eq:fudge}.   Note that rollout of $\bs_{t+1:T}|\bs_{1:t}$ used to sample from $\base(\bs_{1:T})$ should be independent of the rollout used to sample from $\sigma(o_T|\bs_{1:t})$.

\vheader\section{Decoding Strategies using Learned Twists from \citet{mudgal2023controlled}}\label{app:mudgal_decoding_all}
\vheader
\subsection{Proposal Sampling in \citet{mudgal2023controlled}}\label{app:mudgal_decoding}
\vheader

\newcommand{\vartwistcd}[1]{\psi_{#1}^{\text{cd}\thb}}

As noted in \cref{app:mudgal_loss} (and in $\cL^*(\thb)$ in \citet{mudgal2023controlled}), the CD losses can be seen as enforcing the optimality conditions
\begin{align}
\psi^{\text{cd}*}_{t}
(\bs_{1:t}) = \sum_{\bs_{t+1:T}} \base(\bs_{t+1:T}|\bs_{1:t}) \finaltwist ,\qquad \quad \forall t . \label{eq:cd_opt}
\end{align}
In RL terms, we interpret the twists $\psi^{\text{cd}*}_{t}$ %
as performing \textit{policy evaluation} of the expected \textit{unregularized} `reward' $\finaltwist$ under a fixed policy $\base(\bs_{1:T})$.  
{The notation of \citet{mudgal2023controlled} (their Eq. (1), (5), our \cref{eq:cd_opt}) indeed corresponds to }
\begin{align}
    \finaltwist \eqqcolon r_{\text{cd}}(\bs_{1:T}).  \label{eq:cd_reward} \tag{CD reward}
\end{align}

However, \citet{mudgal2023controlled} propose to use the learned twist functions $\vartwist{t}$ to perform one-step sampling 
as
\begin{align}
    q_{t}^{\text{cd}}(s_t|\bs_{1:t-1}) \propto \base(s_t|\bs_{1:t-1}) e^{\beta ~ \vartwist{t}(\bs_{1:t})} \label{eq:mudgal_prop} \tag{CD proposal}
\end{align}
We proceed to explain that this scheme \textit{does not correspond to sampling from the twist-induced proposal} under two different definitions of the target $\sigma(\bs_{1:T})$ (or \phishort $\finaltwist$) in our SMC framework.   
\vheader
\paragraph{Comparison with Our
$\finaltwist= r_{\text{cd}}(\bs_{1:T})$ Case:
}  As we have argued above, the CD-Q and CD-FUDGE may be viewed as learning twist values $\vartwist{t}$ for a terminal \phishort $\finaltwist= r_{\text{cd}}(\bs_{1:T})$.   However, our twist-induced proposal which minimizes the variance of the one-step importance weights with these SMC targets $\{ \pitwisttheta{t} \}$ would yield
\begin{align}
q_{t}^{\pi}(s_t|\bs_{1:t-1}) \propto \base(s_t|\bs_{1:t-1}) \vartwist{t}(\bs_{1:t}),
\tag{Twist-Ind. proposal ($\finaltwistonly = r_{\text{cd}}$)}\label{eq:mudgal_twist_ind}
\end{align}
which, compared to \cref{eq:mudgal_prop} does not exponentiate or scale $\vartwist{t}$ and is directly proportional to the expected $r_\text{cd}$.

\vheader
\paragraph{Comparison with Our $\finaltwist= e^{\beta r_{\text{cd}}(\bs_{1:T})}$ Case (Soft RL):}
The stochastic sampling in \cref{eq:mudgal_prop} is also reminiscent of the twist-induced proposal in the soft RL case of our framework where, in contrast to \cref{eq:cd_reward}, the target is defined via $\finaltwist= e^{\beta r_{\text{cd}}(\bs_{1:T})}$.   As in \cref{app:soft_rl}, 
\begin{align}
    q_{t}^{\pi}(s_t|\bs_{1:t-1}) \propto \base(s_t|\bs_{1:t-1}) e^{\beta ~ V_t^{\thb}(\bs_{1:t})}   \label{eq:soft_rl_twisted} \tag{Twist-Ind. proposal ($\finaltwistonly = e^{\beta r_{\text{cd}}}$)}
\end{align}
We proceed to write both $q_{t}^{\text{cd}}$ and $q_t^\pi$ as the solution to a variational optimization, \hl{highlighting similarities in {blue}}, but noting the \textit{different definitions of $\finaltwistonly$ in terms of $r_{\text{cd}}$}.   We assume no intermediate \phishort or reward, and consider the optimal twists to emphasize the role of $r_\text{cd}$.
Using \citet{mudgal2023controlled} Eq. 2 and Thm 2.1 (for CD) and \cref{eq:soft_value_app} (for soft RL), we have
\begin{align}
 q_t^{\text{cd}^*}(s_t|\bs_{1:t-1}) &= 
 \hl{ \argmax \limits_{\optdist(s_t|\bs_{1:t-1}) } \mathbb{E}_{\optdist(s_t|\bs_{1:t-1})}} \Big[
 \underbrace{ \mathbb{E}_{\base(\bs_{t+1:T}|\bs_{1:t})}\big[ \hl{ r_{\text{cd}}(\bs_{1:T}) }\big]}_{
 \psi_t^{\text{cd}*}(\bs_{1:t})  \text{ ~~(for $\finaltwistonly = r_\text{cd}$)}} \Big]  \hl{- \frac{1}{\beta} \DKL\pV{\optdist(s_{t}|\bs_{1:t-1})}{\base(s_{t}|\bs_{1:t-1})}} \tag{CD proposal optimization} \label{eq:mudgal_opt} \raisetag{10pt} \\[2ex]
  q_{t}^{\pi^*}(s_t|\bs_{1:t-1}) 
  &=\hl{ \argmax \limits_{\optdist(s_t|\bs_{1:t-1}) } \mathbb{E}_{\optdist(s_t|\bs_{1:t-1})} }
    \Big[ \underbrace{ \frac{1}{\beta} \log \mathbb{E}_{\base(\bs_{t+1:T}|\bs_{1:t})}\big[ e^{\beta \hl{ r_{\text{cd}}(\bs_{1:T}) }}  \big]}_{V_t^{*}(\bs_{1:t}) \text{ ~~(for $\finaltwistonly = e^{\beta r_\text{cd}}$)}} \Big] \hl{ - \frac{1}{\beta} \DKL\pV{\optdist(s_{t}|\bs_{1:t-1})}{\base(s_{t}|\bs_{1:t-1})}  } 
   \tag{Soft RL proposal optimization}\label{eq:mudgal_soft_rl_opt}
\end{align}
The second terms of \cref{eq:mudgal_opt} and \cref{eq:mudgal_soft_rl_opt} match and correspond to one-step KL divergence regularization of the policy $q_t(s_t|\bs_{1:t-1})$.   However, the expectation terms differ as we now discuss.

\vheader
\paragraph{Soft Values Account for Future Regularization}
Using \cref{eq:soft_value_app} to expand the definition of the soft value function, we see that \cref{eq:mudgal_soft_rl_opt} also implicitly contains an expected terminal reward,
\begin{align}
   V_t^*(\bs_{1:t}) = \frac{1}{\beta} \log \mathbb{E}_{\base(\bs_{t+1:T}|\bs_{1:t})} e^{\beta \hl{r_{\text{cd}}(\bs_{1:T})}} = \max \limits_{\optdist(\bs_{t+1:T}|\bs_{1:t})}  \mathbb{E}_{\optdist(\bs_{t+1:T}|\bs_{1:t})}\big[ \hl{ r_{\text{cd}}(\bs_{1:T})} \big] - \frac{1}{\beta} \DKL\pV{\optdist(\bs_{t+1:T}|\bs_{1:t})}{\base(\bs_{t+1:T}|\bs_{1:t})} \label{eq:value_mudgal_sec}
\end{align}
As $\beta \rightarrow 0$ in \cref{eq:value_mudgal_sec}, this optimization strictly enforces $\optdist(\bs_{t+1:T}|\bs_{1:t}) = \base(\bs_{t+1:T}|\bs_{1:t})$, and the soft value function recovers the expected reward under the base model $\mathbb{E}_{\base(\bs_{t+1:T}|\bs_{1:t})}[r_\text{cd}(\bs_{1:T})] $, which appears in first term \cref{eq:mudgal_opt}.
On the other hand, the second term in \cref{eq:mudgal_opt} uses $\beta > 0$ for optimization of the proposal $q(s_{t}|\bs_{1:t-1}) $ at the current step.
This inconsistency in \cref{eq:mudgal_opt} (using $\beta=0$ in the first term and $\beta>0$ in the second term) arises from the fact that \cref{eq:mudgal_opt} does not consider the effect of \textit{future} regularization, while the MDP formulation in \cref{eq:mudgal_soft_rl_opt} does so via the optimization in \cref{eq:value_mudgal_sec} and the log-mean-exp form of the soft value function $V_t^*$. 

\vheader
\paragraph{On \citet{mudgal2023controlled}'s One-Step Proposal and SMC Interpretation}
As noted in \cref{eq:cd_opt}, the twists learned by \citet{mudgal2023controlled} correspond to policy evaluation for the reward $r_\text{cd}$ under the base model $\base$.    
However, we have argued that the one-step proposal in \cref{eq:mudgal_prop} (which considers one-step KL regularization of $q_t^{\text{cd}}$ to $\base$) does not immediately 
fit within our SMC framework.   In particular, it is not apparent that the composition of one-step proposals $q^{\text{cd}}(\bs_{1:T}) = \prod_{\tau=1}^t q_\tau^{\text{cd}}(s_\tau|\bs_{1:\tau-1})$ samples from the marginals $\sigma(\bs_{1:t})$ of a natural target distribution $\sigma(\bs_{1:T})$ at optimality.

\vheader
\paragraph{
Flexible Inference-Time $\beta$ Scaling }

The experiments in \citet{mudgal2023controlled} evaluate tradeoff curves between expected reward and $\DKL\pV{q^{\text{cd}}(\bs_{1:T})}{\base(\bs_{1:T})}$ for various values of regularization strength $\beta$. 
Since the twists learned by \citet{mudgal2023controlled} in \cref{eq:cd_opt} do not depend on $\beta$,  sampling according to \cref{eq:mudgal_prop} or \cref{eq:mudgal_opt} has the benefit of allowing {flexible} tempering or $\beta$-scaling at inference time {without} additional learning.

Such tradeoff curves are also natural from the perspective of soft-RL (c.f. \cref{eq:variational_post} and \cref{eq:elbo_gap}).   While \cref{eq:value_mudgal_sec} appears to require {separate} twist-learning procedures for each $\beta$, flexible inference-time $\beta$ scaling 
could be achieved with a single training run in our framework by learning a conditional twist network
$\vartwist{t}(\bs_{1:t}, \beta)$ which considers $\beta$ 
in its input and training loss, or adapting the methods of \citep{bae2022multi} proposed in the context of rate-distortion optimization.

  \paragraph{Comparison with \citet{khanov2024args}}
 \citet{khanov2024args} consider softmax decoding similar to \cref{eq:mudgal_twist_ind}.   However, instead of $ V_t^\theta(\bs_{1:t})$ as the logit, they use a reward model $r_T(\bs_{1:T})$ which is trained from full sequences ($\finaltwist = e^{\beta r_T(\bs_{1:T})}$), but applied to partial sequences without modification, $r_T(\bs_{1:t})$.   This clearly does not correspond to a twist or soft value function $V_t^*(\bs_{1:t})=  \frac{1}{\beta} \log \sum_{\bs_{t+1:T}} \base(\bs_{t+1:T}|\bs_{1:t}) e^{\beta r_T(\bs_{1:T})} \neq r_T(\bs_{1:t})$.

\vheader
\subsection{Blockwise Greedy Decoding in \citet{mudgal2023controlled}}
\vheader
As an alternative use of the twist functions at inference time and a generalization of best-of-$K$ decoding to partial sequences, \citet{mudgal2023controlled} also consider 
a `blockwise' decoding scheme using the learned twist functions $\vartwist{t}$. 
In particular, for $K$ 
partial completions of length $M$ (from a prefix $\bs_{1:t}$),
sampled from the base model, $\bs_{t+1:t+M}^{(k)} \sim \base( \bs_{t+1:t+M} | \bs_{1:t})$, \citet{mudgal2023controlled} propose to choose 
\begin{align}
\bs_{t+1:t+M}^{\sind} = \argmax \limits_{k} \vartwist{t+M}(\bs_{1:t+M}^{k})
\end{align}
and proceed with sampling $K$ further continuations with prefix $\bs_{1:t+M}^{\sind}$ until the next resampling step or an end-of-string token is reached.
The $\argmax$ selection strategy may seem natural from the {unregularized} RL (as $\beta \rightarrow \infty$) or expected future reward perspective in \cref{app:mudgal_decoding}, but does not yield samples from $\sigma(\bs_{1:T})$ with the corresponding optimal twists.

Our SMC framework instead would advocate \textit{probabilistic} resampling based on the approximate twist functions using the ($c$- or $M$-step) importance weights 
in \cref{sec:twist_smc} in order to match the desired target distribution.  %

Finally, \citet{khanov2024args} also consider $\argmax$ decoding of next tokens using the unmodified $r_T(\bs_{1:t})$ described above.
\section{Proposal Learning Methods}\label{app:prop}
We next describe methods for learning variational policies or proposals $q^\qparam(\bs_{1:T}\condzero) = \prod_{t=1}^T q_t^\qparam(s_{t}|\bs_{1:t-1})$
\normalsize
parameterized by $\qparam$, 
which can be used for \gls{SMC} sampling with intermediate targets %
$\pitwisttheta{t}(\bs_{1:t}\condzero)$ and learned twists $\vartwist{t}(\bs_{1:t})$ or $V_t^{\thb}(\bs_{1:t})$ parameterized by $\thb$.
Alternatively, such proposals may be used directly in the \gls{IWAE} bounds on $\log \cZ_\sigma$, which rely on simple importance sampling over full sequences as in \cref{sec:simple_is} and \textit{do not require the definition of intermediate targets} $\pitwist{t}$.

In \cref{app:dpg}, we provide a detailed description of the DPG policy gradient method, which can be interpreted as a maximum likelihood objective for a sequential energy-based model.   To distinguish this EBM approach from our  CTL method for twist learning, we emphasize issues which can arise from naive use of a proposal-learning objective to define intermediate twisting targets for SMC in \cref{app:twist_from_proposal}.   %

\subsection{Path Consistency Learning for Controlled Generation}\label{app:pcl}
\citet{guo2021efficient} consider learning $Q$-values to obtain a fine-tuned variational policy which can be directly used as a sampling distribution for controlled generation.   
Building on the \gls{PCL} loss in 
\citet{nachum2017bridging} and \cref{app:pcl_twist_learning},
\citet{guo2021efficient} consider parameterizing the proposal using $Q_t^\qparam(s_{t},\bs_{1:t-1})$, 
\begin{align}
 \optdist_{t}^{\qparam}(s_{t}|\bs_{1:t-1}) =  \base(s_{t}|\bs_{1:t-1}) \expof{\beta Q_t^\qparam(s_t, \bs_{1:t-1}) - \beta V_{Q^\qparam}(\bs_{1:t-1})}
 \label{eq:proposal_qvalues1}
\end{align}
where $V_{Q^\qparam}(\bs_{1:t-1}) = \frac{1}{\beta} \log \sum_{s_t} \base(s_t|\bs_{1:t-1}) e^{\beta Q_t^\qparam }$ enforces normalization.

\citet{guo2021efficient} define the targets using $\bar{Q}_t^\qparam(s_{t}, \bs_{1:t-1})$, a slowly-updated target network based on $Q_t^\qparam$.   Using the implied form of the soft value $\bar{V}(\bs_{1:t-1}) \coloneqq \frac{1}{\beta} \log \sum_{s_{t}} \base(s_{t}|\bs_{1:t-1}) e^{\beta \bar{Q}_t^\qparam(s_{t}, \bs_{1:t-1})}$, %
the single-step \gls{PCL} loss becomes
\begin{align}
    \cL_{\textsc{pcl}-\textsc{q}}({\qparam})     &= \min \limits_{\qparam} \sum \limits_{t=1}^{T} \mathbb{E}_{\samplingdist(\bs_{1:t}\condzero)}\bigg[ \Big(  r_{t}(\bs_{1:t}) +  \texttt{sg}(\bar{V}_{t}(\bs_{1:t})) -  \texttt{sg}(\bar{V}_{t-1}(\bs_{1:t-1}))
    - Q_t^\qparam(s_{t},\bs_{1:t-1}) + V_{Q^\qparam}(\bs_{1:t-1}) \Big)^2 \bigg] \label{eq:pcl_guo}
\end{align}
where $\texttt{sg}(\cdot)$ indicates stop gradient. 
Building on the interpretation in \cref{sec:consistency_twist}, we view $ \bar{V}_t(\bs_{1:t})$ and $\bar{V}_{t-1}(\bs_{1:t-1})$ as the twisting targets, with a learned proposal parameterized by $Q_t^\qparam$ as in \cref{eq:proposal_qvalues1} (or \cref{app:final_timestep}).
While the loss in \cref{eq:pcl_guo} is similar in practice to the soft Q-learning loss 
in \cref{app:soft_q}, we emphasize that the latter is motivated from the SMC perspective with the twisting targets as the primary object of interest and flexibility in the choice of proposal.   By contrast, \citet{guo2021efficient} are interested in learning a proposal policy and do not consider, for example, resampling according to $\bar{V}_t$.

\citet{guo2021efficient, nachum2017bridging} also consider `multi-step' \gls{PCL} losses (\cref{eq:pcl_multi}) which
use observed reward during rollouts of length $\lambda$ to limit reliance on estimated intermediate values $\bar{V}_{t}(\bs_{1:t})$.
The objective in \citet{hu2023amortizing} also corresponds to a \gls{PCL} objective.

\vheader 
\subsection{Policy Gradient Methods}
\vheader
Traditional \gls{RLHF} pipelines use a policy gradient method such as \gls{PPO} to optimize the objective in \cref{eq:variational_post},
\begin{align}
\mathcal{L}_{\textsc{elbo}}(\qparam) = \max \limits_{\qparam} ~ \mathbb{E}_{q^{\qparam}(\bs_{1:T}\condzero)}\left[ r_T(\bs_{1:T}) \right] - \frac{1}{\beta}\DKL\pV{q^{\qparam}(\bs_{1:T}\condzero)}{\base(\bs_{1:T}\condzero) } = \min \limits_{\qparam} \DKL \pV{q^\qparam(\bs_{1:T}\condzero) }{ \target(\bs_{1:T}\condzero) }
\end{align}
where $r_T(\bs_{1:T}) = \frac{1}{\beta} \log \finaltwist$ corresponds to our final twist.  As in \cref{eq:elbo_gap}, the gap in this optimization is the mode-seeking \textsc{KL} divergence $\DKL\pV{\prop^\qparam(\bs_{1:T})}{\sigma(\bs_{1:T})}$.

Notably, this objective does not make use of exact target samples from $\target(\bs_{1:T}\condzero)$ when they are available.  Further, the mode-seeking behavior has been shown to reduce diversity of fine-tuned models \citep{stiennon2020learning, go2023aligning}.
To combat this, \citet{go2023aligning} derive policy gradient methods to optimize arbitrary $f$-divergences  $D_f\pV{q^\qparam(\bs_{1:T}\condzero) }{ \target(\bs_{1:T}\condzero)} $ between the learned variational policy $q^{\qparam}$ and target $\target$.    

\vheader\subsection{Policy Gradient with Mass-Covering / Maximum Likelihood KL Divergence}\label{app:dpg}
\vheader
We focus on the case of minimizing the mass-covering \textsc{kl} divergence $\DKL\pV{ \target(\bs_{1:T}\condzero) }{ q^\qparam(\bs_{1:T}\condzero)}$ to train $q_{\qparam}$, which constitutes the distributional policy gradients (\textsc{dpg}) method for language model finetuning \citep{parshakova2019distributional, khalifa2020distributional, korbak2022controlling, go2023aligning} and has been used to learn \gls{SMC} proposals in state-space models in \citep{gu2015neural}.

In particular, the gradient of $\DKL\pV{ \target(\bs_{1:T}) }{ q^\qparam(\bs_{1:T})} = \mathbb{E}_{\target(\bs_{1:T})}[ \log {\target(\bs_{1:T})} - \log{ q^\qparam(\bs_{1:T})}] $
is 
\begin{align}
\begin{split}
\nabla_{\qparam} \DKL\pV{\target(\bs_{1:T}\condzero)}{ q^\qparam(\bs_{1:T}\condzero)} = - \mathbb{E}_{\target(\bs_{1:T}\condzero)}\left[ \nabla_{\qparam} \log q^\qparam(\bs_{1:T}\condzero)\right]&= -\mathbb{E}_{ q^\qparam(\bs_{1:T}\condzero)}\left[ \frac{\target(\bs_{1:T}\condzero)}{q^\qparam(\bs_{1:T}\condzero)} \nabla_{\qparam} \log q^\qparam(\bs_{1:T}\condzero) \right] \\
&= -\mathbb{E}_{ q^\qparam(\bs_{1:T}\condzero)}\left[\frac{1}{\cZ_{\sigma}} \frac{\ttarget(\bs_{1:T}\condzero)}{q^\qparam(\bs_{1:T}\condzero)} \nabla_{\qparam} \log q^\qparam(\bs_{1:T}\condzero) \right]
\label{eq:kl_grad} 
\end{split}
\end{align}
We recognize the importance weights $ w(\bs_{1:T}) = \frac{\ttarget(\bs_{1:T}\condzero)}{q^\qparam(\bs_{1:T})} $ from \cref{eq:single_sample_w_ali}.
\citet{go2023aligning} consider estimating \cref{eq:kl_grad} using a moving average estimate of the partition function $\hat{Z}_\sigma$
\begin{align}
\nabla_{\qparam} \DKL\pV{ \target(\bs_{1:T}\condzero) }{ q^\qparam(\bs_{1:T}\condzero)}~ \approx ~ \sum_{k=1}^K \frac{1}{\hat{Z}_\sigma
} w(\bs_{1:T}^{(k)})  \nabla_{\qparam} \log q^\qparam(\bs_{1:T}^{(k)}\condzero) \label{eq:dpg} \tag{DPG (general $\hat{Z}_\sigma$)},
\end{align}
Alternatively, the expectation may thus be estimated using 
\gls{SIS} with the variational policy $q^\qparam(\bs_{1:T}\condzero)$. 
Using \gls{SNIS} to estimate \cref{eq:kl_grad} as in \cref{eq:snis} %
corresponds to $\hat{Z}_\sigma = \sum_{j=1}^K w(\bs_{1:T}^{(k)})$, with
\begin{align}
\nabla_{\qparam} \DKL\pV{\target(\bs_{1:T}) }{ q^\qparam(\bs_{1:T})}~ &\approx ~ \sum_{k=1}^K \frac{w(\bs_{1:T}^{(k)})}{\sum_{j=1}^K w(\bs_{1:T}^{(j)})}\nabla_{\qparam} \log q^\qparam(\bs_{1:T}^{(k)}\condzero). 
 \label{eq:dpg_snis}
\end{align}
\normalsize
We use this gradient for DPG proposal learning in the main text experiments, 
 although we use the parameterization 
described in \cref{eq:snis_grad} below.

\paragraph{{DPG as Sequential Maximum Likelihood Objective}}
We now show \cref{eq:dpg_snis} is equivalent to
a {sequential maximum likelihood \gls{EBM} objective}.   Consider minimizing the KL divergence,
\begin{align}
\klof{\target(\bs_{1:T}\condzero)}{q^\qparam{}(\bs_{1:T}\condzero)} =  \sum \limits_{t=1}^{T} &~ \mathbb{E}_{\target(\bs_{1:t-1}\condzero)} \klof{\target(s_t|\bs_{1:t-1})}{ 
q_t^\qparam(s_t|\bs_{1:t-1})
} \label{eq:full_ebm} \tag{EBM proposal learning} \\
\text{where} \qquad 
q_t^\qparam(s_{t}|\bs_{1:t-1}) &= \frac{\base(s_{t}|\bs_{1:t-1}) \vartwistq{t}(\bs_{1:t})}{\sum_{s_{t}} \base(s_{t}|\bs_{1:t-1}) \vartwistq{t}(\bs_{1:t}) }. \label{eq:qprop}
\end{align}
\normalsize
{
While this is reminscent of the twist-induced proposal in \cref{prop:transition}, we emphasize distinctions between energy-based learning of the proposal (DPG) versus energy-based learning of twist functions (CTL) in  \cref{app:twist_from_proposal}.}

The gradient of \cref{eq:full_ebm} becomes
\small 
\begin{align}
\nabla_{\qparam} \klof{\target(\bs_{1:T})}{q^\qparam(\bs_{1:T})} &= \sum \limits_{t=1}^{T} \mathbb{E}_{\target(\bs_{1:t-1}\condzero)}\Big[ \mathbb{E}_{\target(s_{t}|\bs_{1:t-1})}\big[ \nabla_\qparam
\log \vartwistq{t}(\bs_{1:t}) \big]
- \mathbb{E}_{
q^{\qparam}_t(s_t|\bs_{1:t-1})}
\big[ \nabla_\qparam 
\log \vartwistq{t}(\bs_{1:t}) 
\big] \Big] ~ . \label{eq:seq_ebm}
\end{align}
\normalsize
Starting from \cref{eq:dpg_snis}, we now seek to recover \cref{eq:seq_ebm}.   
Using \cref{eq:qprop}, we can write 
\small \begin{align*}
\log \optdist^\qparam(\bs_{1:T}^{(k)}\condzero)  &= \sum_{t=1}^{T}
\big( \log \base(s_{t}^{(k)}|\bs_{1:t-1}^{(k)}) + \log \vartwistq{t}(\bs_{1:t}^{(k)}) - \log \sum_{s_{t}} \base(s_{t} |\bs_{1:t-1}^{(k)}) \vartwistq{t}(s_t,\bs_{1:t-1}^{(k)}) \big) \\
\nabla_\qparam \log \optdist^\qparam(\bs_{1:T}^{(k)}\condzero)&= \sum_{t=1}^{T} \left( \nabla_\qparam \log  \vartwistq{t}(\bs_{1:t}^{(k)}) - \mathbb{E}_{q^{\qparam}_t(s_{t}|\bs_{1:t-1}^{(k)})}\left[ \nabla_\qparam \log 
\vartwistq{t}(s_t,\bs_{1:t-1}^{(k)})\right]\right)
\end{align*}
\normalsize
Substituting into \cref{eq:dpg_snis}, we recover
\begin{align}
\hspace*{-.22cm} \nabla_{\qparam} \DKL\pV{\target(\bs_{1:T}) }{ \optdist^\qparam(\bs_{1:T})}~ &\approx ~ \sum_{k=1}^K \frac{w(\bs_{1:T}^{(k)})}{\sum_{j=1}^K w(\bs_{1:T}^{(k)})}\sum \limits_{t=1}^{T} \left(  \nabla_{\qparam} \log \vartwistq{t}(\bs_{1:t}^{(k)}) - 
\mathbb{E}_{q^{\qparam}_t(s_{t}|\bs_{1:t-1}^{(k)})}\left[  \nabla_{\qparam} \log
\vartwistq{t}(s_t,\bs_{1:t-1}^{(k)}) \right]
\right) \label{eq:snis_grad} \tag{DPG}
\end{align}
\normalsize
which is an SNIS estimate of the maximum likelihood EBM gradient in \cref{eq:seq_ebm}, as desired.   Note that the expectation over $q^{\qparam}_t(s_{t}|\bs_{1:t-1}^{(k)})$ can be calculated exactly.

\paragraph{Comparison with CTL Objective}   The gradient in \cref{eq:snis_grad} above appears similar to our \gls{CTL} objective and gradient in \cref{sec:ctl}.    However, the DPG loss in \cref{eq:full_ebm} is a single (joint) KL divergence over the entire sequence, whereas CTL optimizes $T$ separate KL divergences for each intermediate marginal. %

For the DPG gradient in \cref{eq:seq_ebm}, negative sampling is performed using a `positive' prefix $\bs_{1:t-1}^{(k)} \sim \target(\bs_{1:t-1})$ and an \textit{exact} `negative' sample from the one-step-ahead $q_t^{\qparam}(s_{t}|\bs_{1:t-1}^{(k)})$ (\cref{eq:qprop}, which we have assumed to be tractable).   In practice, we obtain the prefixes using the truncation of exact samples or approximate positive sampling with the final target weights as in \cref{eq:snis_grad}.
By contrast, the CTL gradient in \cref{eq:grad_separate_ebm}
involves \textit{approximate} negative sampling under each $\pitwist{t}(\bs_{1:t})$.

\newcommand{\vartwistqo}[1]{\psi_{#1}^{\qparam*}}
\newcommand{\vartwistqc}[1]{\psi_{#1}^{\qparam c}}
\newcommand{\zonestepqq}[2]{Z^{\qparam}_{#1}(\bs_{1:#2})}

\vheader
\subsubsection{Naive Use of Proposal Learning to define Twisted SMC Targets}\label{app:twist_from_proposal} %
\vheader
While we have shown in 
\cref{prop:transition} how one-step proposals $\{ \proptwist{t}(s_t|\bs_{1:t-1}) \}_{t=1}^{T}$ can be induced from a given set of twist functions $\{ \approxtwist{t}(\bs_{1:t}) \}_{t=1}^{T}$ or target distributions $\{ \pitwist{t}(\bs_{1:t}) \}_{t=1}^T$, we now emphasize that moving the other direction (inducing intermediate twisting targets from a proposal learning scheme parameterized by $\{\vartwistq{t}\}_{t=1}^T$) does not yield the correct intermediate targets for resampling (\cref{app:optimality}), even at optimality in the proposal learning objective.

We focus our arguments on learning with the EBM maximum likelihood objective in \cref{eq:full_ebm} as an example.
The proposal energies $\vartwistq{t}(\bs_{1:t})$ appear to play a role analogous to the twist function $\approxtwist{t}(\bs_{1:t})$ in the one-step proposal induced from twist targets $\{\pitwist{t} \}_{t=1}^T$ in \cref{sec:twist_smc}.  

However, we proceed to show in \cref{prop:twist_issue} that naive use of $\vartwistq{t}$ to define twisting targets using \footnote{We assume no intermediate \phishorts in this section, as in the main text.}
\begin{align}
\hspace*{-.25cm}
\begin{split}
\pitwist{t}^{\qparam} (\bs_{1:t}\condzero) = \begin{cases} 
\frac{1}{\zfilter{t}}~  p_0(\bs_{1:t}\condzero) ~
\vartwistq{t}(\bs_{1:t}) \qquad  \hfill t \neq T\\
\frac{1}{\cZ_\sigma} ~ p_0(\bs_{1:T}\condzero)~ 
\finaltwist
\hfill t= T
\end{cases}
\end{split} \label{eq:filtering_q} \raisetag{20pt}
\end{align}
{need not lead to an SMC procedure for which $\pitwist{t}^{\qparam}(\bs_{1:t}) = \sigma(\bs_{1:t})$, even if $ 
\prop^{\qparam}_t(s_t|\bs_{1:t-1}) = \target(s_t|\bs_{1:t-1})$ for all $t$.     We thus argue that $\vartwistq{t}$ learned using \cref{eq:full_ebm} \textit{should not be used} as target twists in \cref{eq:filtering_q}, since they do not yield the optimal interemdiate target distributions at optimality (\cref{app:optimality}).}

We begin by showing a simple lemma for the one-step conditionals in \cref{eq:full_ebm}.
\begin{lemma}\label{lemma:opt_prop}
Any twist induced proposal $\prop^{\qparam}_t(s_t|\bs_{1:t-1})$ (induced by $\vartwistq{t}(\bs_{1:t})$) is invariant to rescaling $\vartwistq{t}(\bs_{1:t})$ by an arbitrary constant $c(\bs_{1:t-1})$ with respect to $\bs_{1:t-1}$,
\begin{align}
\vartwistqc{t}(\bs_{1:t}) \coloneqq c(\bs_{1:t-1}) \vartwistq{t}(\bs_{1:t}) \label{eq:opt_twists_q}
\end{align}
\end{lemma}
\begin{proof}
\small 
\begin{align*}
\prop^{\qparam c}_t(s_{t}|\bs_{1:t-1})
=
\frac{\base(s_{t}|\bs_{1:t-1}) \vartwistqc{t}(\bs_{1:t})}{\sum_{s_{t}} \base(s_{t}|\bs_{1:t-1}) \vartwistqc{t}(\bs_{1:t}) } 
&= 
\frac{\base(s_{t}|\bs_{1:t-1}) c(\bs_{1:t-1})\vartwistq{t}(\bs_{1:t})}{\sum_{s_{t}} \base(s_{t}|\bs_{1:t-1}) c(\bs_{1:t-1})\vartwistq{t}(\bs_{1:t}) }
= 
\frac{\base(s_{t}|\bs_{1:t-1}) \vartwistq{t}(\bs_{1:t})}{\sum_{s_{t}} \base(s_{t}|\bs_{1:t-1}) \vartwistq{t}(\bs_{1:t}) }
= 
\prop^{\qparam}_t(s_{t}|\bs_{1:t-1})
\,.\nonumber 
\end{align*}
\normalsize
\end{proof}

\vheader 
\begin{proposition}\label{prop:twist_issue}
There exist $\{ \vartwistqo{t} \}_{t=1}^T$ such that (i) $\prop^{\qparam*}_t(s_t|\bs_{1:t-1}) = \sigma(s_t|\bs_{1:t-1})$ and (ii) the SMC targets $\{ \pitwist{t}^{\qparam*}(\bs_{1:t}) \}_{t=1}^T$ induced by $\{ \vartwistqo{t} \}_{t=1}^T$ via \cref{eq:filtering_q} are different from $\sigma(\bs_{1:t})$.
\end{proposition}
\begin{proof}
To satisfy condition (i) of the current proposition, we define
\begin{align}
\vartwistqo{\tau}(\bs_{1:\tau}) \coloneqq \begin{cases}
\phantom{c(\bs_{1:\tau-1})} \sum_{\bs_{\tau+1:T}} \base(\bs_{\tau+1:T}|\bs_{1:\tau}) \finaltwist \hfill \tau \neq t \\ c(\bs_{1:t-1}) ~ \sum_{\bs_{t+1:T}} \base(\bs_{t+1:T}|\bs_{1:t}) \finaltwist \qquad \hfill \tau = t
\end{cases}\label{eq:opt_twists_q}
\end{align}
which for all $\tau$, yields optimal proposals:  $(i)$ $\prop^{\qparam*}(s_{\tau}|\bs_{1:\tau-1}) = \sigma(s_{\tau}|\bs_{1:\tau-1}) \propto \base(s_\tau|\bs_{1:\tau-1})\vartwistqo{\tau}(\bs_{1:\tau}) $ via \cref{lemma:opt_prop}.
However, it is clear that $c(\bs_{1:t-1}) \neq 1$ can break the necessary condition for optimality of SMC sampling that $\pi_t(\bs_{1:t}) = \sigma(\bs_{1:t})$ (\cref{prop:opt_intermediate}).   In particular, consider 
\begin{align}
\pitwist{t}^{\qparam*} (\bs_{1:t}\condzero) = \frac{1}{\zfilter{t}}~  p_0(\bs_{1:t}\condzero) ~
\vartwistqo{t}(\bs_{1:t})  &= \frac{1}{\zfilter{t}}~ c(\bs_{1:t-1}) p_0(\bs_{1:t}\condzero) ~ \sum_{\bs_{t+1:T}} \base(\bs_{t+1:T}|\bs_{1:t}) \finaltwist \nonumber
\\ &= 
\frac{1}{\zfilter{t}} c(\bs_{1:t-1}) \tilde{\sigma}(\bs_{1:t}) \neq \sigma(\bs_{1:t}) 
\end{align}
\normalsize
for $c(\bs_{1:t-1}) \neq 1$, which introduces an additional factor which depends on $\bs_{1:t}$.   Thus, the twist target $\pitwist{t}^{\qparam*} (\bs_{1:t}\condzero)$ induced from $\vartwistqo{t}(\bs_{1:t})$ in \cref{eq:opt_twists_q} is not equal to the desired marginal $\sigma(\bs_{1:t})$, despite the fact that all proposals are optimal.
\end{proof}

We indeed observed experimentally that resampling based on \cref{eq:filtering_q} after training using \cref{eq:full_ebm} could lead to \textit{worse} SMC $\log \cZ_\sigma$ bounds than simply calculating the SIS or IWAE bound with $\prod_{t=1}^T q_t^{\qparam}(s_t|\bs_{1:t-1})$.%

\paragraph{Optimality in CTL Objective implies Optimal Twisted SMC}
In contrast to \cref{prop:twist_issue}, note that the
global optimum of our CTL objective $\min \sum_{t=1}^{T} \DKL\pV{\sigma(\bs_{1:t})}{\pitwist{t}^{\psi}(\bs_{1:t})}$ 
(which occurs for the optimal twists $\{\optimaltwist{t}\}_{t=1}^{T-1}$ in \cref{prop:optimal_twists}), results in both the twist-induced proposal $\propopttwist{t}(s_t|\bs_{1:t-1}) = \sigma(s_t|\bs_{1:t-1})$ and the twisting targets $\pitwist{t}^*(\bs_{1:t}) = \sigma(\bs_{1:t})$
satisfying the necessary and sufficient conditions for optimality outlined in \cref{app:optimality} \cref{prop:optimality_conditions}.

\newcommand{\vartwistqon}[1]{\bar{\psi_{#1}^{*}}}
\vheader
\subsubsection{
SMC with Normalized Targets Induced by Learned Proposal 
Leads to Uniform Weights
}\label{app:normalized_weights_fix}

\vheader
The issue in \cref{prop:twist_issue} arises from the degree of freedom $c(\bs_{1:t-1})$ in the normalization constant of the one-step proposal.   
To avoid this, we can instead define \textit{normalized} twisted intermediate targets using
\begin{align}
\begin{split}
\tpitwist{t}^{\qparam} (\bs_{1:t}\condzero) &= \begin{cases}  
 p_0(\bs_{1:t}\condzero) ~ \prod \limits_{\tau=1}^t \frac{\vartwistq{\tau}(\bs_{1:\tau}) }{\zonestepqq{\tau}{\tau-1}} ~~= ~{\prod\limits_{\tau=1}^t \prop^{\qparam}_\tau(s_\tau | \bs_{1:\tau-1}) }
\qquad  \hfill t \neq T\\
p_0(\bs_{1:T}\condzero)~ 
\finaltwist 
\hfill t= T
\end{cases}
\end{split} \label{eq:filtering_q2} \raisetag{20pt}
\end{align}
where $\zonestepqq{t}{t-1}$ arises from the proposal $%
\prop^{\qparam}_t(s_t|\bs_{1:t-1})\coloneqq \frac{1}{\zonestepqq{t}{t-1}}
\base(s_{t}|\bs_{1:t-1})\vartwistq{t}(\bs_{1:t})$ learned according to \cref{eq:full_ebm}.   

Crucially, $\tpitwist{t}^{\qparam}$ in \cref{eq:filtering_q2} are automatically normalized for $t \neq T$, as the product of normalized proposals.
In this case, SMC resampling with $\prop^{\qparam}$ or the twist-induced proposal 
yields uniform resampling weights,
\begin{equation}
\hspace*{-.22cm} \resizebox{.95\textwidth}{!}{\ensuremath{
(\text{for } t < T):\,\,  w_t(\bs_{1:t}) = 
    \frac{\tpitwist{t}^{\qparam}(\bs_{1:t}\condzero) }{ \tpitwist{t-1}^{\qparam}(\bs_{1:t-1}\condzero) \prop^{\qparam}(s_{t}|\bs_{1:t-1})} = 
    \frac{\base(\bs_{1:t})\prod \limits_{\tau=1}^t \frac{\vartwistq{\tau}(\bs_{1:\tau}) }{\zonestepqq{\tau}{\tau-1}} }{\base(\bs_{1:t-1}) \left( \prod \limits_{\tau=1}^{t-1} \frac{\vartwistq{\tau}(\bs_{1:\tau}) }{\zonestepqq{\tau}{\tau-1}} \right) \frac{1}{\zonestepqq{t}{t-1} } \base(s_{t}|\bs_{1:t-1}) \vartwistq{t}(\bs_{1:t})}
    =
   1 
   }}\label{eq:weights_equal_1}
\end{equation}
Although we were able to construct well-behaved intermediate twisting targets from a proposal-learning scheme $\prop^{\qparam}_t(s_t|\bs_{1:t-1}) \propto \base(s_{t}|\bs_{1:t-1})\vartwistq{t}(\bs_{1:t})$, 
\cref{eq:weights_equal_1} shows that this \textit{does not 
lead to meaningful intermediate SMC resampling}. 
{In other words, for $t < T$, the marginal distributions of SMC samples $\s_{1:t}^k$ with 
this scheme
are simply $\prop^{\qparam}(\bs_{1:t})$,
the same as we would obtain with no resampling (SIS/IWAE).
}
\vheader
\section{Bidirectional SMC}\label{app:bounds_smc}\label{app:smc}
\vheader

\newcommand{\mys}{s}%
\newcommand{\bom}{\boldsymbol{h}}%
\newcommand{\barbom}{\boldsymbol{\bar{h}}}
\newcommand{\listl}{\pmb{[} }%
\newcommand{\listr}{\pmb{]} }%

In this section, we recall the extended state-space probabilistic interpretation of \gls{SMC} from \citep{maddison2017filtering, andrieu2010particle}.     
The idea is to define an unnormalized target distribution $\smctgt(\boldsymbol{S})$ and normalized proposal $\smcprop(\boldsymbol{S})$ over an extended state space $\boldsymbol{S}$ containing all random variables relevant to SMC sampling and importance weighting with $K$ sequences of length $T$.
Defining $\tsmctgt(\boldsymbol{S})$ such that its normalization constant matches $\cZ_\sigma$, 
we can use 
simple importance sampling (SIS) in this extended state space to show that $K$-sequence SMC sampling yields an unbiased estimator of $\cZ_\sigma$, for example $\cZ_\sigma = \mathbb{E}_{\smcprop(\boldsymbol{S}) }[\frac{\tsmctgt(\boldsymbol{S})}{\smcprop(\boldsymbol{S})}]$ (as in \cref{eq:smc_ess}).
Our end goal is to use this probabilistic interpretation to derive the lower and upper bounds on $\log \cZ_\sigma$ in \cref{prop:smc_proposition}, following \citet{brekelmans2022improving} App. A.    

We define the extended state space proposal and target distributions below, noting that our bounds will require sampling from normalized $\smctgt(\boldsymbol{S})$ or $\smcprop(\boldsymbol{S})$, and evaluating $\tsmctgt(\boldsymbol{S})$ and $\smcprop(\boldsymbol{S})$. %
We summarize the algorithm for sampling $\smctgt(\boldsymbol{S})$ in \cref{alg:smc_target}, using concatenation notation for simplicity instead of the 
probabilistic interpretation using index histories in the text.%

\vheader
\paragraph{Single-Sequence Target and Proposal}
We construct our importance sampling bounds with the goal of estimating the (log) partition function and sampling from a target distribution $\sigma(\bs_{1:T}\condzero) = \ttarget(\bs_{1:T}\condzero) / \cZ_\sigma$.    
We leverage a sequence of intermediate target distributions, $
\{ \pitwist{t}(\bs_{1:t})  = \frac{1}{\cZ_t} \tpitwist{t}(\bs_{1:t}) \} _{t=1}^{T}$ over partial sequences, with the final target $\pitwist{T}(\bs_{1:T}) = \sigma(\bs_{1:T})$ and $\cZ_T= \cZ_\sigma$.   We assume $\tilde{\pi}_{0}(\bs_0)=1$ for all prompts with $\cZ_0 = 1$.
Finally, our bounds and sampling procedures also depend on a given set of proposal distribution $\{q\pv{s_{t}}{\s_{1:t-1}}\}_{t=1}^T$, as in \cref{sec:smc_bg}. 

\vheader
\paragraph{Extended State Space Random Variables}
Consider an extended state space $\boldsymbol{S}$ containing $KT$ tokens  $ \{ \mys_{t}^{k} \}_{t=1,k=1}^{T,K}$ with $\mys_t^k \in \cV$ and $KT$ indexing random variables
$\{\omega_t^{k} \}_{t=1,k=1}^{T,K}$ with $\omega_t^{k} \in [1,K]$, 
to represent the results of resampling (\cref{eq:incremental_resample}),
\begin{align}
\boldsymbol{S} \coloneqq \big\{ \mys_{t}^{k}, \omega_t^{k} \big\}_{t=1,k=1}^{T,K}
\end{align}
For ease of notation (and similarly to \citet{maddison2017filtering, andrieu2010particle}), we call attention to our use of recursive backtracking index 
operations to collect sequences $\{\bs_{1:t} \}$ based on the results of resampling $\{\omega_t^{k}\}$.   
We use \textit{lists} of index histories to construct sequences of tokens, with two recursive definitions of histories.  
Letting {$+$} %
indicate 
appending of lists,%
\def\rddots#1{\cdot^{\cdot^{\cdot^{#1}}}}
\begin{align}
\qquad \qquad \begin{split}
\bom_{0}^{k} \coloneqq \listl ~ 
\listr ~~~ \forall k,
\qquad %
\bom_{t}^{k} &\coloneqq \bom_{t-1}^{\omega_t^k}  + \listl \omega_t^k \listr
\\
\barbom_{0}^{k} \coloneqq \listl ~ 
\listr ~~~ \forall k,
\qquad  \barbom_{t}^k &\coloneqq \bom_{t-1}^{k} + \listl k \listr \qquad \phantom{\omega_t^k} 
\end{split}\label{eq:index} \tag{Index Notation} \raisetag{15pt}
\\[2ex]
\intertext{
For example, the history $\bom_{t-1}^{k}$ will be used to construct prefix sequences $\bs_{1:t-1}^{\bom_{t-1}^{k}}$ (i.e. lists of tokens) for sampling a next token $s_t^k$.  
We denote sequences of tokens with the index history in the superscript and also expand the definition for clarity,
} 
\begin{split}
\qquad \qquad \bs_{1:t}^{\bom_{t}^{k}} &\coloneqq \bs_{1:t-1}^{\bom_{t-1}^{\omega_t^k}} + \listl s_{t}^{\omega_{t}^k}\listr  ~~\quad = \listl s_1^{\bom_{t-1}^{\omega_{t}^{k}}[1]},   ... , s_{t-1}^{\bom_{t-1}^{\omega_{t}^{k}}[t-1]}, s_{t}^{\omega_{t}^{k}} \listr  = \listl s_{1}^{\omega_{1}^{\rddots{\omega_{t}^{k}}}},   ... , s_{t-2}^{\omega_{t-2}^{\omega_{t-1}^{\omega_{t}^{k}}}}, s_{t-1}^{\omega_{t-1}^{\omega_{t}^{k}}}, s_{t}^{\omega_{t}^{k}} \listr 
 \\[1.25ex]
\bs_{1:t}^{\barbom_{t}^{k}} 
&\coloneqq \bs_{1:t-1}^{\bom_{t-1}^{k}} + \listl s_t^k \listr 
\end{split}
\label{eq:sequence}\raisetag{10pt} \tag{Sequence Notations} 
\end{align}
In the second line, we define $\bs_{1:t}^{\barbom_{t}^{k}}$ as a sequence of length $t$ which concatenates the  prefix $\bs_{1:t}^{\bom_{t-1}^{k}}$ with next token $s_t^k$.  The notation $\bs_{1:t}^{\barbom_{t}^{k}}$ represents partial sequences \textit{before} resampling.
By contrast, we will use the notation $\bs_{1:t}^{\bom_{t}^{k}}$ in the first line of \cref{eq:sequence} to refer to sequences \textit{after} resampling. 

Consider the sequence $\bs_{1:t}^{\barbom_{t}^{i}}$ in a particular index $i \in [1,K]$ \textit{before} resampling.
Resampling at time $t$ may
result in choosing $\omega_t^k = i$ for some $k$.   Using the first line, we see that $\bs_{1:t}^{\bom_t^k} = \bs_{1:t-1}^{\bom_{t-1}^{\omega_{t}^k}} + \listl s_{t}^{\omega_{t}^k}\listr = \bs_{1:t-1}^{\bom_{t-1}^{i}} + \listl s_{t}^{i}\listr$ 
for those indices such that $\omega_{t}^k = i$.   Indeed, this matches the definition of $\bs_{1:t}^{\barbom_{t}^{i}} = \bs_{1:t-1}^{\bom_{t-1}^{i}} + \listl s_{t}^{i}\listr$ in the second line (before resampling).   Thus, the indexing notation in \cref{eq:sequence} reflects resampling or cloning of sequences $\bs_{1:t}^{\barbom_{t}^{i}}$ into the indices such that $\omega_t^k = i$, which yields prefixes $\bs_{1:t}^{\bom_{t}^k}$ for the next step of sampling ($t+1$) in each index $k \in [1,K]$.

\begin{figure*}[!t]
\begin{minipage}{.38\textwidth}
    \centering
    \includegraphics[width=1\textwidth]{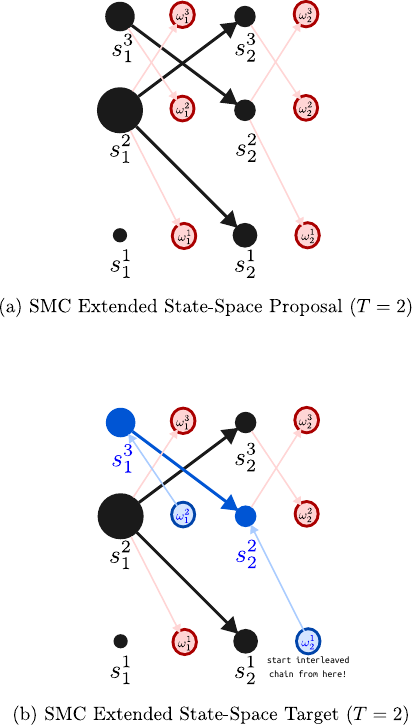}
    \caption{Graphical Models for extended state-space proposal and target distributions which result in the bidirectional SMC bounds.
    We show density evaluation in the proposal and target for a fixed set of $\{ s_t^k, \omega_t^k \}_{k=1, t= 1}^{3,2}$.  We let the size of the circles reflect the (hypothetical) importance weights of sequences $\bs_{1:t}^{\barbom_{t}^k}$ and $\omega_t^k$ reflect the (hypothetical) results of resampling with these weights.   In $(b)$, we assume fixed $j_{T+1}= j_3 = 1$ as in the text, with $\omega_2^1 = 2$.}
\label{fig:smc_bounds_app}
\vspace*{-.3cm}
\end{minipage}
\hfill
\begin{minipage}{.61\textwidth}
\begin{algorithm}[H]
\caption{(Twisted) SMC Target Sampling ($\smctgt$)\\(blue indicates changes from SMC proposal algorithm; $\ssb_{1:T}$ is an exact posterior sample)}
  \begin{algorithmic}%
    \STATE \textbf{SMC-TARGET}$\left(\base,q,\left\{ \psi_{t}\right\} _{t=1}^{T-1},\phi,K, \left\{ t_{r}\right\} _{t=1}^{R-1}, t_{0}=0,t_{R}=T,  {\color{blue}\ssb_{1:T}}\right)$:{\color{blue} 
    \STATE Initialize $j\sim\mathtt{uniform}\left(\left\{ 1,\ldots,K\right\}\right)$}
    \FOR{$t = 1,...,T$} 
        \FOR{$k = 1,...,K$}
            {\color{blue} 
            \IF{$k=j$}
                \STATE $s_{t}^{k}\leftarrow\ss_{t}$
            \ELSE
                {\normalcolor \STATE Sample $s_{t}^{k}\sim q\pv{s_{t}}{\s_{1:t-1}^{k}}$}
            \ENDIF}             
            \STATE $\s_{1:t}^{k}\leftarrow\mathtt{concat}\left(\s_{1:t-1}^{k},s_{t}^{k}\right)$
            \IF{$t < T$}
                \STATE $w_t^k \leftarrow \fr{\base\pv{s_{t}^{k}}{\bs_{1:t-1}^{k}}}{q\pv{s_{t}^{k}}{\bs_{1:t-1}^{k}}}\fr{\psi_{t}\left(\bs_{1:t}^{k}\right)}{\psi_{t-1}\left(\bs_{1:t-1}^{k}\right)}$ 
            \ELSE
                \STATE $w_t^k \leftarrow \fr{\base\pv{s_{t}^{k}}{\bs_{1:t-1}^{k}}}{q\pv{s_{t}^{k}}{\bs_{1:t-1}^{k}}}\fr{\phi\left(\bs_{1:t}^{k}\right)}{\psi_{t-1}\left(\bs_{1:t-1}^{k}\right)}$             
            \ENDIF            
        \ENDFOR
        \IF{$t \in \left\{ t_{r}\right\} _{r=1}^{R-1}$}{\color{blue}
            \STATE Sample $j\sim\mathtt{uniform}\left(\left\{ 1,\ldots,K\right\}\right)$}
            \FOR{$k = 1,...,K$}{\color{blue}
                \IF{$k=j$}
                    \STATE $\s_{1:t}^{k}\leftarrow\ssb_{1:t}$
                \ELSE
                    {\normalcolor \STATE  $\om_{t}^{k}\sim\mathtt{cat}\left({\bigg\{ {\fr{  \prod_{t = t_{r-1}+1}^{t_r} w_{t}^{i}}{\sum_{j=1}^{K}\prod_{t = t_{r-1}+1}^{t_r} w_{t}^{j}}}\bigg\} _{i=1}^{K}}\right)$
                    \STATE $\s_{1:t}^{k}\leftarrow\s_{1:t}^{\om_{t}^{k}}$}
                \ENDIF}    
            \ENDFOR     
        \ENDIF
    \ENDFOR
    \STATE \textbf{return}  {$\left\{ {\normalcolor \s_{1:T}^{k},  \prod_{t = t_{R-1}+1}^{T} w_{t}^{k}}\right\} _{k=1}^{K}$} \\
    \STATE \phantom{\textbf{return}} $\check{\cZ}^{\textsc{smc}}_{\sigma} =  \prod_{r=1}^R \frac{1}{K} \sum_{k=1}^K \prod_{t = t_{r-1}+1}^{t_r} w_t^k$
\label{alg:smc_target}
\end{algorithmic}
\end{algorithm}%
\end{minipage} 
\end{figure*}

\vheader
\paragraph{Extended State Space Proposal Distribution}
Sampling from the extended state space proposal corresponds to the procedure described in \cref{sec:smc_bg} and 
Alg. 1,
which we write as\footnote{Note that $\bom_{t}^{k}$, $\bs_{1:t}^{\barbom_{t}^{k}}$, and $\bs_{1:t}^{\bom_{t}^{k}}$ are deterministically constructed from $\{ s_t^k, \omega_t^k \}_{t=1,k=1}^{T,K} $ during sampling, and simply track the quantities to be calculated when evaluating densities.}
\begin{align}
\smcprop\left(\{ s_t^k, \omega_t^k \}_{t=1,k=1}^{T,K} \right) &\coloneqq
\prod_{t=1}^{T}\bh{\prod_{k=1}^K
q\pv{s_{t}^{k}}{\bs_{1:t-1}^{\bom_{t-1}^{k}}}\prod_{k=1}^K
q\pv{\om_{t}^{k}}{\s_{1:t}^{1:K}} } \label{eq:smc_prop} \tag{SMC Extended Proposal} \\
\text{where  $\forall ~ k$, }  ~~~ & q\pv{\om_{t}^{k}=i}{\s_{1:t}^{1:K}} \coloneqq \frac{\fr{\tilde{\pi}_{t}\left(\bs_{1:t}^{\barbom_{t}^{i}}\right)}{\tilde{\pi}_{t-1}\left(\bs_{1:t-1}^{\bom_{t-1}^{i}}\right)q\pv{s_{t}^{i}}{\bs_{1:t-1}^{\bom_{t-1}^{i}} }}}{\sum \limits_{\kappa=1}^{K}~\fr{\tilde{\pi}_{t}\left(\bs_{1:t}^{\barbom_{t}^{\kappa}}\right)}{\tilde{\pi}_{t-1}\left(\bs_{1:t-1}^{\bom_{t-1}^{\kappa}}\right)q\pv{s_{t}^{\kappa}}{\bs_{1:t-1}^{\bom_{t-1}^{\kappa}} }}} =\fr{w_{t}\left(\bs_{1:t}^{\barbom_{t}^{i}}\right)}{\sum_{\kappa=1}^{K}w_{t}\left(\bs_{1:t}^{\barbom_{t}^{\kappa}}\right)} \label{eq:omega_weights}
\end{align}
To recount the description above, note that the next token $s_t^i$ in index $i$ is sampled from the proposal, conditioned on the prefix $\bs_{1:t-1}^{\bom_{t-1}^{i}}$. 
We concatenate these tokens $\bs_{1:t}^{\barbom_{t}^{i}}=\bs_{1:t-1}^{\bom_{t-1}^{i}} + \listl s_t^i \listr $ ( \cref{eq:sequence}) and calculate importance weights.   We perform  resampling in each index $k$ according to
$q(\omega_t^k|\bs_{1:t}^{1:K})$, or \gls{SNIS} with the calculated weights (as in \cref{eq:incremental_resample}).
Finally, after resampling, we clone the sequence in the chosen index $\omega_t^k$ into index $k$ and proceed to sample $s_{t+1}^k$ with an prefix defined by the indices $\bom_{t}^{k} = \bom_{t-1}^{\omega_t^k } + \listl \omega_t^k \listr$.   

\textit{Worked Example: } To make this more concrete, we provide a worked example of the procedure in \cref{fig:smc_bounds_app} (a).  At step $t =1$, we 
resample the token $s_{t=1}^{k=2}$ twice (for indices $k=1,3$), with $\omega_1^1=\omega_1^3=2$ (and in index $2$, set $\omega_1^2 = 3$ to sample $s_1^3$).
We record the prefix history as, for example, $\bom_1^{1}= \bom_1^{3} = \listl \omega_1^1 \listr = \listl 2\listr $, which corresponds to $\bs_{1}^{\bom_1^{1}} = s_1^2$.

At step 2 in (a), we proceed to sample $s_{2}^{1} \sim q(s_{2}|\bs_{1}^{\bom_1^{1}} = \listl s_{1}^{2}\listr)$ (and similarly $s_2^3 \sim q(s_2|\bs_{1}^{\bom_1^{3}} = \listl s_1^2\listr )$), whereas $s_2^2 \sim q(s_2|\bs_{1}^{\bom_1^{1}} = \listl s_1^3\listr )$.  
We next evaluate the importance weights over three concatenated sequences: $\bs_{1}^{\barbom_1^{1}} = \listl s_1^2\listr+\listl s_2^1\listr $, $\bs_{1}^{\barbom_1^{2}} = \listl s_1^3\listr+\listl s_2^2\listr$, and $\bs_{1}^{\barbom_1^{3}} = \listl s_1^2\listr+\listl s_2^3\listr $, emphasizing that $s_2^k$ is the final token in each index.   Shown in the red circles, we proceed to resample $\omega_2^1 = 2, \omega_2^2=3, $ and $ \omega_2^3=2$ at step $t=2$.

Finally, we need to backtrack to obtain the history of the indices for the sequence 
to be cloned in resampling.  Namely, for index $1$ where $\omega_{t=2}^{k=1} = 2$,  %
we concatenate $\bom_{1}^{\omega_{2}^{1}} + \listl\omega_{2}^{1}\listr  = \bom_{1}^{2} + \listl2\listr = \listl 3, 2\listr \eqqcolon \bom_{2}^{1}$ (i.e. the history for  time 2, index 1).   This list of indices specifies the prefix $\bs_{1:2}^{\bom_{2}^{1}} = \listl s_1^3, s_2^2 \listr $ at step $t=3$, index $k=1$.
Similar reasoning applies for other indices.

\vheader
\paragraph{Extended State Space Target}
We are finally ready to specify the extended state space target distribution.   The crucial difference is to identify a single sequence $\bs_{1:T}^{\bom_{T}^{1}}$ of length $T$ (the choice of index 1 is arbitrary).   This sequence $\bs_{1:T}^{\bom_{T}^{1}}$ will be evaluated under the unnormalized target distribution $\tpitwist{T}(\s_{1:T}) = \tilde{\sigma}(\s_{1:T})$ or exactly sampled from the target $\bs_{1:T}^{\bom_{T}^{1}} \sim \sigma(\bs_{1:T})$ in the extended state space target distribution.

In particular, we begin by sampling a {full} sequence of indices $\{j_t\}_{t=1}^T$ uniformly at random
$\text{Pr}(  j_1, j_2, ... j_T ) = (1/K)^T$.  
Setting $\omega_T^1 = j_T$, we let $\omega_{t-1}^{j_{t}} = j_{t-1}$ for all $t$.   This implies the following, 
\begin{align}
\omega_T^1 = j_T, ~ \omega_{t-1}^{j_{t}} = j_{t-1}  \quad \implies \quad 
\bom_{T}^{1} &= \listl  j_1, j_2, ... j_T\listr, \qquad  \bom_{t-1}^{j_t} = \listl  j_1, j_2, ... j_{t-1} \listr, \label{eq:indices_special}   \\[1.5ex]
\text{and} \qquad 
\barbom_{t}^{j_t} &= \bom_{t}^{j_{t+1}} \label{eq:omega_arithmetic}
\end{align}
To show these identities, note that $\omega_{t-1}^{j_{t}} = j_{t-1}$ and \cref{eq:index} imply $\bom_{t-1}^{j_t} = \bom_{t-2}^{\omega_{t-1}^{j_t}} +\listl \omega_{t-1}^{j_t}\listr = \bom_{t-2}^{j_{t-1}} +\listl j_{t-1} \listr = \barbom_{t-1}^{j_{t-1}}$, which matches \cref{eq:omega_arithmetic}.  Applying this recursion again yields $\bom_{t-1}^{j_t} =  \bom_{t-3}^{j_{t-2}} + \listl j_{t-2} , j_{t-1} \listr ... = \listl j_1, j_2, ... j_{t-1} \listr$.   
Taken together, these notations allow us to interleave a true target sample in particular indices $\{ j_t \}$, guaranteeing that 
at least one target samples appears at each step. 

The extended state space target distribution differs from $\smcprop$ 
in its handling of this sequence, 
which identified as $\bs_{1:T}^{\bom_{T}^{1}}$ with prefixes $\bs_{1:t-1}^{\bom_{t-1}^{j_{t}}}$ using \cref{eq:indices_special}.   Noting that sampling $\{j_t\}_{t=1}^T$ amounts to specifying a particular set of $\omega_t^{k}$ as in \cref{eq:indices_special}-(\ref{eq:omega_arithmetic}), 
\begin{align}
\hspace*{-.2cm}
\tsmctgt\left(\{ s_t^k, \omega_t^k \}_{t=1,k=1}^{T,K} \right) =
\underbrace{
\text{Pr}( j_1, j_2, ... j_T ) 
}_{\left(\fr 1K\right)^{T}}~  
\tilde{\pi}_{T}\left(\s_{1:T}^{
\bom_{T}^{1}
}\right)
\prod_{t=1}^{T}\bh{\prod_{\substack{k=1\\
k\ne j_{t}
}
}^{K}
q\pv{s_{t}^{k}}{\bs_{1:t-1}^{\bom_{t-1}^{k}}}
\prod_{\substack{k=1\\
k\ne j_{t+1}
}
}^{K}q\pv{\om_{t}^{k}}{\s_{1:t}^{1:K}} } .
\raisetag{5pt}
\tag{SMC Extended Target}\label{eq:smc_ess_tgt}
\end{align}
 Note, the normalization constant
of $\tsmctgt(\boldsymbol{S})$ is equal to $\cZ_\sigma$ since only $\tilde{\pi}_T(\bs_{1:T}) = \tilde{\sigma}(\bs_{1:T})%
$ is unnormalized.

To describe ancestral sampling from \cref{eq:smc_ess_tgt},
we first sample $\{j_t\}_{t=1}^T$ uniformly as above, and place an exact target sequence in indices $\bs_{1:T}^{\bom_{T}^{1}}$ (or, equivalently, sequentially sample $s_t^{j_t} \sim \pi_t(s_t | \bs_{1:t-1}^{\bom_{t-1}^{j_t}})$.
At each step, the remaining $K-1$ indices $k \neq j_t$ are sampled from the proposal.    For resampling, we fix index $j_t$ to hold the exact sample and resample the remaining $K-1$ indices.    Note that the resampling weights $q\pv{\om_{t}^{k}}{\s_{1:t}^{1:K}} $ in \cref{eq:omega_weights} \textit{include} the exact sample, which may be cloned additional times into indices other than $j_t$ if its importance weights are high.  The procedure above 
simply ensures that \textit{at least} one exact sequence is sampled. See \cref{alg:smc_target} for the pseudocode of the algorithm. 

Note that \citet[Alg. 2]{maddison2017filtering} presents a different SMC extended state space target distribution than ours. In their work, $j_1=1$ and they sample $\mathbf{j}_{2:T+1}$, while in ours $j_{T+1}=1$ and we sample $\mathbf{j}_{1:T}$. However, both targets result in the same log partition function bounds.

\textit{Worked Example:}   In \cref{fig:is_smc} (c), we use blue circles and arrows to highlight the exact-sample indices $\bom_{T}^{1} = \listl j_1, j_2 \listr = \listl 3, 2 \listr $ and the target sequence $\bs_{1:T}^{\bom_{T}^{1}} = \listl s_1^3, s_2^2 \listr$.   %
Using the recursion $\omega_{t-1}^{j_t} = j_{t-1}$ with $j_{T+1} = j_3 = 1$ fixed, we may also express $\bom_{T}^{1} = \listl j_1, j_2 \listr = \listl 3, 2 \listr = \listl \omega_1^{2}, \omega_2^{1} \listr$.   At step 2, note the target sequence is sampled/evaluated an additional time in index 3. %

\paragraph{Importance Weights in the Extended State Space}

Assume we are given a fixed set of $\{ s_t^k, \omega_t^k \}_{t=1,k=1}^{T,K}$, which may be sampled from either $\smctgt(\boldsymbol{S})$ or $\smcprop(\boldsymbol{S})$.  
We proceed to show that the unnormalized importance weights in the extended state space simplify as follows.

\begin{lemma}\label{lemma:smc_weights}
For the extended state space target $\tsmctgt$ and proposal $\smcprop$ above, the simple importance weights in the extended state space become
\begin{align}
\frac{\tsmctgt}{\smcprop}\left(\{ s_t^k, \omega_t^k \}_{t=1,k=1}^{T,K} \right) &= \prod_{t=1}^{T} \frac{1}{K}  \sum \limits_{k=1}^{K}~ \fr{\tilde{\pi}_{t}\left(\bs_{1:t}^{\barbom_{t}^{k}}\right)}{\tilde{\pi}_{t-1}\left(\bs_{1:t-1}^{\bom_{t-1}^{k}}\right)q\pv{s_{t}^{k}}{\bs_{1:t-1}^{\bom_{t-1}^{k}} }}
= \prod_{t=1}^{T} \frac{1}{K} \sum_{k=1}^{K} w_t\left(\bs_{1:t}^{\barbom_{t}^{k}}\right)
\eqqcolon \prod_{t=1}^{T} \frac{1}{K} \sum_{k=1}^{K} w_t(\bs_{1:t}^{k})
\end{align}
which can be used to obtain unbiased $\cZ_\sigma$ estimators (\cref{eq:smc_ess}) or bounds on $\log \cZ_\sigma$ (\cref{prop:smc_proposition}, with proof below).
\end{lemma}
\begin{proof}
To evaluate the importance weights (with the goal of estimating $\cZ_\sigma$), we consider
\small 
\begin{align}
\frac{\tsmctgt}{\smcprop}\left(\{ s_t^k, \omega_t^k \}_{t=1,k=1}^{T,K} \right) &= \frac{{\left(\fr 1K\right)^{T}}~  
\tilde{\pi}_{T}\left(\s_{1:T}^{
\bom_{T}^{1}
}\right)
\prod_{t=1}^{T}\bh{\prod_{\substack{k=1\\
k\ne j_{t}
}
}^{K}
q\pv{s_{t}^{k}}{\bs_{1:t-1}^{\bom_{t-1}^{k}}}
\prod_{\substack{k=1\\
k\ne j_{t+1}
}
}^{K}q\pv{\om_{t}^{k}}{\s_{1:t}^{1:K}} } }{
\prod_{t=1}^{T}\bh{\prod_{k=1}^K
q\pv{s_{t}^{k}}{\bs_{1:t-1}^{\bom_{t-1}^{k}}}\prod_{k=1}^K
q\pv{\om_{t}^{k}}{\s_{1:t}^{1:K}} }
} \\
&\overset{(1)}{=} {\left(\fr 1K\right)^{T}}~  
\tilde{\pi}_{T}\left(\s_{1:T}^{
\bom_{T}^{1}
}\right)
\prod_{t=1}^{T}  \frac{1}{
{
q\pv{s_{t}^{j_t}}{\bs_{1:t-1}^{\bom_{t-1}^{j_t}}}
q\pv{\om_{t}^{j_{t+1}}}{\s_{1:t}^{1:K}} }
} \\
\intertext{\normalsize  where in $(1)$, note that terms in the denominator cancel except for the indices $\listl0, j_1, ... j_T\listr=\bom_{T}^{1}$.  
Recalling that $\omega_t^{j_{t+1}} = j_t$ from \cref{eq:omega_arithmetic}, 
we expand the resampling weights $q(j_t |\bs_{1:t}^{1:K})$ for the sequence indexed by $s_t^{j_t}$, $\bs_{1:t-1}^{\bom_{t-1}^{j_t}}$, and $\bs_{1:t-1}^{\barbom_{t}^{j_t}}$,
\small }
&\overset{(2)}{=} {\left(\fr 1K\right)^{T}}~  
\tilde{\pi}_{T}\left(\s_{1:T}^{
\bom_{T}^{1}
}\right)
\prod_{t=1}^{T}  \frac{\sum \limits_{k=1}^{K}~\fr{\tilde{\pi}_{t}\left(\bs_{1:t}^{\barbom_{t}^{k}}\right)}{\tilde{\pi}_{t-1}\left(\bs_{1:t-1}^{\bom_{t-1}^{k}}\right)q\pv{s_{t}^{k}}{\bs_{1:t-1}^{\bom_{t-1}^{k}} }}}{
{
\cancel{q\pv{s_{t}^{j_t}}{\bs_{1:t-1}^{\bom_{t-1}^{j_t}}}}
\fr{\tilde{\pi}_{t}\left(\bs_{1:t}^{\barbom_{t}^{j_t}}\right)}{\tilde{\pi}_{t-1}\left(\bs_{1:t-1}^{\bom_{t-1}^{j_t}}\right) \cancel{q\pv{s_{t}^{j_t}}{\bs_{1:t-1}^{\bom_{t-1}^{j_t}} }}}
}} %
\label{eq:final_step_smc_w}
\end{align}
\normalsize
Finally, we obtain a telescoping cancellation of $\tpi_t$ terms using the indexing identities in \cref{eq:indices_special}-(\ref{eq:omega_arithmetic}).   In particular, since 
$\barbom_{t}^{j_t}= \bom_t^{j_{t+1}}$ and $\barbom_{t-1}^{j_{t-1}}= \bom_{t-1}^{j_{t}}$ with $ \barbom_T^{j_T} = \bom_T^1$, we can simplify the terms in \cref{eq:final_step_smc_w} as
\scriptsize
\begin{align*}
\tpi_T\left(\bs_{1:T}^{\bom_T^1}\right) \prod_{t=1}^T \frac{\tpi_{t-1}\left(\bs_{1:t-1}^{\bom_{t-1}^{j_t}}\right)}{\tpi_t\left(\bs_{1:t}^{\barbom_{t}^{j_t}}\right)} = 
\tpi_T\left(\bs_{1:T}^{\barbom_T^{j_T}}\right)  \prod_{t=1}^T \frac{\tpi_{t-1}\left(\bs_{1:t-1}^{\barbom_{t-1}^{j_{t-1}}}\right)}{\tpi_t\left(\bs_{1:t}^{\barbom_{t}^{j_t}}\right)}  = 
\Cancel{ \tpi_T\left(\bs_{1:T}^{\barbom_T^{j_T}}\right)} \frac{\cancel{ \tpi_{T-1}\left(\bs_{1:T-1}^{\barbom_{T-1}^{j_{T-1}}}\right)} }{\Cancel{\tpi_T\left(\bs_{1:T}^{\barbom_T^{j_T}}\right)}}\frac{\ccancel{\tpi_{T-2}\left(\bs_{1:T-2}^{\barbom_{T-2}^{j_{T-2}}}\right) }}{\cancel{\tpi_{T-1}\left(\bs_{1:T-1}^{\barbom_{T-1}^{j_{T-1}}}\right)}}~...~\frac{1}{\Cancel{\tpi_{1}\left(\bs_{1:1}^{\barbom_{1}^{j_{1}}}\right) }} = 1 \nonumber
\end{align*}
\normalsize
using the assumption that $\tilde{\pi}_{0}(\cdot)=1$.
Simplifying from \cref{eq:final_step_smc_w}, the final unnormalized importance weights become
\small 
\begin{align}
\frac{\tsmctgt}{\smcprop}\left(\{ s_t^k, \omega_t^k \}_{t=1,k=1}^{T,K} \right) &= \prod_{t=1}^{T} \frac{1}{K}  \sum \limits_{k=1}^{K}~ \fr{\tilde{\pi}_{t}\left(\bs_{1:t}^{\barbom_{t}^{k}}\right)}{\tilde{\pi}_{t-1}\left(\bs_{1:t-1}^{\bom_{t-1}^{k}}\right)q\pv{s_{t}^{k}}{\bs_{1:t-1}^{\bom_{t-1}^{k}} }} 
= \prod_{t=1}^{T} \frac{1}{K} \sum_{k=1}^{K} w_t\left(\bs_{1:t}^{\barbom_{t}^{k}}\right)
\eqqcolon \prod_{t=1}^{T} \frac{1}{K} \sum_{k=1}^{K} w_t(\bs_{1:t}^{k}) \qquad \,\,
\end{align}
\normalsize
as desired, where we abbreviate the importance weights as $w_t(\bs_{1:t}^{k})$ for simplicity of notation.
Note that we also obtain an unbiased estimate of the partition function via
\small 
\[
\Z_{\si}=\EEE_{\smcprop(\boldsymbol{S})}\left[ \frac{\tsmctgt(\boldsymbol{S})}{\smcprop(\boldsymbol{S})}\right]  =  \EEE_{\smcprop(\boldsymbol{S})}\left[\prod_{t=1}^{T}\fr 1K\sum_{k=1}^{K}w_{t}\left(\s_{1:t}^{k}\right)\right]
\]
\normalsize
\end{proof}

\vheader
\bdsmc*
\begin{proof}
The proof follows directly from \citet{brekelmans2022improving} App. A, where it is shown that for 
$\sigma_{\text{ext}}(\boldsymbol{S}), \prop_{\text{ext}}(\boldsymbol{S})$ such that $\cZ_\sigma = \mathbb{E}_{\prop_{\text{ext}}(\boldsymbol{S})}[\frac{\tilde{\sigma}_{\text{ext}}(\boldsymbol{S})}{\prop_{\text{ext}}(\boldsymbol{S})}]$, we can construct lower and upper bounds on $\log \cZ_\sigma$ 
\begin{align}
\DKL\pV{\prop_{\text{ext}}(\boldsymbol{S})}{\sigma_{\text{ext}}(\boldsymbol{S})}+\EEE_{\prop_{\text{ext}}(\boldsymbol{S})}\left[ \log \frac{\tilde{\sigma}_{\text{ext}}(\boldsymbol{S})}{\prop_{\text{ext}}(\boldsymbol{S})} \right]= & \log\Z_{\si}=\EEE_{\sigma_{\text{ext}}(\boldsymbol{S})}\left[\log \frac{\tilde{\sigma}_{\text{ext}}(\boldsymbol{S})}{\prop_{\text{ext}}(\boldsymbol{S})}\right]-\DKL\pV{\sigma_{\text{ext}}(\boldsymbol{S})}{\prop_{\text{ext}}(\boldsymbol{S})} \\
\EEE_{\prop_{\text{ext}}(\boldsymbol{S})}\left[ \log \frac{\tilde{\sigma}_{\text{ext}}(\boldsymbol{S})}{\prop_{\text{ext}}(\boldsymbol{S})} \right]\leq & \log\Z_{\si} \leq \EEE_{\sigma_{\text{ext}}(\boldsymbol{S})}\left[\log \frac{\tilde{\sigma}_{\text{ext}}(\boldsymbol{S})}{\prop_{\text{ext}}(\boldsymbol{S})}\right] \label{eq:logz_ub_lb}
\end{align}
where the gap in the lower and upper bounds are $\DKL\pV{\prop_{\text{ext}}(\boldsymbol{S})}{\sigma_{\text{ext}}(\boldsymbol{S})}$ and $\DKL\pV{\sigma_{\text{ext}}(\boldsymbol{S})}{\prop_{\text{ext}}(\boldsymbol{S})}$, respectively.

Substituting our SMC probabilistic interpretation in \cref{eq:smc_prop} and \cref{eq:smc_ess_tgt}, along with the importance weights in \cref{lemma:smc_weights}, into \cref{eq:logz_ub_lb} yields the desired bounds in \cref{eq:smc_lb}.
\end{proof}

\vheader 
\paragraph{IWAE as a Special Case of our SMC Probabilistic Interpretation} Note that we recover \gls{IWAE} (or SIS over $K$ samples) from SMC with no intermediate resampling.   In particular, this corresponds to $\omega_t^k = k$ for all $t < T$, with importance weighting from resampling occurring at the final step $\prod_{k=1}^K q(\omega_T^k | \bs_{1:T}^{1:K})$.   This yields the $1/K$ average inside the log in the IWAE bounds (i.e., SMC with only one resampling step at $t=T$).  While the importance weights are crucial to construct the bound, note that `resampling' is not necessary at the final step and we may return all $K$ samples along with their weights.

Viewing \gls{IWAE} as a special case of our \gls{SMC} probabilistic interpretation is complementary to the interpretations in \citet{domke2018importance,brekelmans2022improving} and also provides upper bounds \citep{sobolev2019importance}.

\vheader
\section{Additional Experiment Details}
\vheader
\subsection{Common Details Across Experiments}
\label{sec:common_details}
\vheader

For all experiments, we use the Adam optimizer with $\beta_1, \beta_2 = \{0.9, 0.999\}$. 
We use custom implementations of SMC.
For PPO, we use the HuggingFace TRL PPO Trainer (\url{https://github.com/huggingface/trl/blob/main/trl/trainer/ppo_trainer.py}), modified slightly to accomodate our custom twist parameterizations, as described below. For other methods, we use Optax (Flax) and custom loss functions. We use HuggingFace models (\url{https://huggingface.co/models}) for the base $\base$ models and build custom layers on top of those. 

For the twist $\vartwist{t}(\bs_{1:t})$, we always parameterize $\log \vartwist{t}(\bs_{1:t})$ for numerical stability.
We choose random normal initializations centered at mean 0, with low variance,\footnote{We specifically use a form of Xavier initialization, taking the variance as ${\frac{2}{n_{\text{inputs}}  + n_{\text{outputs}}}}$.
} such that $\log \vartwist{t}(\bs_{1:t}) \approx 0, \vartwist{t}(\bs_{1:t}) \approx 1$ at the beginning of training, which means the initial sequences generated by the twist-induced proposal approximately come from the base model $\base$. All methods are initialized using the same random seeds, and thus start from the same parameter values. 
See \cref{sec:twist_parameterization} for additional discussion of choices for the twist parameterization.

For methods that directly learn a proposal (DPG and PPO), we could directly finetune a language model that outputs $q(\bs_{1:t})$. However, in order to ensure consistency in terms of model capacity and ease of learning compared to our twisted proposals, we instead have these proposal learning methods output a modifier $\log \vartwist{t}(\bs_{1:t})$ which is added to the base model log probability $\log \base(\bs_{1:t})$. Note that using random normal initializations centered at mean 0 with low variance, this scheme results in initial $q$ samples coming approximately from $\base$. %

For methods that can make use of exact posterior samples, when we have access to them (\cref{sec:infilling}, \cref{sec:exact_vs_appr}), we use them. This is straightforward for methods like DPG, SIXO, and our CTL (unless we have only a single sample, as we discuss for infilling in %
\cref{sec:details}
). For our RL twist learning, we found the best empirical performance training on a combination of $q$ and exact $\sigma$ samples when they were available (as opposed to just $q$ otherwise), and use those results. Similarly, for FUDGE, when exact $\sigma$ samples are available, we use them together with $\base$ samples.  

It is not straightforward to compare PPO versus other methods, because of the inner loop in PPO that repeats several clipped gradient steps on a given set of samples. This means that, for a constant number of samples, PPO makes more gradient updates than other methods, while for a constant number of gradient updates, PPO sees fewer samples. Ultimately we decided to normalize based on the number of samples seen; we consider each outer step (including a full PPO inner loop, in our experiments, 4 gradient steps) as a single ``gradient update.''  We make this choice  since sampling is the main bottleneck in terms of computational cost, and the number of inner PPO steps is a hyperparameter which we did not tune.

All of our experiments were run on a single GPU, usually on an NVIDIA A40 with 48G memory. All experiments took no longer than 9 wall-clock hours to run for a single learning method, with infilling (\cref{sec:infilling}) experiments taking longest; most other experiments took no longer than 4 hours.

\vheader 
\subsection{Choices of Twist Parameterization}
\label{sec:twist_parameterization}
\vheader 

The choice of parameterization for the twist $\log\vartwist{t}(\bs_{1:t})$ is a design decision, independent of our overall framework. While one could keep an entirely separate model for each $\log\vartwist{t}(\bs_{1:t})$, this 
is likely to be
memory-inefficient and 
learn slowly. Instead, we use a shared parameterization across $\bs_{1:t}$, in the same way that the base language model uses a single architecture to output probability distributions over tokens at each time step $t$. We lay out parameterization choices we considered below. 

\vheader 
\subsubsection{Linear Head}
\label{param:linear}
\vheader 

The simplest choice is to replace the linear head of the base language model with a new linear head, keep the base model fixed, and only train the linear head. This parameterization incurs very little additional computation cost compared to just using the base language model.   However, we found this to be 
capacity constrained in our experiments, achieving worse KL divergences than other parameterizations.

\vheader 
\subsubsection{MLP Head}
\label{param:NN}
\vheader 

Instead of a linear head, we consider a 3-layer fully connected neural network (MLP) with ReLU non-linearities as a head on top of the base language model. The base model is still kept fixed; only the MLP head is trained. This incurs more computational cost than a linear head (\cref{param:linear}), but the additional cost is still small 
relative to the cost of a forward pass through the base transformer model. We found this to generally perform well in our experiments, so we use it for the toxicity threshold experiment in \cref{sec:toxthresh} and sentiment in \cref{sec:sent}.

\vheader 
\subsubsection{Separate Transformer for the Twist}
\label{param:separate_linear}
\vheader 

We can also consider an entirely separate transformer that outputs only the twist value. That is, we copy the base model, and repurpose it to output a twist value $\log \vartwist{t}(\bs_{1:t})$ instead of logits for next-token probabilities. We then train the entire network end-to-end. This is significantly more computationally costly than the former approaches, and does not always do better than just an MLP head (\cref{param:NN}), so we generally do not recommend using this. Still, we found it to perform well in toxicity classification in \cref{sec:toxclass}, so we use it there.

\vheader 
\subsubsection{Separate Transformer for the Twist, with MLP Head}
\label{param:separate_NN}
\vheader 

This is similar to \cref{param:separate_linear}, except we also replace the final linear head with a MLP head as in \cref{param:NN}. The model
outputs $\log \vartwist{t}(\bs_{1:t})$ and is trained end-to-end. This is the most computationally costly approach outlined here, and is unnecessary for most of our settings. However, in infilling with 15 generated tokens (\cref{sec:infilling}) we found this parameterization to perform materially better than all others, particularly with DPG (\cref{app:dpg}), so we use it for all infilling experiments.

With both this parameterization and \cref{param:separate_linear}, we increase computation time by a factor of around 2 on the forward pass, and significantly increase memory and time usage on the backwards pass during training (though sampling is still the main time bottleneck). Whether this parameterization is worth the potential gain in performance depends on the desired use case. We emphasize that our overall framework is independent of the choice of parameterization. 

\vheader 
\subsection{Comments on Our Choices of Experiment Settings}
\label{app:expmt_choices}
\vheader 
Our settings and evaluation metrics in \cref{sec:experiments} are chosen to highlight our scientific findings.   
In particular, the toxicity threshold experiment in \cref{sec:toxthresh} demonstrates the improvement of SMC over SIS with the base model with \gls{CTL} learned twists.  In order to highlight this distinction, we have chosen a setting where it is \textit{extremely} difficult to draw samples satisfying the threshold using the base model $\base$ (see SIS/IWAE LB line in \cref{fig:toxthresh}).

However, twist-learning in the toxicity threshold setting presents challenges.   For approximate positive sampling and a thresholded target, all importance weights will be 0 if none of our $K$ samples meet the threshold.   As noted above, sampling from $\base$, 
or the SMC/twisted proposal for $\vartwist{t}(\bs_{1:t}) \approx 1$ at initialization, is extremely unlikely to draw samples meeting the threshold (i.e., within the support of the target) in the setting of \cref{sec:toxthresh}.    
As a result, initial iterations of twist learning receive no learning signal until a thresholded positive sample is drawn from the base model.

To avoid this difficulty for baselines comparisons in \cref{sec:experiment_kls}, we 
instead focused on settings with $\finaltwist$ given by probabilities.
Nevertheless, we note that there are no fundamental differences between the settings considered in \cref{sec:toxthresh} and \cref{sec:experiment_kls}.  Thus, we may also evaluate single-sample 
$\DKL\pV{\sigma}{q}$ and $\DKL\pV{q}{\sigma}$ in the setting of \cref{sec:toxthresh}, or plot $\log \cZ_{\sigma}$ bounds as a function of $K$ in for the settings in \cref{sec:experiment_kls}.

\subsection{Experiment-Specific Details}\label{sec:details}
\paragraph{Details for SIS and SMC Comparison (\cref{sec:toxthresh})}
We generate 10 output tokens, and train twists using \cref{sec:ebm_roger} with approximate positive sampling as discussed in \cref{sec:positive_ebm}.

Note that using $\sigma(\bs_{1:T}) \propto p_0(\bs_{1:T})  \mathbb{I}[\bs_{1:T} \in \cC]  $ where $\cC \coloneqq \{\bs_{1:T} ~ | r(\bs_{1:T}) \leq {\eta} \}$ directly runs into numerical issues for calculating 
$\log \zfinal$
when $\bs_{1:T} \notin \cC$ and $\mathbb{I}[\bs_{1:T} \in \cC] = 0$. 
We instead use $\epsilon + \mathbb{I}[\bs_{1:T} \in \cC]$ everywhere instead of $\mathbb{I}[\bs_{1:T} \in \cC]$, where $\epsilon = 10^{-16}$. In \cref{fig:toxthresh}, this yields a SIS/IWAE $\log \cZ_\sigma$ LB $\approx -36$ 
when no samples are drawn that fall in the set $\cC$.

We use an MLP head to parameterize the twist, as in \cref{param:NN}, with 768 hidden units per layer, matching the TinyStories model's embedding dimension. We use a batch size (number of SMC particles/samples) of 1000, with a learning rate of 0.0001, and train using CTL for a total of 5000 gradient updates. We did not tune hyperparameters because we found this setting to work well, and we are not comparing across different learning methods. 

For each point on each line on \cref{fig:toxthresh}, we run SIS or SMC 20 times, each with a different randomly selected true posterior sample for the upper bounds. The line shows the average value across these 20 runs, while the shaded area shows 95\% confidence intervals. 
See also \cref{sec:common_details} for details common across experiments.

\paragraph{Details for Toxicity (\cref{sec:toxclass})}
We generate 20 output tokens.
We parameterize the twist with a separate network as in \cref{param:separate_linear}. We use a batch size (number of SMC particles/samples) of 100, and train for a total of 2048 gradient updates. 
For each learning method, we used a coarse grid search over learning rates between 0.000001 and 0.001, using the best one found, which was usually 0.00003 or 0.0001.  
We run each learning method over 5 different random seeds, reporting the average KL divergence and 95\% confidence intervals over these 5 seeds.

For each KL divergence evaluation, we first get sandwich bounds on $\log \zfinal$ as laid out in \cref{sec:eval}, using the learned twists for the twisted proposal with $K=500$ samples. We find SIS/IWAE and SMC bounds to be similarly tight, so use SIS/IWAE for simplicity. We do this 4 times, providing 4 upper bound estimates and 4 lower bound estimates, and take the average midpoint as the $\log \zfinal$ estimate for each experiment. We then take the median (across all learning methods and seeds) of these estimates, and use that as our estimate of $\log \zfinal$. This is then used as a common value for the KL divergence across all methods and seeds, which controls for possible noise in $\log \zfinal$ bounds and ensures a fair comparison across methods. We generally have tight bounds (upper bound $\approx$ lower bound), which suggest our $\log \zfinal$ estimates are generally accurate, but note that any inaccuracies in estimating $\log \zfinal$ would only affect the absolute values of the KL divergences, not the relative differences among different learning methods.

We estimate expectations in \cref{eq:fwd_rev_KL_est} with 2000 samples from $\infevaldist$ and 2000 exact posterior samples for $\sigma$. With 2000 samples, our estimates have 95\% confidence intervals generally between 0.05 and 0.10, suggesting that our estimates of expectations are unlikely to be off by more than 0.10.  
The exact posterior samples were collected offline; such a large number of samples takes several hours to collect, and in practical settings, we would likely only be able to collect a much smaller number of samples. All our methods still apply with fewer exact posterior samples, but the variance in estimates will be higher.
See also \cref{sec:common_details} for details common across experiments.

\paragraph{Details for Sentiment (\cref{sec:sent})}
We generate 10 output tokens.
We parameterize the twist using an MLP head (\cref{param:NN}), with 1024 hidden units per layer, matching the GPT2Medium model's embedding dimension. Other details are the same as for toxicity above.
Collecting exact posterior samples is less time consuming in this case (less than an hour).
See \cref{sec:common_details} for common experimental details.

\paragraph{Details for Infilling (\cref{sec:infilling})}
We parameterize the twist using a separate transformer with an MLP head (\cref{param:separate_NN}), with 768 hidden units per layer (matching the TinyStories model's embedding dimension).
 We make the following adjustments to the forward pass of the language model for the conditional twist setting.
Instead of taking in only $\bs_{1:T}$, the model takes in both $\bs_{1:T}$ and $\bs_{T+1:T+c}$ and passes each separately through the body (everything except the head). Thus, $\bs_{T+1:T+c}$ can be seen as a second prompt. For $\bs_{T+1:T+c}$, we take the embeddings produced after the last conditioning token $s_{T+c}$ has been processed, broadcast it across time steps $1:T$, and pass that as additional input to the MLP head (concatenated with embeddings for $\bs_{1:T}$ at each $t \in 1...T$). This allows the MLP head to produce different output depending on the conditioning tokens.

Since we are in the conditional target distribution setting (\cref{sec:conditional_twist}), with  $o_T = \bs_{T+1:T+c}$, to compare across learning methods
using a single quantity,
we estimate
$\mathbb{E}_{o_T}\left[\DKL\pV{q_{o_T}}{\si_{o_T}}\right] := \E_{o_T} [\DKL\pV{\infevaldist(\bs_{1:T} | o_T)}{\sigma(\bs_{1:T} | o_T) 
}] 
$ and $\mathbb{E}_{o_T}\left[\DKL\pV{\si_{o_T}}{q_{o_T}}\right] := \E_{o_T} [\DKL\pV{\sigma(\bs_{1:T} | o_T)}{\infevaldist(\bs_{1:T} | o_T)}]
$
where $\E_{o_T}[\cdot] := \E_{\base(\bs_{T+1:T+c})}[\cdot] $ for %
infilling. 
Note that, %
\begin{align*}
\E_{o_T} [\DKL\pV{\infevaldist(\bs_{1:T} | o_T)}{\sigma(\bs_{1:T} | o_T) 
}] =\E_{o_T} \left[\EEE_{\infevaldist(\bs_{1:T} | o_T)}\left[\log\fr{\infevaldist(\bs_{1:T} | o_T)}{\base(\bs_{1:T})\phi(\bs_{1:T},o_T)}\right]\right] +{\E_{o_T} [\log\Z_{\si}(o_T)] }
\nonumber 
\\
\E_{o_T}[\DKL\pV{\sigma(\bs_{1:T}| o_T)
}{\infevaldist(\bs_{1:T}| o_T)}] =\E_{o_T}\left[\EEE_{\si(\bs_{1:T}| o_T)}\left[\log\fr{\base(\bs_{1:T} )\phi(\bs_{1:T}, o_T)}{\infevaldist(\bs_{1:T}| o_T)}\right]\right]-{\E_{o_T} [\log\Z_{\si}(o_T) ]}
\label{eq:aaa}
\end{align*} 
where for a fixed $o_T$, $\EEE_{\infevaldist(\bs_{1:T} | o_T)}\left[\log\fr{\infevaldist(\bs_{1:T} | o_T)}{\base(\bs_{1:T})\phi(\bs_{1:T},o_T)}\right]$
and $\EEE_{\si(\bs_{1:T}| o_T)}\left[\log\fr{\base(\bs_{1:T} )\phi(\bs_{1:T}, o_T)}{\infevaldist(\bs_{1:T}| o_T)}\right]$ may be evaluated as before, similar to the unconditional setting. In particular, for our experiments, we use 1-sample estimates of these expectations, as we have a single exact sample from $\si(\bs_{1:T}| o_T)$ by the BDMC trick (\cref{sec:conditional_twist}), and we choose to draw a single sample from the conditional proposal $\infevaldist(\bs_{1:T}| o_T)$. We average this over 2000 $o_T \sim \base(\bs_{T+1:T+c})$, approximating the outer expectation, giving us a 2000-sample estimate of 1-sample estimates for the first term in the right hand side of both equations above. With 2000 samples, our estimates have 95\% confidence intervals generally between 0.20 and 0.30.

Note that ${\E_{o_T} [\log\Z_{\si}(o_T) ]}$ is independent of the learning method or proposal $q$, unlike the first term we discussed above. Thus, in order to save computation and provide us with a more accurate estimate of ${\E_{o_T} [\log\Z_{\si}(o_T) ]}$, we estimate this term 
only once.
Specifically, we consider only the learning method with the lowest KL divergence (DPG), and use SIS/IWAE bounds. For each $o_T$, we estimate $\log\Z_{\si}(o_T) $ with $K=500$ samples, which gives us relatively tight sandwich bounds, again taking the midpoint as our estimate.
We average this over 1000 $o_T \sim \base(\bs_{T+1:T+c})$, giving us a 1000-sample estimate of ${\E_{o_T} [\log\Z_{\si}(o_T) ]}$, where each $\log\Z_{\si}(o_T)$ is itself estimated via 500 samples.

For negative sampling with contrastive twist learning (CTL) in this setting, we need 
at least 2 negative samples per set of conditioning tokens $o_T = \bs_{T+1:T+c}$ to perform \gls{SIS} reweighting; this is in contrast with other twist learning methods which can generate a single negative sample per $o_T$.
For the positive sample, we can use our single exact sample directly, or 
we can run the SMC upper bound sampling procedure (``Sampling from $\smctgt$ for SMC Upper Bounds'' section in \cref{sec:bi-SMC}) 
generate more approximate $\sigma$ samples using the given exact sample.   We find the latter to generally perform slightly better than the former, so adopt that for our infilling experiments.

We use a fixed batch size of 100 across all methods for training twists. To clarify the meaning of this batch size, for methods other than CTL, we have 100 draws of exact $\sigma$ samples, each for a different set of conditioning tokens $o_T = \bs_{T+1:T+c}$, so we train over 100 different $o_T$ at a time using 1 negative sample per $o_T$. For CTL, since we need at least 2 negative samples per $o_T$, 
we split the batch size of 100 across the number of different $o_T$ and the number of negative samples per $o_T$, as an additional hyperparameter. We use 25 $o_T$ with 4 negative samples per $o_T$ for the experiments in \cref{sec:infilling} and 10 $o_T$ with 10 negative samples per $o_T$ for the experiments in \cref{sec:infilling_2_1}. Controlling for batch size in this way is arguably disadvantageous for CTL compared to other learning methods, as it learns on a smaller number of $o_T$, but this controls for memory requirements, and we feel is more fair than controlling for the number of $o_T$ seen but allowing more negative samples for CTL relative to other methods.
We train for a total of 5500 gradient updates.
For each method, we used a coarse grid search over learning rates between 0.000001 and 0.001, using the best one found, which was usually 0.0001 or 0.00003. We run each learning method over 5 different random seeds, reporting the average KL divergence and 95\% confidence intervals over these 5 seeds.   
See also \cref{sec:common_details} for details common across experiments.

\vheader
\section{Additional Experimental Results}
\subsection{Qualitative Results}\label{app:qualitative}
\paragraph{Toxicity Controlled Generation when No Exact Posterior Samples are Available}
In \cref{sec:toxclass} we targeted $\sigma(\bs_{1:T}) \propto p_0(\bs_{1:T}) e^{\beta \log p(a|\bs_{1:T})}$ with $\beta = 1$. We can also target $\beta > 1$; higher $\beta$ produces a more peaked distribution of text that is more likely to be of class $a$. However, for $\beta \neq 1$ we can no longer generate exact posterior samples and thus cannot upper bound $\log \zfinal$. %
Our twist learning (\cref{sec:ebm_roger}) with approximate positive sampling (\cref{sec:positive_ebm}) can learn meaningful twists in this setting, which we illustrate with a qualitative example of a story (200 tokens upper limit) and $\beta = 10$:

\small
\textit{``Once upon a time, there was a little girl named Lily. She had a big thumb that she liked to suck on. One day, Lily went to the park to play with her friends. She was having so much fun until her thumb got stuck in her shoe. She tried to pull it out, but it hurt too much. \\
Lily started to cry and her friends tried to help her, but they couldn't get her thumb out either. She was scared and didn't know what to do. Her friends tried to help her, but they couldn't get it out either. Sadly, Lily had to go to the hospital and get a big bandage on her thumb. She couldn't play with her friends anymore. From that day on, Lily never went to the park again.''}
\normalsize

The story is coherent and follows the general style of the TinyStories base model, while having a high probability ($\approx$ 88\%) of being toxic according to the toxicity classifier, likely due to the presence of negative words such as `suck', `hurt', `cry', and `scared'. This supports the ability of our methods to control outputs based on the chosen posterior distribution.

\vheader 
\paragraph{Sentiment Controlled Generation when No Exact Posterior Samples are Available}
As above, 
we also consider $\sigma(\bs_{1:T}) \propto p_0(\bs_{1:T}) e^{\beta \log p(a|\bs_{1:T})}$, where $\beta > 1$, except now $p(a|\bs_{1:T})$ is based on the sentiment classifier in \cref{sec:sent}. 
In \cref{table:sentiment_beta100_results} we provide qualitative examples showing 20 tokens produced with twisted SMC with 500 particles, for $\beta = 100$, using twists trained with \cref{sec:ebm_roger}. 
These illustrate our framework's ability to learn reviews that embody each rating class.\footnote{The results are slightly incoherent; this is a result of the base GPT2-Medium model often being incoherent. Qualitatively, we find that these generations are more coherent than the uncontrolled ones from $\base$.}

\begin{table*}[t]
\vheader
\centering
\caption{\small Qualitative Results - Reviews Very Likely to be of a Particular Rating}
\vspace*{-.1cm}
\label{table:sentiment_beta100_results}
\resizebox{.8\textwidth}{!}{
\small 
\begin{tabular}{ll}
\toprule
  \textbf{\textbf{Class (Rating)}}   & \textbf{Most Text Generated Using Twisted SMC}  \\ \toprule
     1-star & ``I bought this sucker for my wife to use on her python that she sent me last year. It was terrible!''\\ \midrule
     2-star & ``I bought this throat raiser for combating dental caries. I didn't really like it. I didn't like''
     \\ \midrule
     3-star & ``I bought this a few months back, and I enjoyed it every time I held it. I'm giving 3 stars''
     \\ \midrule
     4-star & ``I bought this product a few months ago and have really enjoyed it. Only reason I gave it 4 stars is because''
     \\ \midrule
     5-star & ``I bought this phone recently, and I've been loving it! Gorgeous design, outstanding battery life, fantastic camera''
    \\ \bottomrule
\end{tabular}
}
\normalsize
\end{table*}

\begin{table*}[t]
\vheader
\caption{\small Qualitative Results - Infilling Examples }
\label{table:infilling_examples}
\vheader
\resizebox{\textwidth}{!}{
\small 
\begin{tabular}{llll}
\toprule
  \textbf{Proposal} & \textbf{Prompt ($\bs_{0}$)}   & \textbf{Generated Tokens ($\bs_{1:T}$)}  & \textbf{Conditioning Tokens ($\bs_{T+1:T+c}$)} \\ \toprule
     DPG & Once upon a time, there was a &  little girl named Mia. She had a big heart. Mia loved to help &  others and make them feel safe. Mia liked to
     \\ \midrule
     SIXO & Once upon a time, there was a & girl named Mia. Mia was very kind and compassionate. She always helped her & others and make them feel safe. Mia liked to 
     \\ \midrule
     CTL & Once upon a time, there was a & girl named Mia. She had a thin, pink dress. Mia liked to & others and make them feel safe. Mia liked to 
    \\ \bottomrule \\
\end{tabular}
}\vheader
\vspace*{-.2cm}
\normalsize
\end{table*}

\vheader 
\paragraph{Infilling}
In \cref{table:infilling_examples} we compare qualitative results on an example set of conditioning tokens for DPG, SIXO, and CTL (in that order, to reflect increasing KL divergence). The qualitative results correlate with the quantitative measures of KL divergence; the lowest KL divergence (DPG) corresponds to infilled tokens that respect grammar and the topic. SIXO, which has higher KL divergence, fails to respect grammar. CTL generates incorrect grammar and is less on-topic, corresponding to the highest KL divergence among these methods. 

\newcommand{\vfig}{\vspace*{-.5cm}}
\begin{table}[t]
\centering 
\caption{KL Divergences (averaged over conditioning tokens drawn from the base model) 
for Infilling Experiments (\cref{sec:infilling} ) with 2 Output Tokens and 1 Conditioning Token 
}
\vspace{-.3cm}
\label{table:res_plast2_1}
\resizebox{0.475\textwidth}{!}{
\small 
\begin{tabular}{cccc}
\toprule
Proposal $q_{o_T}$ & Twist Learning & $\mathbb{E}_{o_T}\left[\DKL\pV{q_{o_T}}{\si_{o_T}}\right]$ & $\mathbb{E}_{o_T}\left[\DKL\pV{\si_{o_T}}{q_{o_T}}\right]$ \tabularnewline
\midrule\midrule
Twisted & Contrastive & $0.47 \pm 0.10$ & $0.25 \pm 0.01$ \tabularnewline  \midrule
Twisted & RL & $0.42 \pm 0.10$ & $0.15 \pm 0.01$ \tabularnewline  \midrule
Twisted & SIXO & $0.47 \pm 0.11$ & $0.25 \pm 0.02$ \tabularnewline  \midrule
Twisted & FUDGE & $2.62 \pm 0.33$ & $0.90 \pm 0.02$ \tabularnewline  \midrule \midrule
DPG & -- & \boldmath{$0.16 \pm 0.07$} & \boldmath{$0.14 \pm 0.01$} \tabularnewline  \midrule
PPO & -- & $0.52 \pm 0.04$ & $1.09 \pm 0.34$ \tabularnewline
\bottomrule
\end{tabular}%
}
\normalsize
\end{table}

\vheader 
\subsection{Infilling with Fewer Tokens}
\label{sec:infilling_2_1}
\vheader 

We consider the same setting as \cref{sec:infilling} but only generating 2 tokens, conditioned on 1 token. 
We show KL divergence evaluations in \cref{table:res_plast2_1}. Our evaluation reveals interesting differences among learning methods, even in this easier setting where most methods achieve low KL divergence in both directions. 
DPG and RL learns best, while FUDGE learns notably slower. PPO suffers on $\DKL\pV{\si}q$, though this may be unsurprising since PPO does not make use of exact $\sigma$ samples. 

\vheader 
\subsection{Approximate vs. Exact Posterior Sampling}
\label{sec:exact_vs_appr}
\vheader
In our toxicity and sentiment experiments, we train using approximate $\sigma$ samples to reflect the more common real-world setting where 
the amount of exact samples needed for training are not available. 
However, here we run an additional ablation experiment for insight into 
the effect of positive versus approximate sampling. 
We use rejection sampling (\cref{sec:positive_ebm}) to generate exact posterior samples for training. This is much slower than generating approximate samples, so is not a practical strategy for training; we investigate this solely for understanding.

We provide a comparison of KL divergences (evaluated the same way as in the main paper) when training using exact versus approximate $\sigma$ samples for a selection of methods that performed well in our previous experiments and are able to make use of $\sigma$ samples. Toxicity (\cref{sec:toxclass}) results are in \cref{table:exact_vs_appr_toxc} and sentiment (\cref{sec:sent}) results are in \cref{table:exact_vs_appr_sent}. 
The first two columns of KL divergences are for exact $\sigma$ samples. The next two are for training on the same number of samples, but using approximate positive sampling (\cref{sec:positive_ebm}). 
Overall, for a constant number of samples, having exact $\sigma$ samples improves performance for most methods. 
Note however that there is an additional time cost required for rejection sampling to generate exact samples, so the exact $\sigma$ training requires significantly more wall-clock time for any given number of samples. 

We also plot the single-sample KL divergence in both directions as a function of training time for exact vs. approximate sampling, on toxicity and sentiment experiments, %
in \cref{fig:exact_vs_appr}.
The approximate sampling results match those
in the main paper (with different colors). The exact $\sigma$ sample results cut off earlier because the time cost required for rejection sampling reduces the number of gradient updates that can be made for a given amount of wall-clock time.

\begin{table}[ht]
\centering 
\caption{KL Div. for Toxicity Experiments (\cref{sec:toxclass}), comparing exact $\sigma$ samples versus approximate positive sampling. %
}
\vspace{-.2cm}
\label{table:exact_vs_appr_toxc}
\resizebox{.65\textwidth}{!}{
\small 
\begin{tabular}{cccccc}
\toprule
&  & \multicolumn{2}{c}{Exact $\sigma$ Samples} &  \multicolumn{2}{c}{Same \# of Approx. $\sigma$ Samples}
\tabularnewline
Proposal $q$ & Type of Twist Learning & $\DKL\pV q{\si}$ & $\DKL\pV{\si}q$ & $\DKL\pV q{\si}$ & $\DKL\pV{\si}q$ 
\tabularnewline
\midrule\midrule
Twisted & Contrastive & $2.54 \pm 0.02$ & $2.68 \pm 0.09$
 & $2.99 \pm 0.18$ & $3.22 \pm 0.09$ \tabularnewline  \midrule
Twisted & RL & $3.23 \pm 0.10$ & $3.24 \pm 0.04$
 & $3.48 \pm 0.15$ & $3.49 \pm 0.13$ \tabularnewline  \midrule
Twisted & SIXO & $2.37 \pm 0.06$ & $2.52 \pm 0.05$
 & $2.70 \pm 0.17$ & $3.05 \pm 0.22$ \tabularnewline  \midrule
DPG & -- & $1.51 \pm 0.01$ & $1.50 \pm 0.01$
 & $2.35 \pm 0.15$ & $2.48 \pm 0.10$ \tabularnewline
\bottomrule
\end{tabular}%
}
\normalsize
\end{table}

\vfig
\begin{table}[ht]
\centering
\caption{KL Div. for Sentiment Experiments (\cref{sec:sent}), comparing exact $\sigma$ samples versus approximate positive sampling.
}
\vspace{-.2cm}
\label{table:exact_vs_appr_sent}
\resizebox{0.7\textwidth}{!}{
\small 
\begin{tabular}{cccccc}
\toprule
&  & \multicolumn{2}{c}{Exact $\sigma$ Samples} &  \multicolumn{2}{c}{Same \# of Approx. $\sigma$ Samples}
\tabularnewline
Proposal $q\left(\s\right)$ & Type of Twist Learning & $\DKL\pV q{\si}$ & $\DKL\pV{\si}q$ & $\DKL\pV q{\si}$ & $\DKL\pV{\si}q$ 
\tabularnewline
\midrule\midrule
Twisted & Contrastive & $0.71 \pm 0.02$ & $0.64 \pm 0.02$
 & $0.70 \pm 0.02$ & $0.60 \pm 0.01$ \tabularnewline  \midrule
Twisted & RL & $1.28 \pm 0.05$ & $0.94 \pm 0.02$
 & $2.09 \pm 0.08$ & $1.76 \pm 0.07$ \tabularnewline  \midrule
Twisted & SIXO & $0.68 \pm 0.02$ & $0.60 \pm 0.01$
 & $0.86 \pm 0.02$ & $0.68 \pm 0.01$ \tabularnewline  \midrule
DPG & -- & $0.70 \pm 0.02$ & $0.58 \pm 0.01$
 & $0.89 \pm 0.03$ & $0.69 \pm 0.00$
 \tabularnewline
\bottomrule
\end{tabular}%
}
\normalsize
\end{table}
\vfig

\begin{figure}[t]
\begin{subfigure}{0.49\textwidth}
\centerline{\includegraphics
[width=\textwidth]
{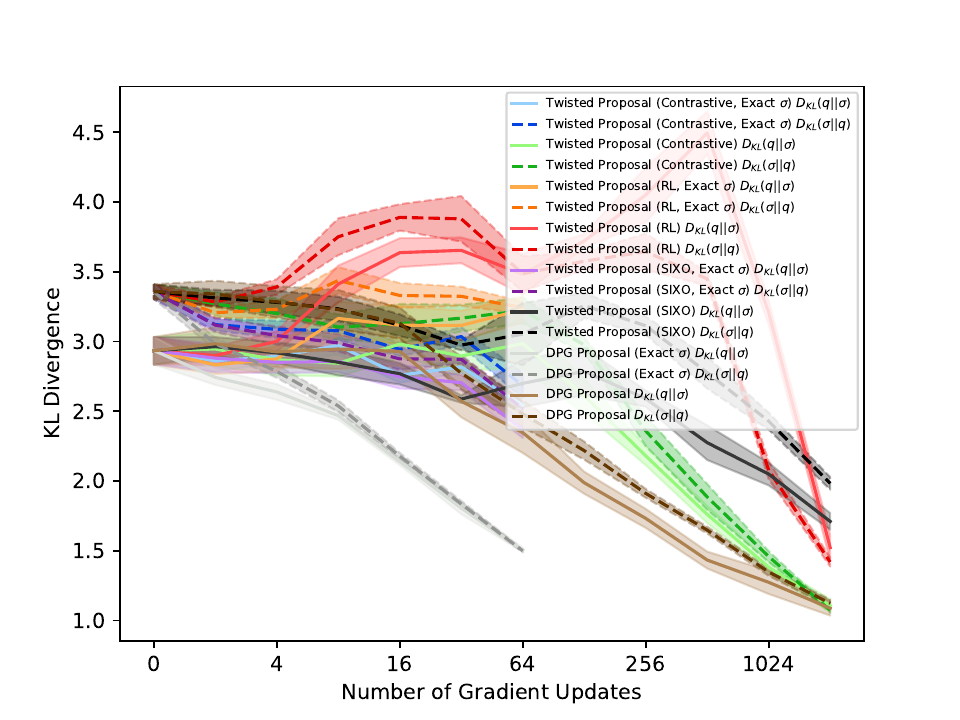}}
\caption{Toxicity (\cref{sec:toxclass})}
\end{subfigure}
\begin{subfigure}{0.49\textwidth}
    \centerline{\includegraphics
[width=\textwidth]
{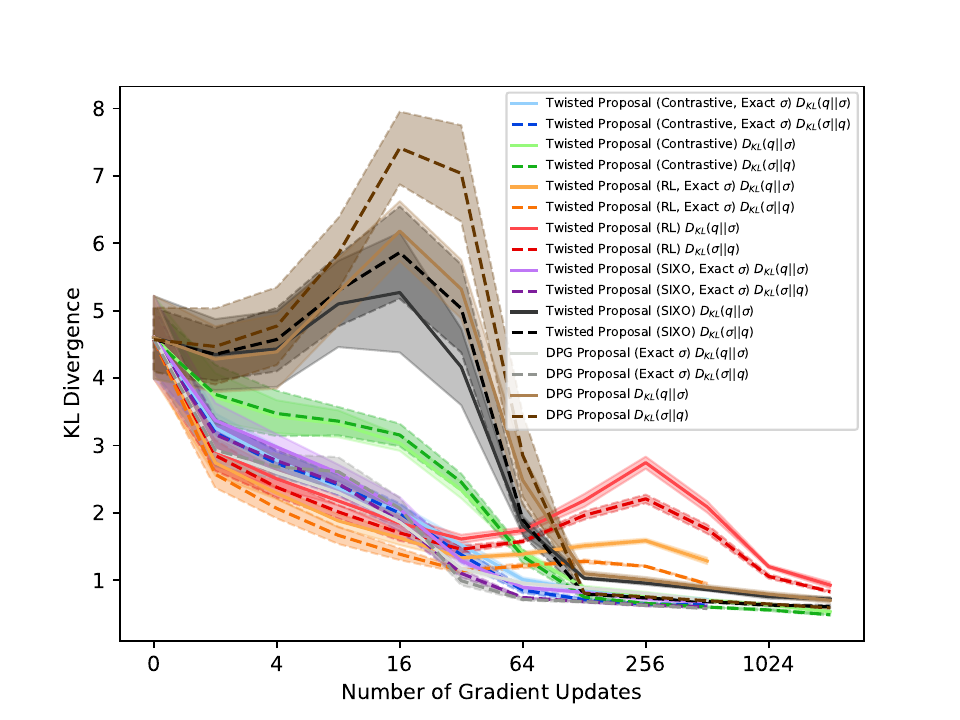}}
\caption{Sentiment (\cref{sec:sent})}
\end{subfigure}
\vspace*{-.2cm}
\caption{Training comparison for Exact versus Approximate $\sigma$ (positive) sampling, as described in \cref{sec:exact_vs_appr}. Having access to exact target samples makes learning lead to lower KL divergences in a more reliable manner.
}\label{fig:exact_vs_appr}
\end{figure}

\end{document}